\pdfoutput=1
\documentclass[11pt]{article}

\usepackage[utf8]{inputenc}
\usepackage{natbib}
\bibpunct{(}{)}{;}{a}{,}{,}

\usepackage{color}
\usepackage{url}
\usepackage{amsmath,amsfonts,amsthm}
\usepackage{graphicx,subfigure}
\usepackage{bm}
\usepackage{algorithm,algpseudocode}
\usepackage[plainpages=false]{hyperref}

\usepackage{titlesec}
\titleformat{\section}
{\normalfont\Large\bfseries\filcenter}{\thesection.}{1 ex}{}
\titleformat{\subsection}%
{\normalfont\normalsize\bfseries}{\thesubsection.}{1 ex}{}
\titleformat{\subsubsection}[runin]
{\normalfont\normalsize\bfseries\filcenter}{\thesubsubsection.}{1 ex}{}
\usepackage[charter]{mathdesign}
\usepackage[mathcal]{eucal}

\usepackage[margin=1.2 in]{geometry}

\usepackage{thmtools,thm-restate}

\declaretheorem[name=Lemma]{Lem}
\declaretheorem[sibling=Lem,name=Theorem]{Thm}
\declaretheorem[sibling=Lem,name=Proposition]{Prop}
\declaretheorem[sibling=Lem,name=Corollary]{Cor}
\declaretheorem[sibling=Thm,name=Conjecture]{Conj}

\declaretheorem[style=definition,sibling=Lem,name=Remark]{Rem}
\declaretheorem[style=definition,sibling=Lem,name=Definition]{Def}

\newcommand{\R}{\ensuremath{\mathbb{R}}}		%
\newcommand{\C}{\ensuremath{\mathbb{C}}}		%
\newcommand{\N}{\ensuremath{\mathbb{N}}}		%
\newcommand{\Z}{\ensuremath{\mathbb{Z}}}        %

\newcommand{\K}{\ensuremath{\mathbb{K}}}

\newcommand{\card}[1]{\left|{#1}\right|}

\newcommand{\rk}{\operatorname{rank}}
\newcommand{\close}{\operatorname{cl}}
\newcommand{\uclose}{\operatorname{ucl}}
\newcommand{\core}{\operatorname{core}}

\newcommand{\vect}{\operatorname{vec}}

\newcommand{\rowspan}{\operatorname{rowspan}}

\newcommand{\calC}{\mathcal{C}}

\newcommand{\calE}{\mathcal{E}}

\newcommand{\calM}{\mathcal{M}}

\newcommand{\calS}{\mathcal{S}}

\newcommand{\calX}{\mathcal{X}}
\newcommand{\calY}{\mathcal{Y}}
\newcommand{\calZ}{\mathcal{Z}}

\newcommand{\mA}{\mathbf{A}}

\newcommand{\mI}{\mathbf{I}}
\newcommand{\mJ}{\mathbf{J}}
\newcommand{\mM}{\mathbf{M}}

\newcommand{\mP}{\mathbf{P}}

\newcommand{\mS}{\mathbf{S}}

\newcommand{\mU}{\mathbf{U}}
\newcommand{\mV}{\mathbf{V}}
\newcommand{\mX}{\mathbf{X}}

\newcommand{\vu}{\mathbf{u}}
\newcommand{\vv}{\mathbf{v}}
\newcommand{\ve}{\mathbf{e}}

\DeclareMathOperator*{\Id}{I}
\DeclareMathOperator*{\Van}{V}

\makeatletter
\providecommand*{\diff}%
{\@ifnextchar^{\DIfF}{\DIfF^{}}}
\def\DIfF^#1{%
\mathop{\mathrm{\mathstrut d}}%
\nolimits^{#1}\gobblespace
}
\def\gobblespace{%
\futurelet\diffarg\opspace}
\def\opspace{%
\let\DiffSpace\!%
\ifx\diffarg(%
\let\DiffSpace\relax
\else
\ifx\diffarg\[%
\let\DiffSpace\relax
\else
\ifx\diffarg\{%
\let\DiffSpace\relax
\fi\fi\fi\DiffSpace}
\makeatother

\begin{document}

\title{The Algebraic Combinatorial Approach for Low-Rank Matrix Completion}
\date{}
\author{Franz J.~Kir\'aly\thanks{Department of Statistical Science, Department of Statistical Science,
London WC1E 6BT, United Kingdom. \url{f.kiraly@ucl.ac.uk}}
\and Louis Theran \thanks{Aalto Science Institute and Department of Information
and Computer Science, Aalto University. \url{theran@math.fu-berlin.de}
(This research was supported by the European Research Council under the
European Union's Seventh Framework Programme (FP7/2007-2013) /
ERC grant agreement no 247029-SDModels while the author was at
Freie Universität, Berlin.)}
\and Ryota Tomioka\thanks{Toyota Technological Institute at Chicago.
6045 S. Kenwood Ave., Chicago, Illinois 60637, USA. \url{tomioka@ttic.edu}}}

\maketitle

\begin{abstract}
\begin{normalsize}
We present a novel algebraic combinatorial view on low-rank matrix completion based on studying relations between a few entries with tools from algebraic geometry and matroid theory. The intrinsic locality of the approach allows for the treatment of single entries in a closed theoretical and practical framework. More specifically, apart from introducing an algebraic combinatorial theory of low-rank matrix completion, we present probability-one algorithms to decide whether a particular entry of the matrix can be completed. We also describe methods to complete that entry from a few others, and to estimate the error which is incurred by any method completing that entry. Furthermore, we show how known results on matrix completion and their sampling assumptions can be related to our new perspective and interpreted in terms of a completability phase transition.
\end{normalsize}
\end{abstract}

\section*{On this revision}
This revision - version 4 - is both abridged and extended in terms of exposition and results, as compared to version 3 \cite{KirTheTomUno12version3}. The theoretical foundations are developed in a more ad-hoc way which allow to reach the main statements and algorithmic implications more quickly. Version 3 contains a more principled derivation of the theory, more related results (e.g., estimation of missing entries and its consistency, representations for the determinantal matroid, detailed examples), but a focus which is further away from applications. A reader who is interested in both is invited to read the main parts of version 4 first, then go through version 3 for a more detailed view on the theory.

\section{Introduction}
\label{Sec:intro}

Matrix completion is the task to reconstruct (to ``complete'') matrices, given a subset of entries at
known positions. It occurs naturally in many practically relevant problems, such as missing feature imputation,
multi-task learning~\citep{ArgMicPonYin08}, transductive learning~\citep{GolZhuRecXuNow10}, or collaborative
filtering and link prediction~\citep{SreRenJaa05,AcaDunKol09,MenElk11}.

For example, in the ``NetFlix problem'', the rows of the matrix
correspond to users, the columns correspond to movies, and the entries
correspond to the rating of a movie by a user. Predicting how \emph{one specific} user
will rate \emph{one specific} movie then reduces to completing a \emph{single unobserved entry}
from the observed ratings.

For arbitrarily chosen position $(i,j)$, the primary questions are:
\begin{itemize}
\item \emph{Is it possible to reconstruct the entry $(i,j)$?}
\item \emph{How many possible completions are there for the entry $(i,j)$?}
\item \emph{What is the value of the entry $(i,j)$?}
\item \emph{How accurately can one estimate the entry $(i,j)$?}
\end{itemize}
In this paper, we answer these questions \emph{algorithmically} under the common \emph{low-rank assumption} - that is, under the model assumption (or approximation) that there is an underlying complete matrix of some low rank $r$ from which the partial observations arise.
Our algorithms are the first in the low-rank regime that provide information about single
entries.  They adapt to the combinatorial structure of the observations in that, if it is possible,
the reconstruction process can be carried out using much less than the full set of observations.
We validate our algorithms on real data.  We also identify \emph{combinatorial} features of the
low-rank completion problem.  This then allows us to study low-rank matrix
completion via  tools from, e.g., graph theory.

\subsection{Results}
Here is a preview of the results and themes of this paper, including the answers to the
main questions.

\subsubsection*{Is it possible to reconstruct the entry $(i,j)$?}
We show that whether the entry $(i,j)$ is completable
depends, with probability one for \emph{any continuous sampling regime},
only on the \emph{positions} of the observations and the position
$(i,j)$ that we would like to reconstruct
(Theorem \ref{Thm:single-entry-finite-generic}).
The proof is explicit and easily converted into an exact (probability one) algorithm for
computing the set of completable positions (Algorithm \ref{Alg:cclosure}).

\subsubsection*{How many possible completions are there for the entry $(i,j)$?}
Whether the entry at position $(i,j)$ is uniquely completable from the
observations, or, more generally, how many completions there are also
depends, with probability one,
only on the positions of the observed entries and $(i,j)$ (Theorem \ref{Thm:uniquecomp}).
We also give an efficient (randomized probability one) algorithm (Algorithm \ref{Alg:cclosure}) that verifies
a sufficient condition for every unobserved entry to be uniquely completable.

\subsubsection*{What is the value of the entry $(i,j)$?}
To reconstruct the missing entries, we introduce a general scheme based on
finding polynomial relations between the observations and one
unobserved one at position $(i,j)$ (Algorithm \ref{Alg:loccomp}).
For rank one matrices (Algorithm \ref{Alg:loccomprkone}), and, in any rank, observation patterns
with a special structure (Algorithm \ref{Alg:minor}) that allows ``solving minor by minor'',
we instantiate the scheme completely and efficiently.

Since, for a specific $(i,j)$, the polynomials needed can be very sparse, our approach
has the property that it adapts to the combinatorial structure of the observed
positions.  To our knowledge, other algorithms for low-rank matrix completion
do not have this property.

\subsubsection*{How accurately can one estimate entry $(i,j)$?}
Our completion algorithms separate out \emph{finding}
the relevant polynomial relations from \emph{solving} them.  When
there is more than one relation, we can use them as different estimates
for the missing entry, allowing for estimation in the noisy setting (Algorithm \ref{Alg:loccomp}).
Because the polynomials are independent of specific observations, the same techniques
yield \emph{a priori estimates} of the variance of our estimators.

\subsubsection*{Combinatorics of matrix completion}
Section \ref{Sec:graphconditions} contains a detailed analysis of whether an entry $(i,j)$ is
completable in terms of a \emph{bipartite graph} encoding the combinatorics of the
observed positions.  We obtain necessary (Theorem \ref{Thm:rank-r-sparse}) and sufficient
(Proposition \ref{Prop:deglb}) conditions for local completability, which are sharp in the sense
that our local algorithms apply when they are met.  We then relate the properties we find
to standard graph-theoretic concepts such as edge-connectivity and cores.  As an application,
we determine a binomial sampling density that is sufficient for solving minor-by-minor nearly exactly via
a random graph argument.

\subsubsection*{Experiments}
Section \ref{Sec:experiments} validates our algorithms on the Movie Lens data set and shows that the
structural features identified by our theories predict completability and completability phase
transitions in practice.

\subsection{Tools and themes}
Underlying our results are a new view of low-rank matrix completion based on
algebraic geometry.  Here are some of the key ideas.

\subsubsection*{Using the local-to-global principle}
Our starting point is that the set of rank $r$, $(m\times n)$-matrices carries the
additional structure of an \emph{irreducible algebraic variety}
(see Section \ref{Sec:detvar}).  Additionally, the observation
process is a polynomial map.  The key feature of this setup is that it gives us access to
fundamental algebraic-geometric ``local-to-global'' results
(see Appendix \ref{Sec:app-algebra}) that assert the observation process
will exhibit a \emph{prototypical behavior}: the answers to the main questions will be the
\emph{same for almost all} low-rank matrices, so they are essentially properties of
the rank and observation map.  This lets us study the main questions in terms of
observed and unobserved \emph{positions} rather than specific \emph{partial matrices}.

On the other hand, the same structural results show we can \emph{certify} that
properties like completability hold via \emph{single examples}.  We exploit
this to replace very complex basis eliminations with fast algorithms based on
numerical linear algebra.

\subsubsection*{Finding relations among entries using an ideal}
Another fundamental aspect of algebraic sets are
characterized exactly by the \emph{vanishing ideal} of polynomials that
evaluate to zero on them.  For matrix completion, the meaning is:
\emph{every} polynomial relation between the observations and
a specific position $(i,j)$ is generated by a \emph{finite}
set of polynomials we can in principle identify (See Section \ref{Sec:closure}).

\subsubsection*{Connecting geometry to combinatorics using matroids}
Our last major ingredient is the use of the Jacobian of the observation
map, evaluated at a ``generic point''. The independence/dependence relation
among its rows is invariant (with matrix-sampling probability one) over the set of
rank $r$ matrices that characterizes whether a position $(i,j)$ is
completable.  Considering the subsets of independent rows as simply
subsets of a finite set, we obtain a \emph{linear matroid} characterizing
completability.  This perspective allows access to combinatorial tools
of matroid theory, enabling the analysis in Section \ref{Sec:graphconditions}.

\subsection{Context and novelty}
Low-rank Matrix Completion has received a great deal of attention from the community.  Broadly speaking, two
main approaches have been developed: convex relaxations of the rank constraints
\citep[e.g.,][]{CanRec09,CandesTao, NegWai11,SalSre10,NegWai12,FoySre11,SreShr05}; and spectral
methods \citep[e.g.,][]{KesMonOh10,Mek09,C12}.  Both of these \citep[see][]{CandesTao,KesMonOh10}
yield, in the noiseless case,
optimal sample complexity bounds (in terms of the number of positions uniformly sampled)
for exact reconstruction of an underlying matrix meeting certain analytic assumptions.
All the prior work of which we are aware concentrates on: (i) completing all the unobserved
entries; (ii) sets of observed positions sampled from some known distribution.  The
results here, by contrast, apply specifically to \emph{fixed} sets of observations and
provide information about \emph{any} unobserved position $(i,j)$.

Analogously, all the prior work on Low-rank Matrix Completion from noisy observations
concentrates on: (i) estimating every missing entry; (ii) denoising every observed entry; and
(iii) minimizing the MSE over the whole matrix.  Our approach allows, for the first time,
to construct \emph{single-entry estimators} that minimize the variance of the entry under
consideration; \cite{KirThe13RankOneEst} showed how to do this efficiently in
rank $1$.

\subsection{Organization}
The sequel is structured as follows: Section \ref{Sec:background} introduces the background material
we need; Sections \ref{Sec:finitecompletability} and \ref{Sec:uniquecompletability} develop
our algebraic-combinatorial theory and derive algorithms for determining when an entry
is completable; Section \ref{Sec:closure} formulates the reconstruction process itself algebraically;
Section \ref{Sec:graphconditions} contains a combinatorial analysis of the problem; finally Section \ref{Sec:experiments}
validates our approach on real data.  The Appendix collects some technical results required in the proofs
of the main theorems.

\section{Background and setup}
\label{Sec:background}
In this section, we introduce two essential objects, the set of
low-rank matrices $\calM(m\times n, r)$ and the set of observed positions $E$.  We
also define the concept of \emph{genericity}.

\subsection{The determinantal variety}\label{Sec:detvar}
First, we set up basic notation. A matrix is denoted by upper-case bold character like
$\mA$. We denote by $[n]$ the set of integers $\{1,2,\ldots,n\}$. $\mA_{I,J}$
denotes the submatrix of an $m\times n$ matrix $\mA$ specified by the sets
of indices $I\subseteq [m]$ and $J\subseteq [n]$. The $(i,j)$ element of a
matrix $\mA$ is denoted by $A_{ij}$. The cardinality of a set $I$ is
denoted by $\card{I}$.

Now we define the set of matrices of rank at most $r$.
\begin{Def}
The set of all complex $(m\times n)$-matrices of rank $r$ or less will be denoted by
$\calM (m\times n, r)=\{\mA\in\C^{m\times n}\;:\; {\rm rank}(\mA)\le r\}.$
We will always assume that $r\le m\le n;$ by transposing the matrices, this is no loss of generality.
\end{Def}

Some basic properties of $\calM(m\times n,r)$ are summarized in the
following proposition.
\begin{Prop}[Properties of the determinantal variety]\label{Prop:detvarprops}
The following hold for $\calM(m\times n,r)$:
\begin{itemize}
\item[(i)] $\calM(m\times n,r)$ is the image of the map
$
\Upsilon:(\mU,\mV)\mapsto \mU\mV^\top
$,
where $\mU\in \C^{m\times r}$ and $\mV\in \C^{n\times r}$, and is therefore
irreducible.
\item[(ii)] $\calM(m\times n,r)$ has dimension
$$d_r(m,n) := \dim\calM (m\times n, r) =
\begin{cases}
r(m + n - r) & \text{if $m\ge r$ and $n\ge r$} \\
mn & \text{otherwise}
\end{cases}
$$
\item[(iii)] Every $(r+1)\times (r+1)$ minor of a matrix in
$\calM(m\times n,r)$ is zero, namely,
$$
\det (\mA_{I,J})=0,\quad \forall I\subseteq [m], J\subseteq [n],
$$
where $\card{I}=r+1$, $\card{J}=r+1$, and $\mA\in\calM(m\times n,r)$.
\item[(iv)] The vanishing ideal of $\calM(m\times n,r)$ is generated by the
vanishing of the minors from part 3.
\end{itemize}
\end{Prop}
\begin{proof}
\textbf{(i)} The existence of the singular-value decomposition imply that
$\calM(m\times n,r)$ is the surjective image of $\C^{r(m+n)}$ under the algebraic
map $\Upsilon$.

\textbf{(ii)}
This follows from \textbf{(i)} and the uniqueness of the SVD, or  \citep[section 1.C, Proposition 1.1]{Bruns}.

\textbf{(iii)} The rank of a matrix equals the order of the largest non-vanishing minor.

\textbf{(iv)} By \cite[Theorem~2.10, Remark~2.12, and Corollary~5.17f]{Bruns}, the ideal generated by the
$r\times r$ minors is prime.  Since it vanishes on the irreducible $\calM(m\times n,r)$, it is the
vanishing ideal.
\end{proof}
The set of observed positions is denoted by $E$ and can be viewed as a
bipartite graph as follows.
\begin{Def}
Let $\calE:=[m]\times [n]$.
The set containing the positions of observed entries is denoted by
$E\subseteq \calE$. We define the bipartite graph
$G(E)=(V,W,E)$ with vertices $V=[m]$ corresponding to rows and vertices
$W=[n]$ corresponding to columns.
We call the $m\times n$ adjacency matrix $\mM(E)$ of
the bipartite graph $G(E)$  a {\em mask}. The map
$$\Omega: \mA \mapsto \left( A_{ij}\right)_{(i,j)\in E},$$
where $\mA\in\calM(m\times n,r)$, is called a {\em masking} (in rank $r$).
\end{Def}

Note that the set of observed positions $E$, the adjacency
matrix $\mM$, and the map $\Omega$ can be used interchangeably. For
example, we denote by $\mM(\Omega)$ the adjacency matrix corresponding to
the map $\Omega$, and by $E(\mM)$ the set of positions specified by $\mM$, and so on.
Figure~\ref{fig:g1g2} shows two bipartite graphs $G_1$ and $G_2$
corresponding to the following two masks:
$$
\mM_1=
\begin{pmatrix}
1 & 0 & 1\\
1 & 1 & 0\\
1 & 0 & 0
\end{pmatrix},
\qquad
\mM_2=
\begin{pmatrix}
1 & 0 & 1\\
0 & 1 & 0\\
1 & 0 & 1
\end{pmatrix}.
$$

\begin{figure}[tb]
\begin{center}
\includegraphics[width=.3\textwidth]{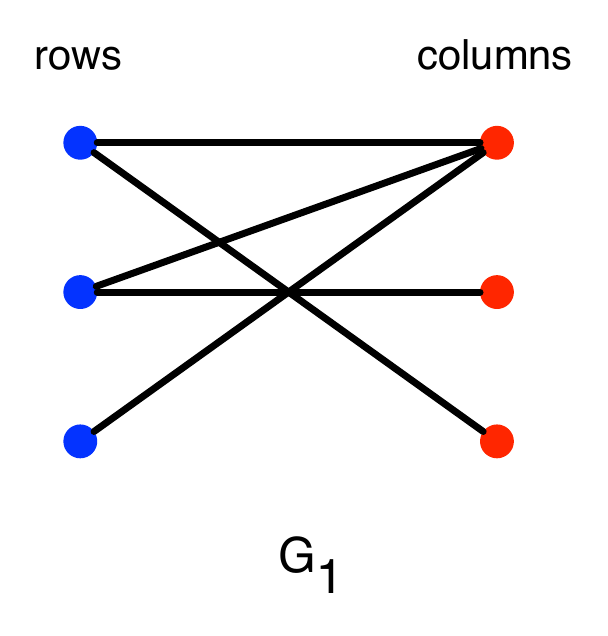}
\includegraphics[width=.3\textwidth]{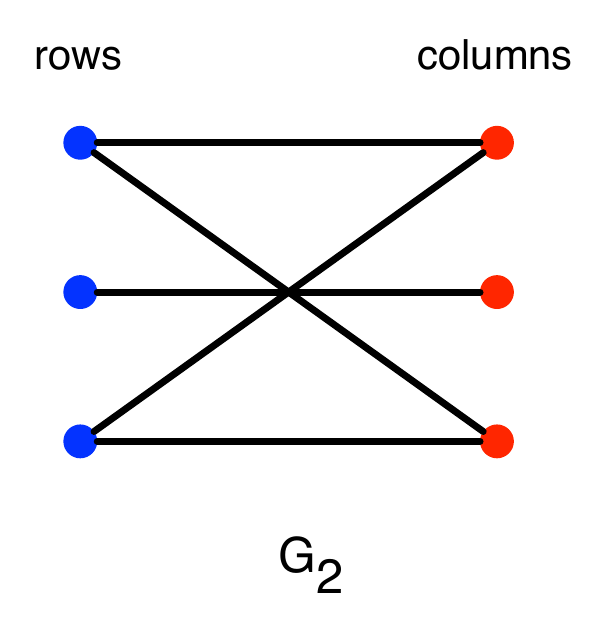}
\caption{Two bipartite graphcs $G_1$ and $G_2$ corresponding to the
masks $\mM_1$ and $\mM_2$, respectively. Every non-edge corresponds to an
unobserved entry.}
\label{fig:g1g2}
\end{center}
\end{figure}
\subsection{The Jacobian of the masking operator}
Informally, the question we are going to address is:
\begin{center}
\emph{Which entries of $\mA$ are (uniquely) reconstructable, given the masking $\Omega(\mA)$?}
\end{center}
The answer will depend
on the interaction between the algebraic structure of
$\calM(m\times n, r)$ and the combinatorial structure of $E$.
The main tool we use to study this is the Jacobian of the map $\Upsilon$, since
at smooth points, we can obtain information about the dimension of the
pre-image $\Omega^{-1}(\mA)$ from its rank.
\begin{Def}\label{Def:Jacobian}
We denote by $\mJ$ the Jacobian of the map $\Upsilon: \mU,\mV \mapsto
\mA=\mU\mV^\top$. More specifically, the Jacobian of the map from $\mU$ and $\mV$ to $A_{ij}$ can be
written as follows:
\vspace{-7mm}\begin{align}\label{Eq:jacobian-row}
\left(\frac{\partial A_{ij}}{\partial \vu_1^\top},\ldots,\frac{\partial A_{ij}}{\partial \vu_m^\top}, \frac{\partial A_{ij}}{\partial \vv_1^{\top}},\ldots,\frac{\partial A_{ij}}{\partial \vv_n^{\top}}\right)=
\begin{array}[tb]{c}
\vphantom{\text{Derivative}}\\ \vphantom{\uparrow}\\
\Bigl(\begin{matrix}
0 & \cdots & \vv_j^\top & \cdots 0  & 0 \cdots & \vu_i^\top & \cdots & 0
\end{matrix}\Bigr)
\\
\begin{matrix}
&  & \uparrow & \quad\,\, &  & &  \quad\,\, & \uparrow & &
\end{matrix}\\
\text{Derivative wrt $\vu_i$ \qquad Derivative wrt $\vv_j$}
\end{array}
\end{align}
where $\vu_i^\top$ is the $i$th row vector of $\mU$ and $\vv_j^\top$ is the
$j$th row vector of $\mV$.
Stacking the above row vectors for $(i,j)\in[m]\times [n]$, we can write
the Jacobian $\mJ(\mU,\mV)$ as an $mn\times r(m+n)$ matrix as follows:
\begin{align}
\label{eq:jacobian}
\mJ(\mU,\mV)=
\begin{pmatrix}
\begin{matrix}
\mI_m\otimes \vv_1^\top \\
\mI_m\otimes \vv_2^\top \\
\vdots              \\
\mI_m\otimes \vv_n^\top
\end{matrix}
& \mI_n \otimes \mU
\end{pmatrix},
\end{align}
where $\otimes$ denotes the Kronecker product. Here the rows of $\mJ$
correspond to the entries of $\mA$ in the column major order.
\end{Def}

\begin{Lem}\label{Lem:kronecker}
Every matrix $\mS\in\C^{m\times n}$ whose vectorization $\vect(\mS)$ lies in the left null space of $\mJ(\mU,\mV)$ satisfies
$$
\mU^{\top}\mS =0,\qquad \mS\mV=0,
$$
and any $\mS$ satisfying the above lies in the left null space of
$\mJ(\mU,\mV)$. In addition, the dimension of the null space is
$(m-r)(n-r)$ if $\mU$ and $\mV$ have full column rank $r$.
\end{Lem}
\begin{proof}
Let $\mP$ be the $mn\times mn$ permutation matrix defined by
\begin{align*}
\mP\vect(\mX) = \vect(\mX^{\top}).
\end{align*}
Note that $\mP^{\top}\mP=\mI_{m+n}$, and
$$
\mP
\begin{pmatrix}
\mI_m\otimes \vv_1^\top \\
\vdots              \\
\mI_m\otimes \vv_n^\top
\end{pmatrix} = \mI_m \otimes \mV.
$$
Thus we have
\begin{align*}
\vect^{\top}(\mS)\mJ &=
\left(
\vect^{\top}(\mS)\mP^{\top} \left(\mI_m\otimes \mV\right) ,
\vect^{\top}(\mS) \left(\mI_n\otimes \mU\right)
\right)\\
&=\left(\vect^{\top}(\mV^{\top}\mS^{\top}),
\vect^{\top}(\mU^{\top}\mS)
\right),
\end{align*}
which is what we wanted. To show the last part of the lemma, let
$\mU_{\perp}\in\C^{m\times (m-r)}$ and $\mV_{\perp}\in\C^{n\times (n-r)}$ be any basis of the orthogonal
complement space of $\mU$ and $\mV$, respectively.
Since the null space can be
parametrized as $\mS=\mU_{\perp}\mS'\mV_{\perp}^{\top}$ by
$\mS'\in\C^{(m-r)\times(n-r)}$, and this parametrization is one-to-one,
we see that the dimension of the null space is $(m-r)(n-r)$.
\end{proof}

Now we define the Jacobian corresponding to the set of observed positions $E$.

\begin{Def}
For a position $(k,\ell)$, we define $\mJ_{(k,\ell)}$ to be the single row of $\mJ$ corresponding to the position $(k,\ell)$.
Similarly, we define $\mJ_E$ to be the submatrix of $\mJ$ consisting of rows
corresponding to the set of observed positions $E$. Due to the chain
rule, $\mJ_E$ is the Jacobian of the map $\Omega\circ\Upsilon$.
\end{Def}

\subsection{Genericity}
The pattern of zero and non-zero entries in~\eqref{Eq:jacobian-row} hints at a connection to
purely combinatorial structure.  To make the connection precise, we introduce \emph{genericity}.

\begin{Def}
We say a boolean statement $P(\mX)$ holds {\em for a generic $\mX$} in irreducible algebraic variety
$\calX$, if for any Hausdorff continuous measure $\mu$ on $\calS$, $P(\mX)$ holds with
probability $1$.
\end{Def}
These kinds of statements are sometimes called ``generic properties,'' and they
are properties of $\calX$, rather than any specific $\mu$.
The prototypical example of a generic property is where $\calX = \C^n$,
$p\neq 0$ is a polynomial, and the statement $P$ is ``$p(\mX)\neq 0$.''

Here, we are usually concerned with the case $\calX = \calM(m\times n,r)$.
Proposition \ref{Prop:detvarprops}
tells us that $m$, $n$ and $r$ define $\calM(m\times n,r)$ completely.
Assertions of the form ``For generic $\mX\in \calM(m\times n,r)$, $P(\mX)$ depends
only on $(t_1,t_2,\ldots)$'' mean $P(\mX)$ is a generic statement for \emph{all}
$\calM(m\times n,r)$ with the parameters $t_i$ fixed.

Although showing whether some statement $P$ holds generically might seem
hard, we are interested in $P$ defined by polynomials.  In this case,
results in Appendix \ref{Sec:app-algebra} imply that it is enough to show that
either $P$ holds: (a) on an open subset of $\calX$
in the metric topology; or (b) almost surely, with respect to a
Hausdorff continuous measure.

As a first step, and to illustrate the ``generic philosophy'' we show that the generic
behavior of the Jacobian $\mJ_E(\mU,\mV)$ is a property of $E$.  We first start
by justifying the definition via $(\mU,\mV)$ (as opposed to $\mA$).
\begin{Lem}\label{Lem:generic-rank}
For all $E\subset \calE$ and $\mA\in\calM(m\times n,r)$ generic,
with $\mA = \Upsilon(\mU,\mV)$, and $\mU$ and $\mV$ generic,
the rank of $\mJ_E$ is independent of $\mA$, $\mU$, and $\mV$.
\end{Lem}
\begin{proof}
We first consider the composed map $\Omega\circ \Upsilon$.  This is a polynomial
map in the entries of $\mU$ and $\mV$, so its critical points (at which the differential
$\mJ_E$ attains less than its maximum rank) is an algebraic subset of $\C^{r(m+n)}$.
The ``Semialgebraic Sard Theorem'' \cite[Theorems 3.1, 4.1]{KOS00}
then implies that the set of critical points is, in fact,
a \emph{proper} algebraic subset of $\C^{r(m+n)}$.

So far, we have proved that the rank of $\mJ_E$ is independent of $\mU$ and $\mV$.
However, $\mU$ and
$\mV$ are not uniquely determined by $\mA$.  To reach the stronger conclusion,
we first observe that a generic $\mA\in \calM(m\times n,r)$ is a regular value of $\Upsilon$,
again by Semialgebraic Sard.  Thus, the set of $(\mU,\mV)$ such that $\Upsilon(\mU,\mV)$ and
$\Omega\circ \Upsilon(\mU,\mV)$
are both regular values is the intersection of two dense sets in $\C^{r(m+n)}$. %
\end{proof}

\section{Finite completability}
\label{Sec:finitecompletability}

This section is devoted to the question \emph{``Is it possible to reconstruct the entry $(i,j)$?''}. We will show that under mild assumptions, the answer depends only on the \emph{position} $(i,j)$, the observed positions, and the rank, but not the observed entries. The main idea behind this result is relating reconstructability to the rows of the Jacobian $\mJ$, and their rank, which can be shown to be independent of the actual entries for almost all low-rank matrices. Therefore, we can later separate the question of reconstructibility from the actual reconstrution process.

\subsection{Finite completability as a property of the positions}\label{sec:completable-closure-basics}
We show how to predict whether the entry at a specific \emph{position}
$(k,\ell)$ will be reconstructible from a specific set of positions
$E\subset \calE$.
For the rest of this section, we fix the parameters $r$, $m$ and $n$, and denote by $E$ a
set of observed positions. The symbol $\K$ will denote either of the real numbers
$\R$ or the complex numbers $\C$.

We start by precisely defining what it means for one set of entries to imply the
imputability of another entry.
\begin{Def}\label{Def:finitely-completable}
Let $E\subset \calE$ be a set of observed positions and $\mA$ be a rank $r$
true matrix. The entry $A_{k\ell}$ is
\emph{finitely completable in rank $r$} from the observed set of entries $\{A_{ij} : (i,j)\in E\}$ if the
entry $A_{k\ell}$ can take only finitely many values when fixing $\Omega(\mA)$.
\end{Def}
There are two subtleties here: the first is that, even if there is an infinity of possible
completions for the whole matrix $\mA$,
it is possible that some specific $A_{k\ell}$ takes on only finitely many values; the question of whether the entry $A_{k\ell}$ at position $(k,\ell)$ is finitely completable
may have different answers for different $\mA$. The theoretical results in this section take care of both issues.
\begin{Thm}\label{Thm:single-entry-finite-generic}
Let $E\subset \calE$ be a set of positions, $(k,\ell)\in \calE\setminus E$ be arbitrary, and
let $\mA\in\K^{m\times n}$ be a generic, $(m\times n)$-matrix of rank $r$.
Whether the entry $A_{k\ell}$ at position $(k,\ell)$ is finitely completable depends only on
the position $(k,\ell)$, the true rank $r$, and
the observed positions $E$ (and not on $\mA$, $m$, $n$, or $\K$).
\end{Thm}
This lets us talk about the finite completability of positions instead of entries.
\begin{Def}
Let $E\subset\calE$ be a set of observed positions, and $(k,\ell)\in \calE\setminus E$.  We say that the
\emph{position
$(k,\ell)$ is finitely completable from $E$ in rank $r$} if, for generic $\mA$, the entry $A_{k\ell}$ is finitely
completable from $\Omega(\mA)$.  The \emph{rank $r$ finitely completable closure $\close_r(E)$} is the
set of positions generically finitely completable from $E$.
\end{Def}
The main tool we use to prove Theorem \ref{Thm:single-entry-finite-generic} is the Jacobian
matrix $\mJ_E$. For it, we obtain
\begin{Thm}\label{Thm:finite-completable-closure}
Let $E\subset \calE$ and let $\mA$ be a generic, rank $r$ matrix.  Then
\[
\close_r(E) = \{(k,\ell)\in \calE : \mJ_{\{(k,\ell)\}}\in\rowspan \mJ_E\}.
\]
\end{Thm}
One implication of Theorem \ref{Thm:finite-completable-closure} is that linear
independence of subsets of rows of $\mJ_E$ is also a generic property.  (In fact,
the proof in Section \ref{sec:finite-proofs} goes in the other direction.)  The
combinatorial object that captures this independence is a matroid.
\begin{Def}\label{Def:determinantal-matoid}
Let $\mA$ be a generic rank $r$ matrix.
The \emph{rank $r$ determinantal matroid} is the linear matroid $(\calE,\rk_r)$, with
rank function
\(
\rk_r(E) = \rk\mJ_E
\).
\end{Def}
In the language of matroids, Theorem \ref{Thm:finite-completable-closure} says that, generically,
the finitely completable closure is equal to the matroid closure in the rank $r$ determinantal
matroid.  This perspective will prove profitable when we consider entry-by-entry algorithms
for completion in Section \ref{Sec:closure} and combinatorial conditions
related to finite completability in Section \ref{Sec:graphconditions}.

\subsection{Computing the finite closure}\label{sec:completable-closure-algorithm}
We describe, in pseudo-code, Algorithm~\ref{Alg:cclosure} which computes the finite closure of $E$. An algorithm for testing whether a single entry $(k,\ell)$ is finitely completable is easily obtained by only testing the entry $(k,\ell)$ in step~\ref{Alg:cclosure.step4}. The correctness of  Algorithm \ref{Alg:cclosure} follows from Theorem \ref{Thm:finite-completable-closure} and the fact that, if we sample $\mU$ and $\mV$
from any continuous density, with probability one, we obtain generic $\mU$ and $\mV$.
\begin{algorithm}[ht]
\caption{\label{Alg:cclosure} Completable closure.\newline
\textit{Input:} A set $E\subset \calE$ of observed positions. \newline %
\textit{Output:} The rank $r$ completable closure $\close_r(E)$. }
\begin{algorithmic}[1]
\State \label{Alg:cclosure.step1} Sample $\mU\in\R^{m\times r},\mV\in\R^{n\times r}$
from a continuous density.
\State \label{Alg:cclosure.step2} Compute the Jacobian matrix $\mJ_E(\mU,\mV)$.
\State \label{Alg:cclosure.step3} Compute the singular value
decomposition of $\mJ_E(\mU,\mV)$. Let $\mV_E$ be the right singular vectors
corresponding to singular values greater than $10^{-12}$.
\State \label{Alg:cclosure.step4} For each $e\in\calE\backslash E$, compute the
projection of $\mJ_{\{e\}}(\mU,\mV)\in\R^{r(m+n)}$ on the subspace
spanned by $\mV_E$. Let the Euclidean norm of the residual of the
projection be $r_{e}$; let $r_{e}=0$ for $e\in E$.
\State \label{Alg:cclosure.step5} Return
$\close_r(E):=\{(i,j)\in\calE\;;\; r_{e}\leq 10^{-8}\}$.
\end{algorithmic}
\end{algorithm}
\begin{Rem}\label{Rem:RAM}
We have presented Algorithm \ref{Alg:cclosure} as a numerical routine based on SVD.
It is also strongly polynomial time in the RAM model. The key observation is that all
of our computations estimate only the \emph{rank} of a matrix, which is a polynomial
identity testing problem.  By Proposition \ref{Prop:detvarprops}, the rank
of $\mJ_E$ is never higher
than $d_r(m,n)$. If $r\ll n + m$, then $d_r(m,n) = O(n + m)$, so the
the rank of $\mJ_E$ is detected by a minor of degree $O(n + m)$.  The main result of
\cite{S80} then implies that, if we sample the entries of $\mU$ and $\mV$ uniformly from a
field $\Z_p$ of prime order $p\approx (n+m)^2$, with probability $1 - O(1/(n + m))$ we obtain the
generic rank of $\mJ_E$ and each $\mJ_{E\cup \{(i,j)\}}$, using, e.g., Gaussian
elimination.
\end{Rem}

\subsection{Proofs}\label{sec:finite-proofs}

\subsubsection{Proof of Theorem~\ref{Thm:finite-completable-closure}}\label{Sec:finite-proof}
Let $(i,j)\in \calE\setminus E$.  Factor the map $\Omega\circ\Upsilon$ into
\begin{equation*}
\C^{r(m+n)} \xrightarrow{\Upsilon} \calM(m\times n,r)
\xrightarrow{f} \C^{\card{E} + 1} \xrightarrow{g} \C^{\card{E}}
\end{equation*}
so that $f$ is the projection of $\Upsilon(\mU,\mV)$ onto the set of entries at positions
$E\cup \{(i,j)\}$ and $g$ then projects out the coordinate corresponding to $(i,j)$.
Lemma \ref{Lem:generic-rank} implies that, since $(\mU,\mV)$ is generic,
all the intermediate image points are smooth. The constant rank theorem then
implies that we can find open neighborhoods
$f(\calM(m\times n,r))\supset M\ni f(\Upsilon(\mU,\mV))$ and
$g\circ f(\calM(m\times n,r)) \supset N\ni g(f(\Upsilon(\mU,\mV)))$ such that
the restriction of $g$ to $M$ is smooth and $g^{-1}(N) \subset M$.  The
constant rank theorem then implies that we have
\[
\dim\left(g^{-1}(N)\right) + \dim N = \dim M.
\]
Since, again using smoothness,
\[
\dim M = \dim\left(g(f(\Upsilon(\mU,\mV)))\right)=\rk\left(\mJ_E(\mU,\mV)\right)
\]
and
\[
\dim N = \dim\left(f(\Upsilon(\mU,\mV))\right)=\rk\left( \mJ_{E\cup\{i,j\}}(\mU,\mV)\right).
\]
the position $(i,j)$ is finitely completable from $E$ and $\Upsilon(\mU,\mV)$,
if and only if
\begin{equation}\label{Eq:equal-rank}
\rk\left(\mJ_E(\mU,\mV)\right)=\rk\left(\mJ_{E\cup\{i,j\}}(\mU,\mV)\right)
\end{equation}
Equation \eqref{Eq:equal-rank} is just the assertion that $\mJ_{\{(i,j)\}}\in\rowspan \mJ_E$.

By Lemma \ref{Lem:generic-rank}, Equation \eqref{Eq:equal-rank} is a generic
statement, independent of $\mA$, $\mU$ and $\mV$.
Because the rows of $\mJ_E$ and $\mJ_{\{(i,j\})}$ have non-zero columns
only at positions depending on $E$ and $(i,j)$, whether \eqref{Eq:equal-rank}
holds does not depend on $m$ and $n$ (which are, by hypothesis, large enough).

Finally, statement that finite completability is the same for $\K=\R$ and $\K=\C$ follows
from Theorem~\ref{Thm:genreal} in the appendix.
\hfill $\qed$

\subsubsection{Proof of Theorem \ref{Thm:single-entry-finite-generic}}
The theorem follows directly from Theorem~\ref{Thm:finite-completable-closure} and the definition of closure.
\hfill $\qed$

\subsection{Discussion}
The kernel of $\mJ_E$ spans the space of infinitesimal deformations of
$(\mU,\mV)$ that preserve $\Omega\circ\Upsilon(\mU,\mV)$.  Because
generic points are smooth, \cite[Curve Selection Lemma]{M68} implies
that every  infinitesimal deformation can be integrated to a
finite deformation. Conversely (this is the harder direction)
every curve in $(\Omega\circ \Upsilon)^{-1}(\mA)$ through $(\mU,\mV)$ has,
as its tangent vector a non-zero infinitesimal deformation.
At non-generic points, this equivalence does not
hold, so the arguments here require genericity and smoothness in an essential
way.

The finite identifiability statements in this section are instances
of a more general phenomenon, which is explored in \cite{KRT13}.
The results there  imply similar identifiability results, such as \cite{HKL12,AMR09,BS85},
that use criteria based on a Jacobian, and also show that our use of the ``$\Upsilon$''
parameterization of $\calM(m\times n,r)$ is not essential.

Another connection is that, since permuting the rows and columns of a matrix
preserves its rank, we get:
\begin{Cor}\label{Cor:graph-matroid}
The rank function $\rk_r(\cdot)$ of the determinantal matroid depends only on the
graph isomorphism type of the graph associated with $E$.
\end{Cor}
In Section \ref{Sec:graphconditions}, we consider completability as a property of \emph{graphs}.
This relies on Corollary \ref{Cor:graph-matroid}.

\section{Unique completability}
\label{Sec:uniquecompletability}

In this section, we will address the question \emph{``How many possible completions are there for the entry $(i,j)$?''}. In section~\ref{sec:completable-closure-basics}, it was shown that whether the \emph{entry} $(i,j)$ is completable depends (under mild assumptions) only on the \emph{position} $(i,j)$, the observed entries, and the rank. In this section, we show an analogue result that the same holds for the number of possible completions as well. Whether there is exactly one solution is of the most practical relevance, and we give a sufficient condition for unique completability.

\subsection{Unique completability as a property of the positions}

We start by defining what it means for one entry to be uniquely completable:

\begin{Def}\label{Def:uniquely-completable}
Let $E\subset \calE$ be a set of observed positions and $\mA$ be a rank $r$
true matrix. The entry $A_{k\ell}$ at position $(k,\ell)\in \calE\setminus E$ is called
\emph{uniquely completable} from the entries $A_{ij},(i,j)\in E$, if $A_{k\ell}$ is uniquely determined by the $A_{ij}, (i,j)\in E$.
\end{Def}

The main theoretical statement for unique completability is an analogue to the main theorem for finite completability; again, whether an entry is uniquely completable, depends only on the positions of the observations, assuming the true matrix is generic.

\begin{Thm} \label{Thm:uniquecomp}
Let $\mA\in\C^{m\times n}$ be a generic $(m\times n)$-matrix of rank $r$, and consider a masking where the entries $A_{ij}$ with $(i,j)\in E\subseteq [m]\times [n]$ are observed. Let $(k,\ell)\in [m]\times [n]$ be arbitrary. Then, whether $A_{k\ell}$ is uniquely completable from the $A_{ij},(i,j)\in E$ depends only on the position $(k,\ell)$, the true rank $r$ and the observed positions $E$ (and not on $\mA$, $m$ or $n$).
\end{Thm}
The proof of Theorem~\ref{Thm:uniquecomp} is a bit more technical than its finite completability analogue, Theorem~\ref{Thm:single-entry-finite-generic}. The main problem is that the constant rank theorem cannot be applied since the latter is a local statement only and does not make say anything about the global number of solutions. The proper tools to overcome that are found in algebraic geometry; a complete proof is deferred to section~\ref{Sec:uniquecompletability.proofs}. The proof we give also shows that there is an analoge statement for the total number of possible completions, even if there is more than one. Since the number of completions over the reals can potentially change even with generic $\mA$, the result is stated only over the complex numbers.

Theorem~\ref{Thm:uniquecomp} shows that it makes sense to talk about positions instead of entries that are uniquely completable, in analogy to the finite case; moreover, it shows that there is a biggest such set:

\begin{Def}
Let $E\subseteq [m]\times [n]$ be the set of observed positions, and let $(k,\ell)\in [m]\times [n]$ be a position. We will call $(k,\ell)$ uniquely completable if $A_{k\ell}$ is uniquely completable from $A_{ij},(i,j)\in E$ for a generic matrix $\mA\in\K^{m\times n}$ of rank $r$.

Furthermore, we will denote by $\uclose_r (E)$ the inclusion-wise maximal set of positions such that every index $(k,\ell)\in \uclose_r (E)$ is uniquely completable from $E$. We will call the $\uclose_r (E)$ \emph{unique closure} of $E$ in rank $r$.
\end{Def}

As for finite completability, we can check generic unique completability of a position by testing a random $\mA$. However, we don't have an analogue for the Jacobian $\mJ_E$ that exactly characterizes unique completability. One could, of course, use general Gr{\"o}bner basis methods, but these are computationally impractical. In the next section, we describe an easy-to-check sufficient condition for unique completability in terms of the Jacobian.

\subsection{Characterization by Jacobian stresses}
As for the case of finite completability, the Jacobian of the masking can be used to provide algorithmic criteria to determine whether an entry is uniquely completable. The characterizing objects will be the so-called
\emph{stresses}, dual objects to the column space of the Jacobian. Intuitively, they correspond to
infinitesimal dual deformations. \citet[Equation~3.7]{Sin10} have conjecturally defined a similar concept.

Mathematically, stresses are left kernels of the Jacobian:

\begin{Def}\label{Def:stress}
A rank-$r$ \emph{stress} of the matrix $\mA= \mU \mV^\top$ is a matrix $\mS\in\C^{m\times n}$ whose vectorization is in the left kernel of the Jacobian $\mJ (\mU,\mV)$; that is,
$$\vect \mS\cdot \mJ (\mU,\mV) = 0.$$
Let $E\subseteq [m]\times [n]$ be a set of observed entries. A stress $\mS$ such that
$\mS_{ij} = 0$ for all $(ij)\not\in E$ is called $E$-\emph{stress} of $\mA$.

The $\C$-vector space of $E$-stresses of $\mA$ will be denoted by $\Psi_\mA (E),$ noting that it does not depend on the choice of $\mU,\mV$. %
\end{Def}

Note that $E$-stresses are, after vectorization and removing zeroes, in the left kernel of the partial Jacobian $\mJ_E$.

The central property of the stress which allows to test for unique completability is its rank as a matrix:

\begin{Def}\label{Def:stressrank}
Let $E\subseteq [m]\times [n]$ be as set of observed entries. We define the \emph{maximal $E$-stress rank}
of $\mA$ in rank $r$ to be
$$\rho_\mA (E) = \max_{\mS\in \Psi_\mA (E)} \rk \mS .$$
\end{Def}

As for the rank of the Jacobian, the dependence on $\mA$ can be removed for generic matrices:
\begin{Prop}\label{Prop:genstressrank}
Let $\mA$ be a generic $(m\times n)$-matrix of rank $r$. The maximal stress rank
$\rho_\mA (E)$ depends only on $E$ and $r$. In particular, $\rho_\mA(E)$ does not depend on the entries of $\mA$.
\end{Prop}
\begin{proof}
Let $\mA = \Upsilon(\mU,\mV)$.  By Cramer's rule, if $\mS\in \Psi_\mA(E)$, the
entries of $\mS$ are rational functions of the entries of $\mU$ and $\mV$.  After
clearing denominators, the proof is similar to that of Lemma~\ref{Lem:generic-rank}.
\end{proof}

We can therefore just talk about the generic $E$-stress rank, omitting again the dependence on the entries $\mA$:
\begin{Def}
Let $E\subseteq [m]\times [n]$ be as set of observed entries. We define the \emph{generic $E$-stress rank} $\rho (E)$ to be equal to $\rho_\mA (E)$ for generic $\mA$ or rank $r$.
\end{Def}

Our main theorem states that if the generic $E$-stress rank is maximal for finitely completable $E$, then $E$ is also uniquely completable:

\begin{Thm}\label{Thm:stressrkuniq}
Let $E\subseteq \calE$. If the generic $E$-stress rank in rank $r$ is $\rho (E) \ge \min (m,n) - r$, then
$\close_r(E) = \uclose_r(E)$.
\end{Thm}
We defer the somewhat technical proof to section~\ref{Sec:uniquecompletability.proofs}.

\subsection{Computing the generic stress rank}

Theorem~\ref{Thm:stressrkuniq} implies that the generic stress rank $\rho(E)$ can be used to certify unique completability of an observation pattern $E$. We explicitly describe the necessary computational steps in Algorithm~\ref{Alg:stressrank}.

\begin{algorithm}[ht]
\caption{\label{Alg:stressrank} Generic stress rank.\newline
\textit{Input:} Observed positions $E\subseteq \calE$.
\textit{Output:} The generic stress rank $\rho (E)$ of $E$ in rank $r$. }
\begin{algorithmic}[1]
\State \label{Alg:stressrank.step1} Randomly sample $\mU\in\R^{m\times r},\mV\in\R^{n\times r}$.
\State \label{Alg:stressrank.step2} Compute $\mJ_E (\mU,\mV)$ with rows
$\mJ_{(i,j)} := (\ve_i \otimes \vv_j^\top, \ve_j \otimes \vu_i^\top),$\\
where $\vu_i$ is the $i$-th row of $\mU$, and $\vv_j$ the $j$-th row of $\mV$.
\State \label{Alg:stressrank.step3} Compute a random vector $\mS\in\R^{\card{E}}$ in the left kernel of $\mJ_E$. Reformat $\mS$ as $(m\times n)$ matrix, where entries with index not in $E$ are zero, and the remaining indices correspond to the row positions in $\mJ_E$.
\State \label{Alg:stressrank.step4} Output $\rho (E) = \rk(\mS).$
\end{algorithmic}
\end{algorithm}

As the algorithm for finite completion, it uses a randomized strategy which allows to compute over the real numbers instead of a field of rational functions by substituting a generic entry. Steps~\ref{Alg:cclosure.step1} and the beginning of step~\ref{Alg:cclosure.step2} are thus analogous as in Algorithm~\ref{Alg:cclosure}. In step~\ref{Alg:cclosure.step2}, the completion matrix $\mJ_E$ is computed, evaluated at the matrices $(\mU,\mV)$. In~\ref{Alg:cclosure.step3}, an evaluated stress $\mS$ is obtained in the left kernel of $\mJ_E$. Its rank, which is computed in step~\ref{Alg:stressrank.step4}, will be the generic stress rank.
Correctness (with probability one) is implied by Proposition~\ref{Prop:genstressrank}.
Also, similar to Algorithm~\ref{Alg:cclosure}, Algorithm~\ref{Alg:stressrank} is a randomized algorithm for which considerations analogue to those in Remark~\ref{Rem:RAM} hold.

\subsection{Proofs}\label{Sec:uniquecompletability.proofs}

\subsubsection{Proof of Theorem~\ref{Thm:uniquecomp}}
\begin{proof}
Consider the algebraic map
\begin{align*}
g: (A_{k\ell}; A_{ij}, (i,j)\in E) \mapsto (A_{ij}, (i,j)\in E)
\end{align*}
By Proposition~\ref{Prop:irreducible} in the appendix, $\Omega$ is a surjective algebraic map of irreducible varieties. Therefore, the generic fiber cardinality $\card{ g^{-1} \circ g (x)}$ for generic $x\in\calX$ does not depend on $x$ by Corollary~\ref{Cor:genericprops}. In particular, whether $1 = \card{ g^{-1} \circ g (x)}$ or not.
\end{proof}

\subsubsection{Proof of Theorem~\ref{Thm:stressrkuniq}}
This sections contains the proof for Theorem~\ref{Thm:stressrkuniq} and some related results.

\begin{Lem}\label{Lem:UVkernel}
Let $\mS\in\C^{m\times n}$ be a stress w.r.t. $m,n,r,\mA=\mU\mV^\top$. Then,
$$\mU^\top \cdot \mS = 0\quad\mbox{and}\quad \mS \cdot \mV = 0$$
(where $0$ denotes the zero matrix of the correct size).
\end{Lem}
\begin{proof}
Since $\mS$ is an stress, it holds by definition that
$\vect \mS\cdot \mJ (\mU,\mV) = 0.$ The statement then follows from Lemma~\ref{Lem:kronecker}.
\end{proof}

Lemma~\ref{Lem:UVkernel} immediately implies a rank inequality:

\begin{Cor}
Let $E\subseteq [m]\times [n]$, assume the true matrix has full rank $r$. Then, it holds that
$\rho (E)\le \min (m,n) - r$
\end{Cor}
\begin{proof}
Keep the notations of Lemma~\ref{Lem:UVkernel}. The statement Lemma~\ref{Lem:UVkernel} implies that for arbitrary $\mS$,
one has $\mS\cdot \mV = 0$. Since $\mV$ is a matrix of full rank $r$, this implies that the null space dimension of
$\mS$ is at least $r$, which is equivalent to the statement by the rank-nullity theorem.
\end{proof}
In keeping with our development of finite completability in terms of $\mJ_{E}$, we have defined stresses
in a way that might depend on the coordinates $(\mU,\mV)$.  In the proof of Theorem \ref{Thm:stressrkuniq},
we will check that this can be removed when necessary.
An alternative but probably less concise approach would be to express the matrix $\mJ_E$ directly in terms of the entries $\mA$.

\subsubsection*{Proof of Theorem \ref{Thm:stressrkuniq}}

We start with a general statement that stresses are invariant over the
pre-image $(\Omega\circ\Upsilon)^{-1}(\Omega(\mA))$.
\begin{Lem}\label{Lem:stress-invariance}
Let $\mA$ be generic, with $\Upsilon(\mU,\mV) = \mA$, $E\subset \calE$, and $\mS$ an $E$-stress.
Then $\mS$ is also an $E$-stress for any $(\mU',\mV')$ with $\Omega\circ\Upsilon(\mU',\mV') = \Omega(\mA)$
\end{Lem}
\begin{proof}
Let $(\mU',\mV')\in (\Omega\circ\Upsilon)^{-1}(\Omega\circ\Upsilon(\mU,\mV))$ be a point different
from $(\mU,\mV)$.  Because $\Omega(\mA)$ is a regular value of the composed map $\Omega\circ\Upsilon$,
the Inverse Function Theorem
provides diffeomorphic neighborhoods $M\ni (\mU,\mV)$ and $N\ni (\mU',\mV')$; let $f : M\to N$ be the
diffeomorphism.

By construction, $df$ is non-singular.  The chain rule then implies
that $(\mJ_{E})_{(\mU',\mV')} = (\mJ_{E})_{(\mU,\mV)}\cdot df^{-1}$, so the left kernels of both
Jacobians are the same. The definition of stress as a vector in the left kernel then proves the lemma.
\end{proof}

\begin{proof}[Proof of Theorem \ref{Thm:stressrkuniq}]
It is clear that $\close_r(E)\supseteq \uclose_r(E)$. Thus we show that
$\close_r(E)\subseteq \uclose_r(E)$.
By Lemma \ref{Lem:stress-invariance}, $\mS$ is a stress for any
$(\mU,\mV)$ that agrees with the observed entries $\Omega(\mA)$ on the
observed positions $E$. Then by Lemma \ref{Lem:UVkernel}, any such pair
$(\mU,\mV)$ must satisfy $\mU^\top\cdot\mS = 0$ and
$\mS\cdot\mV$. Since generically the stress has rank $\min(m,n)-r$,
these equations determine the row and column spans of $\mA$. Once the
row and column spans are fixed, any row or column with at least $r$
observed positions can be uniquely determined. On the other hand, any
row or column with fewer than $r$ observed positions cannot be recovered
(even if the row or column span is known). Therefore we have
$\close_r(E)\subseteq \uclose_r(E)$.
\end{proof}

\section{Local completion}
\label{Sec:closure}
In this section, we connect our theoretical results to the process of reconstructing the
missing entries.  In a nutshell, the idea is that a completable missing entry
$(i,j)\in \calE\setminus E$ is covered by at least one so-called \emph{circuit} in $E\cup\{(i,j)\}$, to which
we can associate \emph{circuit polynomials} which can be used to solve for $A_{ij}$
in terms of the observations, addressing the question \emph{``What is the value of the entry $(i,j)$?''}. Just as in theory where we could separate the reconstructability from the reconstruction, we can obtain a quantitative version of this separation by estimating the entry-wise reconstruction error without actually performing the reconstruction, allowing to give an answer to \emph{``How accurately can one estimate the entry $(i,j)$?''}. We give general algorithms for arbitrary rank, and a closed-form solution for rank one.

\subsection{Circuits as rank certificates}

We start with some concepts from matroid theory.
\begin{Def}\label{Def:circuit}
A set of observed positions $C\subseteq \calE$ is called a \emph{circuit} of rank $r$ if
$\rk_r (C) = \card{C} - 1$ and $\rk_r (S) = \card{S}$ for all proper subsets $S\subsetneq C$.
The graph $G(C)$ is called \emph{circuit graph} of rank $r$.
\end{Def}
A reformulation of Theorem \ref{Thm:single-entry-finite-generic}, in terms of circuits is the following.
\begin{Thm}
The position $(i,j)$ is finitely completable if and only if there is a circuit $C\subset E\cup \{(i,j)\}$
with $(i,j)\in C$.
\end{Thm}
\begin{proof}
See \cite[Lemma 1.4.3]{Oxley}
\end{proof}
The connection to reconstructing missing entries is that every circuit comes with a unique
polynomial:
\begin{Thm}\label{Thm:circpoly}
Let $C\subseteq \calE$ be a circuit in rank $r$, $\Omega_C$ be the mask
corresponding to $C$, and $\mA\in\C^{m\times n}$. There is a unique, up to scalar multiplication, square-free polynomial $\theta_C$ such that:
$\theta_C(\Omega_C(\mA))=0$ if and only if there is $\mA'\in\calM(m\times n,  r)$ and $\Omega_C(\mA)=\Omega_C(\mA')$.
\end{Thm}
\begin{proof}
This follows indirectly from Theorem~1.1 in~\cite{DL87}, or from the discussion in section~5.2 of~\cite{KRT13}
\end{proof}

In other words, circuit polynomials minimally certify for the rank $r$ condition being fulfilled on the entries in $C$.
The simplest example of a circuit is an $(r+1)\times (r+1)$ rectangle in $\calE$.  The associated polynomial is
the determinant of an $(r+1)\times (r+1)$ minor of $\mA$.  Thus, Theorem \ref{Thm:circpoly}
is a generalization of the linear algebra fact that a matrix is rank $r$ if and only if all $(r+1)$-minors vanish.
\begin{Def}
We will call the polynomial $\theta_C$ from Theorem~\ref{Thm:circpoly} a \emph{circuit polynomial} associated to the circuit $C$. Understanding that there are an infinity up to multiplication with a scalar multiple, we will also talk about \emph{the} circuit polynomial when that does not make a difference.
\end{Def}

\begin{Rem}
The circuit polynomial can be interpreted algorithmically as follows: let $C\subseteq \calE$ be a circuit,
assume all entries but one in $C$ are observed, e.g., $(k,\ell)\in C$ is not observed and $E=C\setminus (k,\ell)$
is observed. Then, $\theta_C(\Omega_C(\mA))=\theta_C(A_{k\ell}, \Omega_E(\mA))$ can be interpreted as a polynomial in the one
unknown $A_{k\ell}$. That is, the circuit polynomial allows to solve entry-wise for single missing entries.
\end{Rem}

\begin{Def}
Fix some set of observed entries $E\subseteq \calE$. A circuit $C\subseteq \calE$ is called \emph{completing}
for the observations in $E$, or with respect to $E$, if $\card{C\cap E} \ge \card{C}-1$.
\end{Def}

\subsection{Completion with circuit polynomials}\label{sec:circuits-closure}

The circuit properties inspire a general solution strategy.
\begin{algorithm}[ht]
\caption{\label{Alg:compcirc} Completion with circuits.\newline
\textit{Input:} A set $E\subset \calE$ of observed positions. \newline
\textit{Output:} Estimates for the entries $\close_r(E)\setminus E$}
\begin{algorithmic}[1]
\Repeat
\State \label{Alg:compcirc.step1} Find an unobserved entry
$(k,\ell)\in \close_r(E)\setminus E$,
\State \label{Alg:compcirc.step2} Find the set $\calC = \{C_1,\ldots,C_t\}$
of all circuits (w.r.t. $E$) containing $(k,\ell)$.
\State \label{Alg:compcirc.step3} Compute the circuit polynomials $\theta_{C_i}$.
\State \label{Alg:compcirc.step4} Substitute the entries $\{A_{ij} : (i,j)\in E\}$
into the $\theta_{C_i}$ to get a family of polynomials in the variable $A_{k\ell}$
and find a solution $A_{k\ell}$ common to all of them.
\State \label{Alg:compcirc.step5} $E\leftarrow E\cup (k,\ell)$
\Until $E=\close_r(E)$.
\end{algorithmic}
\end{algorithm}
In general, Algorithm \ref{Alg:compcirc} is ineffective, in the sense that Step \ref{Alg:compcirc.step3}
is unlikely to have a sub-exponential time algorithm in the general case.  However, there is a specific
instance in which it is effective: when the circuit $C$ is always an $(r+1)\times (r+1)$ rectangle.
In this case, the circuit polynomial is the corresponding $(r+1)\times
(r+1)$ minor. This means that enumerating all the circuits through
$(i,j)$ is not necessary, because a minor is linear in the unknown entry $A_{k\ell}$.

A practical algorithm for computing the closure of a mask $E$ and
recovering the corresponding entries based on $(r+1)\times (r+1)$ minors
is given in Algorithm~\ref{Alg:minor}. In Step 5, $N(j)$ and $N(i)$
denote the set of neighbors of vertices $j\in W$ and $i\in V$,
respectively. In Step 10, $A_{I',J'}^{+}$ denotes the Moore-Penrose
pseudoinverse of $A_{I',J'}$.
Intuitively, the algorithm iterates over missing
edges and look if there is a $(r+1)\times(r+1)$ biclique in the union of
current set of edges $E_k$ and $(i,j)$.
If such a biclique exists, then the edge $(i,j)$ is added to
$E_{k+1}$ so that the edge is used in the next round. The iteration
terminates when there is no more edge to add.

\begin{algorithm}
\caption{\texttt{MinorClosure}$((V,W,E),r)$\newline
\textit{Inputs:} bipartite graph $(V,W,E)$, rank $r$.\newline
\textit{Outputs:} completed matrix $A$ and minor closure of $E$.
}
\label{Alg:minor}
\begin{algorithmic}[1]
\State Let $E_0\leftarrow E$ and $k\leftarrow 0$.
\Repeat
\State $E_{k+1}\leftarrow E_{k}$
\For{each missing edge $(i,j)$ in $\calE\backslash E_k$} \label{Alg:minor:foreachmissing}
\State Let $I\leftarrow N(j)\subseteq V$, $J\leftarrow N(i)\subseteq
W$, where the neighbors are defined with respect to graph $(V,W,E_k)$.
\State  $E_k'\leftarrow I\times J \cap E_k$.
\State $(I',J')\leftarrow \texttt{FindAClique}((I,J,E_k'),r,r)$. \label{Alg:minor:findaclique}
\If{$|I'|>0$ and $|J'|>0$}
\State $E_{k+1}\leftarrow E_{k+1}\cup (i,j)$.
\State $A_{ij}\leftarrow A_{i,J'}A_{I',J'}^+A_{I',j}$.
\EndIf
\EndFor
\State $k\leftarrow k+1$.
\Until{$E_{k}=E_{k-1}$ or $E_k=\calE$} \label{Alg:minor:until}
\State Return $(A,E_k)$.
\end{algorithmic}
\end{algorithm}

Note that $E_{k+1}$ is uniquely determined from $E_k$ and the process is
monotone and bounded, i.e., $E_k\subseteq E_{k+1}\subseteq \calE$. The
first statement is true because the order of the iteration over missing edges in line
\ref{Alg:minor:foreachmissing} is irrelevant as we look if
there is a $(r+1)\times (r+1)$ biclique in $E_k\cup{(i,j)}$ for each
missing edge $(i,j)$.
Therefore, Algorithm~\ref{Alg:minor} terminates with either $E_k=\calE$ or
$E_k\subsetneq \calE$ and the following definition is valid.

\begin{Def}\label{Def:minor-closable}
A set $E\subset \calE$ is \emph{minor closable in rank $r$} if Algorithm \ref{Alg:minor}
reconstructs all the entries in positions $\calE\setminus E$. Moreover,
we say $E$ is $k$-step minor closable in rank $r$, if Algorithm
\ref{Alg:minor} terminates with $k$ steps, i.e., $E_k=\calE$ in line \ref{Alg:minor:until}.
\end{Def}

Since each entry is uniquely determined when it is reconstructed, any minor closable set is uniquely
completable.

A crucial step in Algorithm~\ref{Alg:minor} is \texttt{FindAClique} in
line \ref{Alg:minor:findaclique}. The function should return the indices
of rows and columns, if an $r\times r$ biclique exists in subgraph
$(I,J,E')$. This can be achieved in various ways. Although the worst case
complexity is $O(|I|^r|J|^r)$, it can be much more efficient in
practice, because many vertices can be safely pruned due to the fact
that any $r\times r$ biclique may not contain vertices with degree less
than $r$. An efficient implementation that employs a row-wise recursion
of this step, proposed by Takeaki Uno, is presented in Appendix~\ref{Sec:takeakialgo}.

We would like to note that Algorithm~\ref{Alg:compcirc}, as presented above, and all related algorithms below, need the true matrix to be generic. Probabilities for this supposition to hold can be backed out of from Remark~\ref{Rem:RAM}.

\subsection{Local completion}\label{sec:circuits-local}
The circuit property can also be interpreted differently: instead of using multiple ciruits to complete many different entries, one can also think of concentrating on one single entry and trying to reconstruct that as accurately as possible. Algorithm~\ref{Alg:loccomp} describes a general strategy on how to obtain estimates of single finitely or uniquely completable entries, from noisy observations via local circuit completion.

\begin{algorithm}[ht]
\caption{\label{Alg:loccomp} Local completion/denoising of a single entry $(k,\ell)$.\newline
\textit{Input:} A set $E\subset \calE$ of observed positions, the entry. \newline
\textit{Output:} Estimate for $A_{k\ell}$}
\begin{algorithmic}[1]
\State \label{Alg:loccomp.step1} Find completing (w.r.t. $E$) circuits $C_1,\dots, C_N$ containing $(k,\ell)$
\State \label{Alg:loccomp.step2} Compute the circuit polynomials $\theta_{C_i}$, where the observed entries are substituted and $A_{k\ell}$ is the only unknown
\State \label{Alg:loccomp.step3} For all $i$, find all solutions $a^{(i,j)}$ of $\theta_{C_i}$.
\State \label{Alg:loccomp.step4} Return $A_{k\ell} = f(\dots, a^{(i,j)},\dots)$, where $f$ is an appropriate averaging function
\end{algorithmic}
\end{algorithm}
The idea in Algorithm~\ref{Alg:loccomp} is to obtain many candidate solutions in step~\ref{Alg:loccomp.step3} and then trade them off appropriately in step~\ref{Alg:loccomp.step4}. If all circuit polynomials $\theta_{C_i}$ have degree one, there is only one solution per polynomial, and $f$ can be taken as the mean, or a weighted average that minimizes some loss or a variance. If there are some circuit polynomial with higher degree, then one can try to decide which solution is the right one - e.g., by clustering the $a^{(i,j)}$ and rejecting all candidate solutions except the one which contains some $a^{(i,j)}$ for the highest number of $i$, and then proceeding as in the degree one case. Also, one can imagine $f$ being adaptive, e.g. including Bayesian learning methods.

For rank one, an closed explicit form is possible for the variance minimizing estimate, as it was shown in~\cite{KirThe13RankOneEst}.
For arbitrary rank, a first-order approximation to variance minimization can be employed to yield fast and competitive single-entry estimates, by results from~\cite{BlyKirThe14Running}.

For illustration, we give a short overview of the crucial statements in the rank one case. The proofs can be found in~\cite{KirThe13RankOneEst}.

\begin{Thm}\label{Thm:rankonecirc}
The rank one circuit graphs are exactly the simple cycles (bipartite and thus of even length). The corresponding circuit polynomials are all binomials of the form
$$\theta_C = \prod_{\nu=1}^L A_{i_\nu j_\nu} - \prod_{\nu=1}^L A_{i_\nu j_{\nu+1}},$$
where $L$ is an arbitrary number, $i_1,\dots, i_L$ are arbitrary disjoint numbers, and $j_1,\dots, j_L$ are arbitrary disjoint numbers, with the convention that $j_1=j_{L+1}$. The $i_\nu$ and $j_\nu$ do not need to be disjoint from each other.
\end{Thm}

In particular, Theorem~\ref{Thm:rankonecirc} implies that the circuit polynomials are all linear in every occurring variable. Moreover, the specific structure of the problem allows a further simplification:

\begin{Rem}\label{Rem:rkonelin}
Keep the notations of Theorem~\ref{Thm:rankonecirc}. Write $B_{ij}:=\log |A_{ij}|$. Then, the equations
$$L_C = \sum_{\nu=1}^L B_{i_\nu j_\nu} - \sum_{\nu=1}^L B_{i_\nu j_{\nu+1}}$$
vanish on all rank one matrices.
\end{Rem}

With the elementary computation in Remark~\ref{Rem:rkonelin}, matrix completion becomes estimation with linear boundary constraints. That is, the function $f$ in step~\ref{Alg:loccomp.step4} of Algorithm~\ref{Alg:loccomp} could be taken as the least squares regressor of all $B_{k\ell}$ obtained from completing circuits for $(k,\ell)$. The algorithm in~\cite{KirThe13RankOneEst} gives a version which takes different observation variances into account, and efficient graph theoretical observations making the computation polynomial.

We paraphrase this as Algorithm~\ref{Alg:loccomprkone}; more details, e.g. on how to efficiently
find a basis for the set of completing circuits\footnote{This is equivalent to finding a basis for
first $\mathbb{Z}$-homology of the graph $G$, taken as a $1$-complex.}
is efficiently found, or how the kernel matrix $\Sigma$ is constructed, can be found in~\cite{KirThe13RankOneEst}.

\begin{algorithm}[ht]
\caption{\label{Alg:loccomprkone} Local completion/denoising of a single entry $(k,\ell)$ in a rank $1$ matrix.\newline
\textit{Input:} A set $E\subset \calE$ of observed positions, observation variances $\sigma$, the position $(k,\ell)$. \newline
\textit{Output:} Estimate for $A_{k\ell}$}
\begin{algorithmic}[1]
\State \label{Alg:loccomprkone.step1} Find a basis $C_1,\dots, C_N$ for the set
of completing circuits (w.r.t $E$) for $(k,\ell)$
\State \label{Alg:loccomprkone.step2} Find solutions $a_i$ for the corresponding
circuit polynomials, write $b_i:=\log |a_i|$
\State \label{Alg:loccomprkone.step3} Compute the $(N\times N)$-path kernel matrix
$\Sigma=\Sigma(E,\sigma)$ corresponding to the $C_i$; set $\alpha:=\Sigma^{-1} \cdot \mathbf{1}$
\State \label{Alg:loccomprkone.step4} Compute the weighted mean
$b:= \left(\sum_{i=1}^N \alpha_i\cdot a_i\right)/\left(\sum_{i=1}^N \alpha_i\right)$
\State \label{Alg:loccomprkone.step5} As estimate, return
$\widehat{A}_{k\ell}=\pm\exp (b)$, where the sign is determined by the sign parity of the circuits.
\end{algorithmic}
\end{algorithm}

\subsection{Variance and error estimation}\label{sec:variance-local}

The locality of circuits also allows to obtain estimates for the reconstruction error of single missing entries obtained by the strategy in section~\ref{sec:circuits-local}, independent of the method which does the actual reconstruction. The simplest estimate of this kind is obtained from a variational approach: say $\theta_{C}$ is a completing circuit (w.r.t $E\subseteq \calE$) for the missing entry $(k,\ell)$. In the simplest case, where $\theta_C$ is linear in the missing entry $A_{k\ell}$, we can obtain a solving equation
$$\widehat {A}_{k\ell} = \theta_C (A_e, e\in E),$$
by solving for $A_{k\ell}$ as an unknown. A first order approximation for the standard error can be obtained by the variational approach
$$\delta \widehat {A}_{k\ell} = \sum_{e\in E}\frac{\partial\theta_C}{\partial A_e} (A_e, e\in E)\; \delta A_e.$$
The right hand side can be obtained from a suitable noise model and the observations $A_e$, or, if the error should be estimated independently from the $A_e$, from a noise model plus a sampling model for the $A_e$. A general strategy for entry-wise error estimation is analogous to Algorithm~\ref{Alg:loccomp} for local completion. For rank one, it has been shown in~\cite{KirThe13RankOneEst} that the variance estimate depends only on the noise model and not on the actual observation, and takes a closed logarithmic-linear form, as it is sketched in Algorithm~\ref{Alg:varrkone}.

\begin{algorithm}[ht]
\caption{\label{Alg:varrkone} Error prediction for a single entry $(k,\ell)$, rank one.\newline
\textit{Input:} A set $E\subset \calE$ of observed positions, observation variances $\sigma$, the position $(k,\ell)$. \newline
\textit{Output:} Estimate for the (log-)variance error of the estimate $\widehat{A}_{k\ell}$}
\begin{algorithmic}[1]
\State \label{Alg:varrkone.step1} Calculate $\Sigma$ and $\alpha$, as in Algorithm~\ref{Alg:loccomprkone}.
\State \label{Alg:varrkone.step2} As log-variance, return $\alpha^\top \Sigma\alpha$.
\State \label{Alg:varrkone.step3} If an estimate $\widehat{A}_{k\ell}$ is available, as standard error, return $\widehat{A}_{k\ell}\cdot\left(\exp(\alpha^\top \Sigma\alpha) - 1\right)$
\end{algorithmic}
\end{algorithm}

Note that the log-variance error is independent of the actual estimate $\widehat{A}_{k\ell}$, therefore the variance patterns can be estimated without actually reconstructing the entries.

\section{Combinatorial completability conditions}
\label{Sec:graphconditions}
Through Sections \ref{Sec:finitecompletability} and
\ref{Sec:uniquecompletability}, we have shown that for a given $E\subseteq \calE$, both finitely completable closure
(Theorem \ref{Thm:finite-completable-closure}) and uniquely completable
closure (Theorem \ref{Thm:uniquecomp}) are properties of the
(isomorphism type of) the associated bipartite graph $G(E)$; see also Corollary \ref{Cor:graph-matroid}.

In this Section, using tools from graph and matroid theories, we relate
the structural properties of the bipartite graph $G(E)$ to the finite
completability.

For a set of observed positions, $E\subseteq \calE$, let $G(E)=(V,W,E)$
be a bipartite graph, where the sets of vertices $V$ and $W$ correspond
to row and column of the observed positions; we call $V$ and $W$ {\em row
vertices} and {\em column vertices}, respectively. We assume that $G(E)$
has no isolated vertices (those corresponding to rows or columns with no observed positions.)

As usual, we will take $r$, $n$, and $m$ to be the rank and parameters of the
ground set $\calE$, respectively.  However, since our convention for graphs is that
they do not have isolated vertices, we will take care to indicate the ambient
ground set.

\subsection{Sparsity and independence}
Suppose we want to maximize the size of the completable closure $\close_r(E)$,
with the number of positions to observe fixed.  To do this, consider the
process of constructing $E$ one position at a time.  What we need is
to pick each successive entry in a way that causes $\close_r(E)$ to grow.
Theorem \ref{Thm:finite-completable-closure}
implies that a position $(k,l)$ is finitely completable from $E$, if and only if
$\mJ_{\{(k,l)\}}$ lies in the span of $\mJ_{E}$.  In particular, this tells us that
adding such a $(k,\ell)$ to $E$ will not affect the finite completability of other
unobserved positions; in matroid terminology, we say $(k,l)$ is
dependent on $E$.
We see, then, that it is wasteful to choose positions that are dependent
on the already chosen positions. Therefore intuitively we need to choose the
positions so that they are well spread out, which we call {\em rank-$r$
sparse}; see Section \ref{Sec:sparsity}. Rank-$r$ sparsity implies a more
classical combinatorial property, namely $r$-connectivity; see Section
\ref{Sec:connectivity}. Finally, in Section \ref{Sec:sparsitynotsufficient},
we show by a counterexample that
rank-$r$ sparsity, though necessary, is not a sufficient condition for
finite completability.

We recall some basic terminologies from matroid theory.
The rank function $\rk_r(E)$ of the rank $r$ determinantal matroid is
defined in Definition \ref{Def:determinantal-matoid}. Note that
$\rk_r(E)\leq d_r(m,n)$, where $d_r(m,n)=r(m+n-r)$ if
$m\geq r$ and $n\geq r$, $d_r(m,n)=mn$, otherwise. A set of positions $E\subseteq \calE$ is called {\em independent} if $\card{E}=\rk_r(E)$. On
the other hand, it is called {\em dependent} if $\card{E}>\rk_r(E)$.
A basis $B$ of $E\subseteq \calE$ is a maximally independent subset of
$E$. In addition, a basis of $\calE$ is called a basis of the rank $r$
determinantal matroid. A basis $B$ of $E$ consists of $\rk_r(E)$ edges.
In particular, a basis $B$ of the rank $r$ determinantal matroid
consists of $\rk_r(\calE)=d_r(m,n)$ edges. A basis of $E$ is not unique
unless $E$ is independent. A circuit $C\subseteq \calE$ of of the
rank $r$ determinantal matroid is a minimally dependent set in the sense
that for any $(i,j)\in C$, $C-\{(i,j)\}$ is an independent set; see also
Definition~\ref{Def:circuit}.

We have the following two properties from matroid theory.
\begin{Prop}\label{Prop:matroid}
\begin{enumerate}
\item Let $E\subseteq \calE$ be a set of observed positions and
$B\subseteq E$ be any basis of $E$. Then,  $\close_r(B)=\close_r(E)$.
\item Let $E\subseteq \calE$ be an independent set in the rank $r$
determinantal matroid. Then, any $E'\subseteq E$ is independent.
\end{enumerate}
\end{Prop}
In other words, (i)
the finitely completable closures of $E$ and any basis $B$ of $E$ are the
same (ii) and  an independent graph $G(E)$ cannot contain a
dependent subgraph $G(E')$. Both statements arise from the fact that
the rank-$r$ determinantal matroid is a linear matroid defined by
the linear independence of the rows of the Jacobian $\mJ_E$ and that the
matroid closure coincides with the finitely completable closure.

\subsubsection{Rank-$r$-sparsity}
\label{Sec:sparsity}
Let $G'=(V',W',E')$ be a subgraph of $G=(V,W,E)$.
Since $E'$ being independent implies a bound on the
cardinality $\card{E'}\leq d_r(|V'|,|W'|)$, we consider the notion of
\emph{rank-$r$-sparsity} defined as follows.

\begin{Def}\label{Def:rank-r-sparse}
A graph $G = (V,W,E)$ is \emph{rank-$r$-sparse} if, for all subgraphs $G' = (V',W',E')$ of $G$,
$\card{E'} \le d_r(\card{V'},\card{W'})$.
\end{Def}
\begin{Thm}\label{Thm:rank-r-sparse}
Let $E\subseteq \calE$ be an independent set in the rank $r$ determinantal matroid on $[m]\times[n]$.
Then $G(E)$ is rank-$r$-sparse.
\end{Thm}
\begin{proof}
Suppose that there is a subgraph $G'=(V',W',E')$ with $\card{E'}>
d_r(\card{V'},\card{W'})\geq \rk_r(E')$, then this subgraph must be dependent, which
contradicts Proposition \ref{Prop:matroid}, part 2.
\end{proof}

\subsubsection{Connectivity and vertex degrees}
\label{Sec:connectivity}
Rank $r$ sparsity implies some other, more classical, graph theoretic properties in
a straightforward way, since rank-$r$-sparsity is hereditary.
\begin{Cor}\label{Cor:combinatorial-independence-conditions}
Let $m, n > r$, and $E\subseteq \calE$ be the set of observed
positions. If $G(E)$ contains a rank-$r$ sparse subgraph $G(E')$
with $\card{E'}=d_r(m,n)$ edges, then:
\begin{enumerate}
\item $G(E)$ has minimum vertex degree at least $r$.
\item $G(E)$ is $r$-edge-connected.
\end{enumerate}
\end{Cor}
In particular, if $E$ is finitely completable, it contains a basis $E'$
(Proposition \ref{Prop:matroid}, part 1) with $\card{E'}=d_r(m,n)$
edges and $G(E')$ is rank-$r$ sparse. Thus, $E$ is $r$-edge connected.

The proof of the above corollary relies on the following lemma:
\begin{Lem}\label{Lem:sparsedual}
Let $E\subseteq \calE$ be rank-$r$ sparse with $\card{E}=d_r(m,n)$ edges, and $E=\cup_{i=1}^{N}E_i$ be an edge disjoint
partition of $E$. For any set  $E'\subseteq \calE$ of edges
incident to $m'$ row and $n'$ column vertices, we define
$d_r(E'):=d_r(m',n')$. Then we have
$$ d_r(m,n)\leq \sum_{i=1}^{N}d_r(E_i).$$
\end{Lem}
\begin{proof}
By the assumption,
$$ d_r(m,n)= \card{E}=\sum_{i=1}^{N}\card{E_i}\leq\sum_{i=1}^{N}d_r(E_i),$$
where the first equality holds because $E$ is independent and the last
inequality follows from  Theorem \ref{Thm:rank-r-sparse}.
\end{proof}

\begin{proof}[Proof of Corollary \ref{Cor:combinatorial-independence-conditions}]
Since Statement 2 implies statement 1, we prove Statement
2. First, we can assume without loss of generality that $E$ is rank-$r$
sparse and $E'=E$ without loss of generality, because if $E'$ is
$r$-edge-connected, so is $E$.

Consider any partition $V=V_1\cup V_2$ and $W=W_1\cup W_2$. $V_1$ or
$W_1$ can be empty (but not at the same time). This induces an edge disjoint partition $E=E_1\cup
E_2 \cup_{(i,j)\in E-E_1-E_2}\{(i,j)\}$, where $E_1$ and $E_2$ are
sets of edges induced by $(V_1,W_1)$ and $(V_2,W_2)$,
respectively. Treating each edge in $E-E_1-E_2$ as a subgraph, we have
$d_r({(i,j)})=1$. By applying Lemma \ref{Lem:sparsedual}, we have
\begin{align}
\label{eq:bound-bridge}
\card{E-E_1-E_2} \geq d_r(m,n)-d_r(E_1)-d_r(E_2).
\end{align}
Let $m_1:=\card{V_1}$,
$m_2:=\card{V_2}$, $n_1:=\card{W_1}$, and $n_2:=\card{W_2}$. Due to
symmetry, there are three situations that we need to consider. First,
if $m_1,m_2,n_1,n_2\geq r$, $\text{RHS of \eqref{eq:bound-bridge}}=r^2$.
Next, if $m_1\leq r$ and $n_2\leq r$, $\text{RHS of
\eqref{eq:bound-bridge}}=r(m+n-r)-m_1n_1-m_2n_2\geq r^2$, which is true
considering maximizing the inner product between $(m_1,m_2)$ and
$(n_1,n_2)$ subject to $m_1+m_2=m$ and $n_1+n_2=m$. Finally, if
$m_1,n_1\leq r$, $\text{RHS of
\eqref{eq:bound-bridge}}=r(m_1+n_1)-m_1n_1\geq r$. The minimum is
obtained for $m_1=1$ and $n_1=0$, or vice versa.
Therefore $E$ is $r$-edge connected.
\end{proof}

\subsubsection{Sparsity is not sufficient}
\label{Sec:sparsitynotsufficient}
On the other hand, rank $r$ sparsity is \emph{not} a sufficient condition for
independence in determinantal matroids. The bipartite graph defined by
the following mask in rank $2$ have $d_2(5,5)=16$ edges and rank-$2$
sparse but not independent:
\[
\begin{pmatrix}
1 & 1 & 1 & 0 & 0 \\
1 & 1 & 1 & 0 & 0 \\
1 & 1 & 0 & 1 & 1 \\
0 & 0 & 1 & 1 & 1 \\
0 & 0 & 1 & 1 & 1
\end{pmatrix}
\]
This example amounts, graph theoretically, to gluing the graphs of two bases of the
determinantal matroid together along $r$ vertices in a way that preserves rank-$r$-sparsity
but not independence. One can make the construction rigorous to show that, for any $r\ge 2$,
there are infinitely many rank-$r$-sparse dependent sets in the determinantal matroid.

\subsection{Circuit and stress supports}
\label{Sec:circuitstress}
We have discussed stresses in Section~\ref{Sec:uniquecompletability} and
circuits in Section~\ref{Sec:closure}. Here we show that for each
circuit $C$, there is a corresponding stress $\mS$ that is supported on
every position of $C$. Here the support $S\in\calE$ of stress $\mS$ is
defined as $S=\{(i,j)\in\calE:S_{ij}\neq 0\}$.
Moreover, using the structure of the Jacobian matrix (see
Definition \ref{Def:Jacobian}), we show that every vertex of circuit $C$ has degree at least $r+1$.
These results further imply that any finitely completable position spans
vertices in the $r$-core (see Section \ref{Sec:r-core}).
Furthermore, combining the above degree lower bound with the rank-$r$
sparsity shown in the previous subsection, we show a bound on the number
of circuits in the rank $r$ determinantal matroid in Section \ref{Sec:circuit-counting}. The proof of the key
Theorem \ref{Thm:stress-for-circuit} is presented in Section \ref{Sec:proof-stress4circuit}.

\begin{Thm}\label{Thm:stress-for-circuit}
For a generic  $\mA\in\calM(m\times n, r)$, and a circuit $C$, the
stress space $\Psi_\mA (C)$ is one dimensional; thus a stress
$\mS$ of a circuit $C$ is unique up to scalar multiplication. Moreover,
the support of $\mS$ is all of $C$.
\end{Thm}
The power of Theorem~\ref{Thm:stress-for-circuit} can be seen in the following
proposition, which lower bounds the degree of a vertex in a circuit.

\begin{Prop}\label{Prop:deglb}
Let $C\subseteq E$ be a circuit in the  rank $r$ determinantal matroid.
Then every vertex in the graph $G(C)$ has degree at least $r+1$
edges.
\end{Prop}
\begin{proof}
By Theorem~\ref{Thm:stress-for-circuit},
for generic $(\mU,\mV)$, the
rows of $\mJ_C$ are dependent, with the associated stress $\mS$ supported
on all the rows.

From \eqref{eq:jacobian},  we see that any vertex $j$ is associated with exactly
$r$ columns in $\mJ_C$. Let $J$ be the indices of these columns. The
number of non-zero rows in $\mJ_C[\cdot,J]$ is exactly the degree $d$
of $j$. If we suppose $d\leq r$, the stress $\mS$ cannot generically cancel these
$d$ columns. Therefore, it holds that $d\geq r+1$.
\end{proof}

\subsubsection{Where are the completable positions?}
\label{Sec:r-core}
The concept of $k$-core is useful for narrowing down where the completable
positions can be and where the circuits can lie.

We recall a concept from graph theory:
\begin{Def}\label{Def:k-core}
Let $G$ be a graph, and let $k\in\N$.  The \emph{$k$-core} of $G$, denoted $\core_k(G)$,
is the maximal subgraph of $G$ with minimum vertex degree $k$.
\end{Def}
In rank $r$, the non-trivial aspects of matrix completion occur inside the $r$-core.
\begin{Thm}\label{Thm:r-r+1-core}
Let $E\subseteq \calE$,
\begin{enumerate}
\item[(i)]  If $(i,j)\in\calE\setminus E$ and $(i,j)\in \close_r(E)$,
then the vertices $i$ and $j$ are in $\core_r(G(E))$.
\item[(ii)] Any circuit $C\subseteq E$ is contained in $\core_{r+1}(G(E))$.
\end{enumerate}
\end{Thm}
\begin{proof}
(i) We have $(i,j)\in \close_r(E)$ if and only if there is a circuit $C\subseteq E\cup \{(i,j)\}$
with $(i,j)\in C$.  Then $(i)$ will follow from $(ii)$ because for
$G(C)\subseteq \core_{r+1}(G(E))$, we need
$i\in\core_r(G(E))$ and $j\in\core_r(G(E))$. \\
(ii) This follows from the fact that the $(r+1)$-core is the union of all
induced subgraphs with minimum degree at least $r+1$ and by Proposition
\ref{Prop:deglb}, every $C$ lies inside such an induced subgraph.
\end{proof}

Note here that $\uclose_r(E)\subseteq \close_r(E)$, so the same things are true for the uniquely
completable closure.

\subsubsection{Circuit size and counting}
\label{Sec:circuit-counting}
Combining the results in this section, we obtain bounds on the number of circuits in the
rank $r$ determinantal matroid.
\begin{Thm}\label{Thm:circuit-size}
Let $C$ be a circuit in the rank $r$ determinantal matroid with graph $G(C) = (V,W,C)$.
Then $\card{W} \le r(|V| - r) + 1$
\end{Thm}
\begin{proof}
Let $m'=\card{V}$ and $n' = \card{W}$.
Using Proposition \ref{Prop:deglb} for the lower bound and
Theorem \ref{Thm:rank-r-sparse} for the upper bound, we have
$n'(r + 1)\le \card{C}\le r(m' + n' - r) + 1$.  Subtracting $n'r$
from both sides, we get $\card{W} = n'\le r(m'-r) + 1$.
\end{proof}
\begin{Cor}\label{Thm:number-of-circuits}
The number of circuits in the rank $r$ determinantal matroid on $[m]\times [n]$
is at most $2^{mr(m - r) + m}$.
\end{Cor}

\subsubsection{Proof of Theorem \ref{Thm:stress-for-circuit}}
\label{Sec:proof-stress4circuit}
First, by Definition \ref{Def:stress}, $\rk_r(C)=\rk \mJ_C=|C|-1$. Thus
the left null space of $\mJ_C$ is one dimensional.

Next, we explicitly construct a stress $\mS$. By Theorem \ref{Thm:circpoly},
there is a unique polynomial $\theta_C$
for each circuit $C$. Then taking the derivative of $\theta_C$, we have
\begin{align*}
\sum_{(i,j)\in C}\left.\frac{\partial \theta_C}{\partial A_{ij}}\right|_{\Omega_C(\mA)} dA_{ij}=0,
\end{align*}
for any tangent vector $(dA_{ij})_{(i,j)\in C}$ of $\calM(m\times n,r)$
at $\mA$. The vector $(\partial \theta_C/\partial
A_{ij})_{(i,j)\in C}|_{\Omega_C(\mA)}$ is, then, a stress for $C$. In addition, the
coefficient of the stress is uniquely determined by the entries
$\Omega_C(\mA)$.  If any of the coefficients of $(\partial \theta_C/\partial A_{ij})$
were identically zero, we could remove the associated
row $ij$ of $\mJ_C$ and the left-kernel of $\mJ_{C\setminus (i,j)}$ would still be
one-dimensional.  Since this is a contradiction to $C$ being a circuit, we
conclude that none of the coefficients are identically zero.  Since the
coefficients are, in addition, polynomials in $\Omega_C(\mA)$,
each of them is non-vanishing on a Zariski open subset of
$\mathcal{M}(m\times n,r)$.  The (finite) intersection of these
sets is again open, proving that the generic
support of the stress is all of $C$.

\subsection{Completability of random masks}
\label{sec:Theory.RandomMasks}
Up to this point we have considered the completability of a fixed mask, which we
have shown to be equivalent to questions about the associated bipartite graph.
We now turn to the case where the masking is sampled at random, which,
by Corollary~\ref{Cor:graph-matroid}, implies that, generically, this is a question about
\emph{random bipartite graphs}.

\subsubsection{Random graph models}
A \emph{random graph} is a graph valued random variable.  We are specifically
interested in two such models for bipartite random graphs:
\begin{Def}
The \emph{Erd\H{o}s-R\'{e}nyi} random bipartite graph $G(m,n,p)$ is a bipartite graph on $m$ row and $n$ column vertices
vertices with each edge present with probability $p$, independently.
\end{Def}
\begin{Def}
The \emph{$(d,d')$-biregular} random bipartite graph $G(m,n,d,d')$ is the uniform distribution on
graphs with $m$ row vertices, $n$ column ones, and each row vertex with degree $d$ and
each column vertex with degree $d'$.
\end{Def}
Clearly, we need $md = nd'$, and if $m=n$, the $(d,d')$-regular random bipartite graph is,
in fact $d$-regular.

We will call a mask corresponding to a random graph a \emph{random mask}.  We
now quote some standard properties of random graphs we need.
\begin{Prop}\label{Prop:graphprops}
\begin{description}
\item[(Connectivity threshold)] The threshold for $G(m,n,p)$ to become connected, w.h.p., is
$p=\Theta((m + n)^{-1}\log n)$ \cite[Theorem 7.1]{BollobasBook}.
\item[(Minimum degree threshold)] The threshold for the minimum degree in $G(n,n,p)$
to reach $d$ is $p = \Theta((m+n)^{-1}(\log n + d\log\log n + \omega(1)))$.  When
$p = cn$, w.h.p., there are isolated vertices \cite[Exercise 3.2]{BollobasBook}.
\item [($d$-regular connectivity)] With high probability, $G(m,n,d,d')$ is $d$-connected \cite[Theorem 7.3.2]{BollobasBook}.  (Recall
that we assume $m\le n$) .
\item[(Density principle)] Suppose that the expected number of edges in either of our
random graph models is at most $Cn$, for constant $C$.  Then for every $\epsilon>0$,
there is a constant $c$, depending on only $C$ and $\epsilon$ such that, w.h.p.,
every subgraph of $n'$ vertices spanning at least $(1+\epsilon)n'$ edges has $n'\ge cn$ \cite[Lemma 5.1]{JL07}.
\item[(Emergence of the $k$-core)] Define the $k$-core of a graph to be the maximal
induced subgraph with minimum  $k$.  For each $k$, there is a constant $c_k$ such that
$p = c_k/n$ is the first-order threshold for the $k$-core to emerge.  When the $k$-core
emerges, it is giant and afterwards its size and number of edges spanned grows smoothly with $p$ \cite{PSW96}.
\end{description}
\end{Prop}

\subsubsection{Sparser sampling and the completable closure}
The lower bounds on sample size for completion of rank $r$ \emph{incoherent}
matrices do \emph{not} carry over verbatim to the generic setting
of this paper.  This is because genericity and incoherence
are related, but incomparable concepts: there are generic
matrices that are not incoherent (consider a very small perturbation of the
identity matrix); and, importantly, the block diagonal examples showing the
lower bound for incoherent completability are not generic, since many of the
entries are zero.

Thus, in the generic setting, we expect sparse sampling to be more powerful.
This is demonstrated experimentally in Section~\ref{sec:Exps.PhaseTr}. In the rest
of this section, we derive some heuristics for the expected generic
completability behavior of sparse random masks.  We are particularly
interested in the question of: \emph{when are $\Omega(mn)$ of the
entries completable from a sparse random mask}?
We call this the \emph{completability transition}.
We will conjecture that
there is a sharp threshold for the completability transition, and that the threshold
occurs well below the threshold for $G(n,m,p)$ to be completable.

Let $c$ be a constant.  We first consider  the emergence of a
circuit in $G(n,n,c/n)$.  Theorem~\ref{Thm:r-r+1-core} implies that any circuit
is a subgraph of the $(r+1)$-core.  By Theorem \ref{Thm:finite-completable-closure}
and Proposition \ref{Prop:matroid}, having a circuit is
a monotone property, which occurs with probability one for graphs with more than $2rn$
edges, and thus the value
\[
t_r := \sup \{ t : \text{$G(n,n,t/n)$ is $r$-independent, w.h.p.} \}
\]
is a constant.  If we define $C_r$ as
\[
C_r := \sup \{ c : \text{the $(r+1)$-core of $G(n,n,c/n)$ has average degree at most $2r$, w.h.p.} \}
\]
smoothness of the growth of the $(r+1)$-core implies that we have
\[
c_{r+1} \le t_r \le C_{r+1}
\]
where we recall that $c_{r+1}$ is the threshold degree for the $(r+1)$-core to emerge.
Putting things together we get:
\begin{Prop}\label{Prop:circthresh}
There is a constant $t_r$ such that, if $c < t_r$ then w.h.p., $G(n,n,c/n)$ is
$r$-independent, and, if $c > t_r$ then w.h.p. $G(n,n,c/n)$ contains a
giant $r$-circuit inside the $(r+1)$-core.  Moreover, $t_r$ is at most the
threshold for the $(r+1)$-core to reach average degree $2r$.
\end{Prop}
Proposition \ref{Prop:circthresh} gives us some structural information about where to
look for rank $r$ circuits in $G(n,n,c/n)$: they emerge suddenly inside of the $(r+1)$-core
and are all giant when they do.  If rank $r$ circuits were themselves completable, this would then yield a threshold for the completability transition.
Unfortunately, the discussion in Section \ref{Sec:sparsitynotsufficient} tell us
that this is not always true.  Nonetheless, we conjecture:
\begin{Conj}\label{Conj:completablecomponent}
The constant $t_r$ is the threshold for the completability transition in
$G(n,n,c/n)$.  Moreover, we conjecture that almost all of the $(r+1)$-core is
completable above the threshold.
\end{Conj}
We want to stress that the conjecture includes a conjecture about the
\emph{existence} of the threshold for the completabilty transition, which hasn't been
established here, unlike the existence for the emergence of a circuit.  The subtlety is
that we haven't ruled out examples of $r$-independent graphs with no rank-$r$-spanning
subgraph for which, nonetheless, the closure in the rank $r$ completion matroid is giant.  Conjecture
\ref{Conj:completablecomponent} is explored experimentally in Sections~\ref{sec:Exps.randomdiff}
and~\ref{sec:Exps.PhaseTr}.  The conjectured behavior is analogous to what has been proved for
distance matrices (also known as \emph{bar-joint frameworks}) in dimension $2$ in \citep{KMT11}.

Our second conjecture is about $2r$-regular masks.
\begin{Conj}\label{Conj:2r}
With high probability $G(n,n,2r,2r)$ is completable.  Moreover, we conjecture that
it remains so, w.h.p., after removing $r^2$ edges uniformly at random.
\end{Conj}
We provide evidence in Section~\ref{sec:Exps.PhaseTr}. This behavior is strikingly different than
the incoherent case, and consistent with proven results about $2$-dimensional distance matrices~\citep[Theorem 4.1]{JSS07}.

\subsubsection{Denser sampling and the $r$-closure}
The conjectures above, even if true, provide only information about matrix \emph{completability} and
not matrix \emph{completion}.  In fact, the convex relaxation of \citet{CanRec09} does not seem to
do very well on $2r$-regular masks in our experiments,  %
and the density principle for sparse
random graphs implies that, w.h.p., a $2r$-regular mask has no dense enough subgraphs for our
closability algorithm in section~\ref{sec:Algos.closability.graph} to even get started.  Thus it seems possible that
these instances are quite ``hard'' to complete even if they are known to be completable.

If we consider denser random masks, then the closability algorithm becomes more practical.  A
particularly favorable case for it is when every missing entry is part of some $K^-_{r+1,r+1}$.
In this case, the error propagation will be minimal and, heuristically, finding a $K^-_{r+1,r+1}$
is not too hard, even though the problem is NP-complete in general.

Define the \emph{$1$-step $r$-closure} of a bipartite graph $G$ as the graph $G'$ obtained by adding the missing
edge to each $K^-_{r+1,r+1}$ in $G$.  If the $1$-step closure of $G$ is $K_{n,n}$, we define $G$ to be
\emph{$1$-step $r$-closable}. We conjecture an upper bound on the threshold for $1$-step $r$-closability.
\begin{Conj}\label{Thm:1step}
There is a constant $C > 0$ such that, if $p = Cn^{-2/(r+2)}\log n$ then, w.h.p.,
$G(n,n,p)$ is $1$-step $r$-closable.
\end{Conj}

\section{Experiments}
\label{Sec:experiments}
In this section we will investigate the set of entries that are
finitely completable from a set of given entries. In
section~\ref{Sec:finitecompletability} we have seen that the finitely
completable closure $\close_r(E)$ does not depend on the values of the
observed entries but only on their positions $E$. First, we check the set of completable entries for synthetic random positions and empirically investigate the completability phase transitions in terms of the number of known entries, as described in Section~\ref{sec:Theory.RandomMasks}. We also check the number of completable entries for MovieLens data set in terms of the putative rank. Then, we present experiments on actual reconstruction and algorithm-independent error estimation in the case of rank one matrices.

\subsection{Randomized algorithms for completability}
\label{sec:Exps.randomdiff}

For a quantitative analysis, we perform experiments to investigate how the expected number of completable entries is influenced by the number of known entries. In particular, section~\ref{sec:Theory.RandomMasks} suggests that a phase transition between the state where only very few additional entries can be completed and the state where a large set of entries can be completed should take place at some point. Figure~\ref{fig:compcore} shows that this is indeed the case when slowly increasing the number of known entries: first, the set of completable entries is roughly equal to the set of known entries, but then, a sudden phase transition occurs and the set of completable entries quickly reaches the set of all entries.

\begin{figure}[ht]
\begin{center}
\subfigure[Results for $m=15,n=15,r=2$]{%
\label{fig:compcore1515r2}
\includegraphics[height=0.4\textwidth]{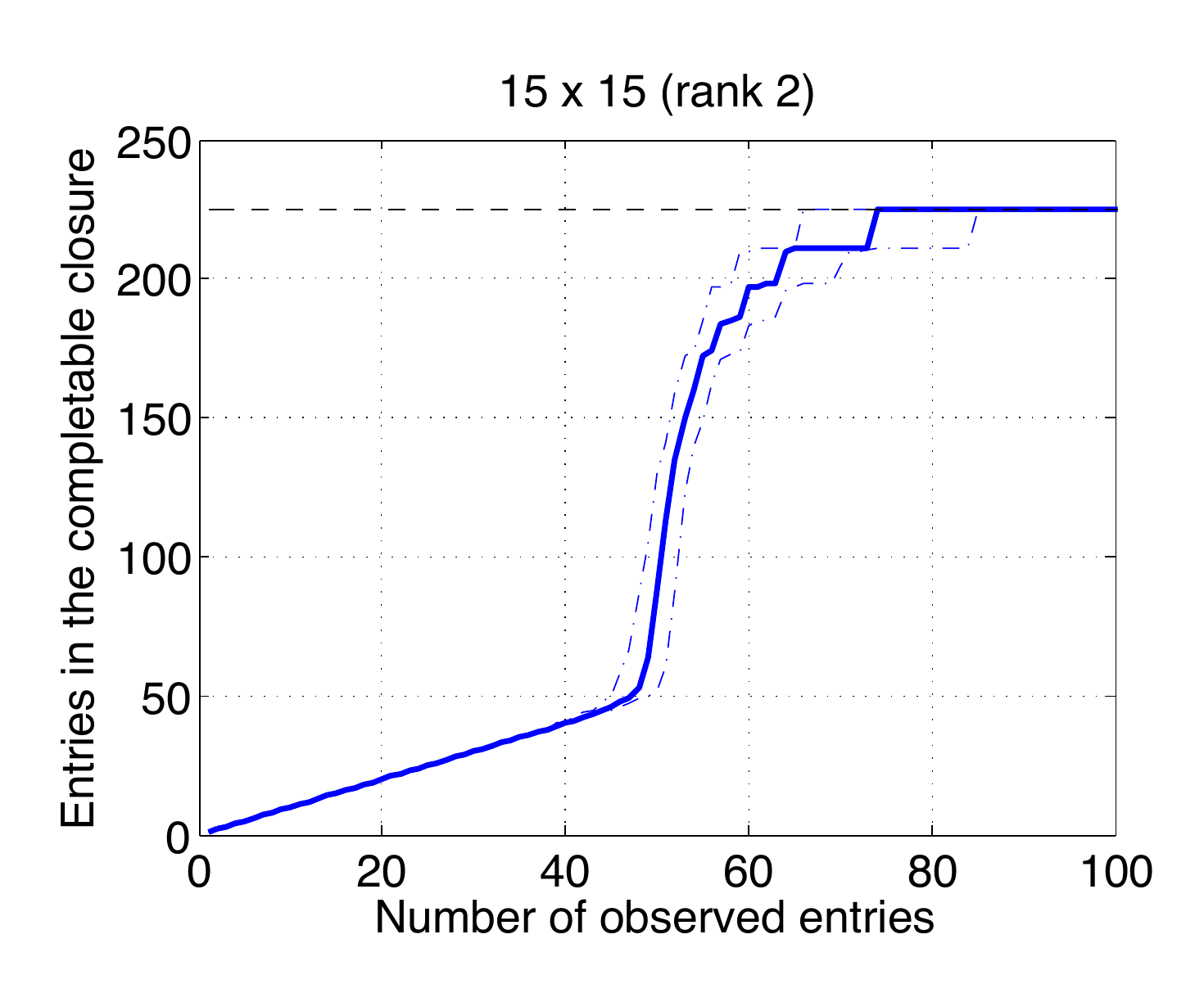}
}~\subfigure[Results for $m=20,n=20,r=5$]{%
\label{fig:compcore2020r5}
\includegraphics[height=0.4\textwidth]{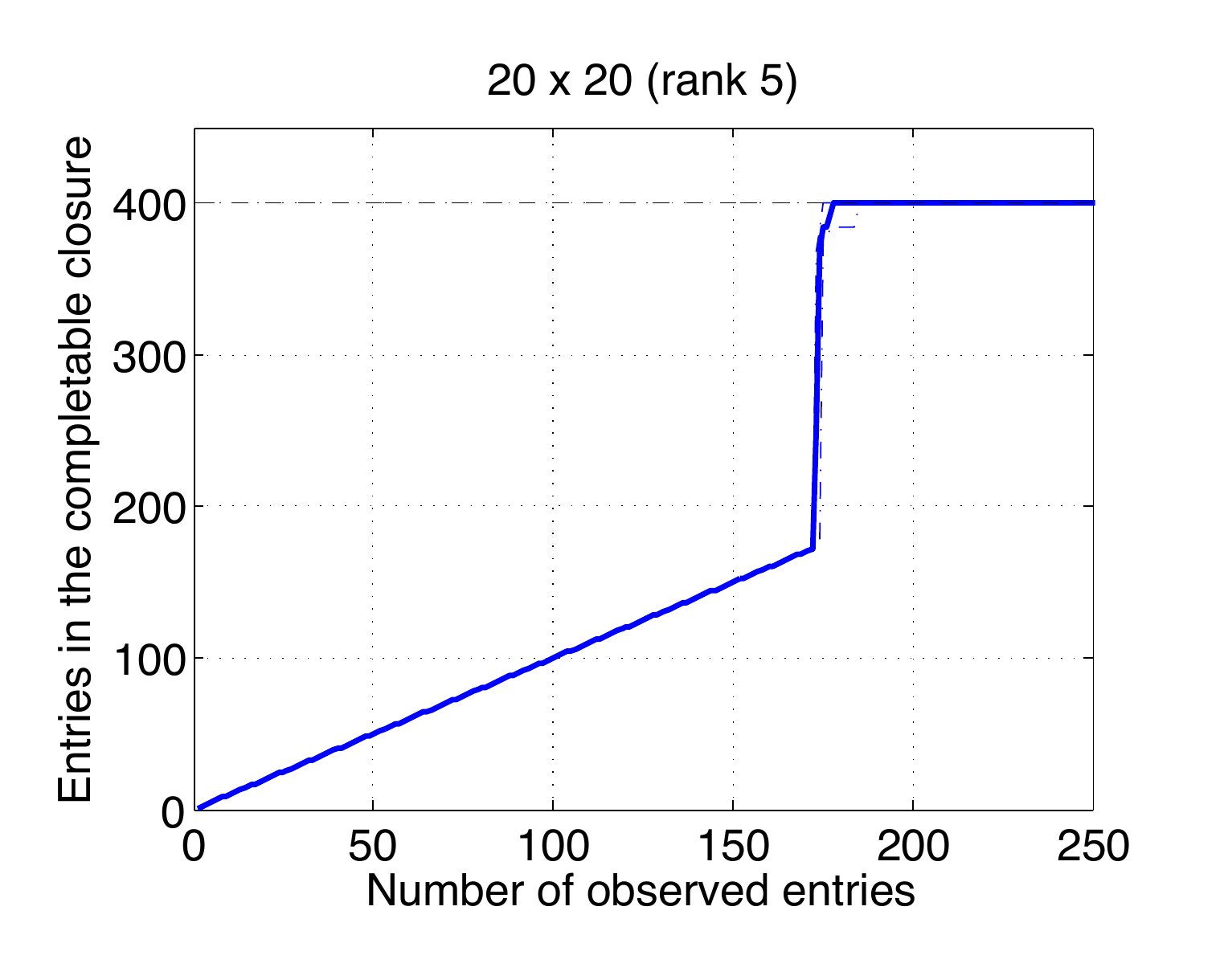}
}

\end{center}
\caption{
Expected number of completable entries (in rank $r$) versus the number of known entries where the positions of the known entries are uniformly randomly sampled in an $(m\times n)$-matrix. The expected number of completable entries was estimated for each data points from repeated calculations of the completable closure ($200$ for $r=2$, and $20$ for $r=5$). The blue solid line is the median, the blue dotted lines are the two other quartiles. The black dotted line is the total number of entries, $m\cdot n$.
}
\label{fig:compcore}
\end{figure}

\subsection{Phase transitions}
\label{sec:Exps.PhaseTr}

Figure~\ref{fig:conditions_100_100_3} shows phase transition curves of
various conditions for $100\times 100$ matrices at rank 3. We consider
uniform sampling model here. More specifically, we generated random
$100\times 100$ masks with various number of edges by
first randomly sampling the order of edges (using MATLAB
\texttt{randperm} function) and adding 100 entries at a
time from 100 to 6000 sequentially. In this way, we made sure to preserve the
monotonicity of the properties considered here. This experiment was
repeated 100 times and averaged to obtain estimates of success probabilities.
The conditions plotted are (a) minimum degree at least $r$, (b)
$r$-connected, (c) completable at rank $r$, (d) minor closable in rank
$r$ (e) nuclear norm successful, and (f) one-step minor closable. For nuclear norm minimization (e),
we used the implementation of the algorithm in~\citep{THK10} which solves the
minimization problem
\begin{align*}
\hat{\mX}= {\rm arg}\min_{\mX}\|\mX\|_{\ast}\quad\text{subject to}\quad X_{ij}=A_{ij}\quad\forall (i,j)\in E,
\end{align*}
where $\|\mX\|_{\ast}=\sum_{j=1}^{r}\sigma_j(\mX)$ is the nuclear norm of
$X$. The success of nuclear norm minimization is defined as the relative
error $\|\hat{\mX}-\mA\|_F/\|\mA\|_F$ less than 0.01.

The success probabilities of the (a) minimum degree, (b) $r$-connected,
and (c) completable are almost on top of each other, and
exceeds chance (probability 0.5) around $|E|\simeq 1,000$. The success probability of the (d)
minor closable curve passes through 0.5 around $|E|\simeq
1,300$. Therefore the $r$-closure method is nearly optimal. On the other hand, the nuclear norm minimization required about
$2,200$ entries to succeed with probability larger than 0.5.

Figure~\ref{fig:conditions_100_100_6} shows the same plot as above for
$100\times 100$ matrices at rank 6. The success probabilities of the (a)
minimum degree, (b) $r$-connected, (c) completable are again
almost the same, and exceeds chance probability 0.5 around $|E|\simeq
1,400$. On the other hand, the number of entries required for
minor closability is at least $3,700$. This is because the masks that we
need to handle around the optimal sampling density is so large
and sparse that we cannot hope to find a $6\times 6$ biclique required by
the minor clusre algorithm to even get started. The nuclear norm minimization required about $3,100$ samples.

\begin{figure}[tb]
\begin{center}
\includegraphics[width=.7\textwidth,clip]{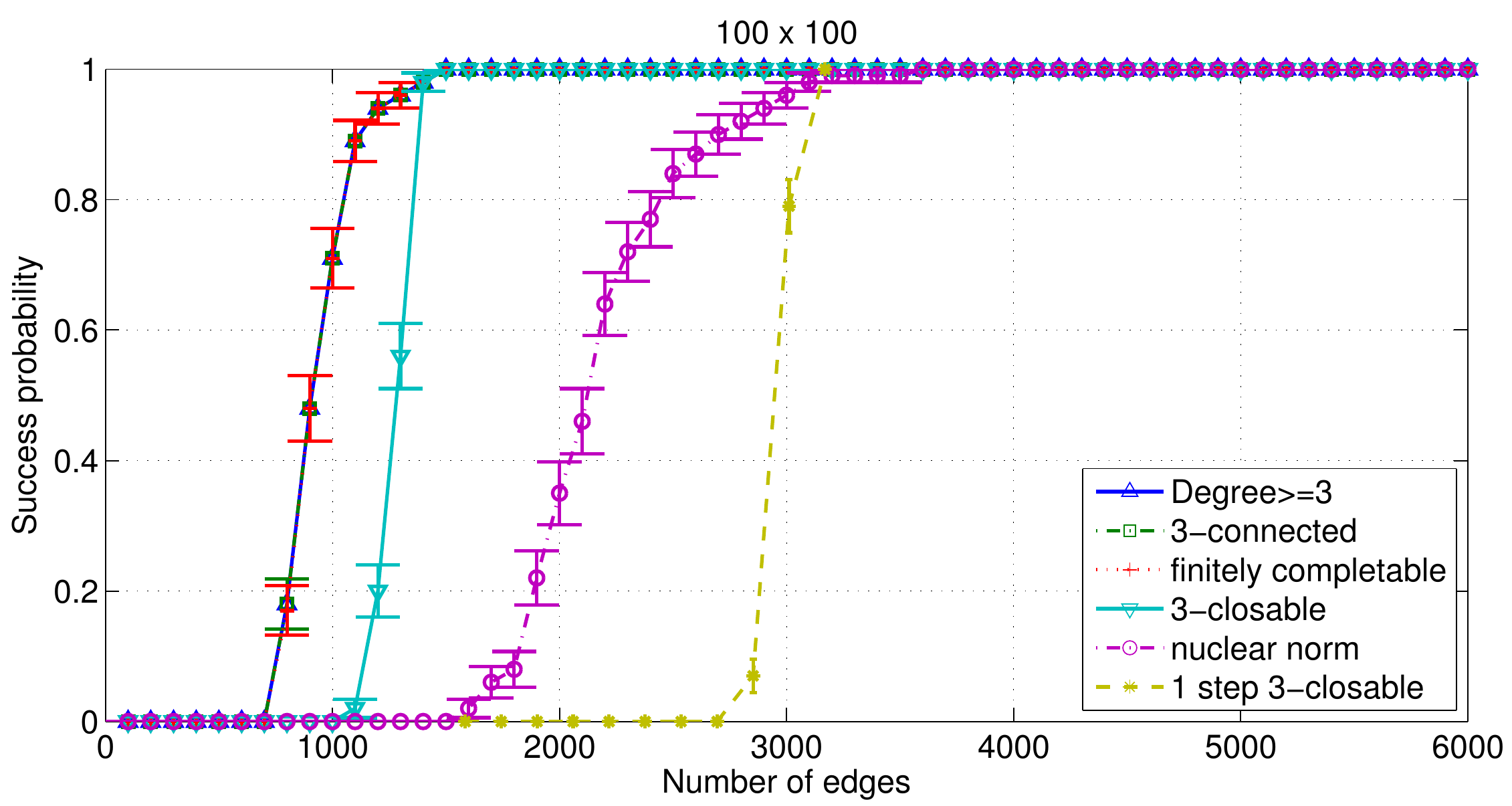}
\caption{Phase transition curves of various conditions for $100\times
100$ matrices at rank 3.}
\label{fig:conditions_100_100_3}
\end{center}
\end{figure}

\begin{figure}[tb]
\begin{center}
\includegraphics[width=.7\textwidth,clip]{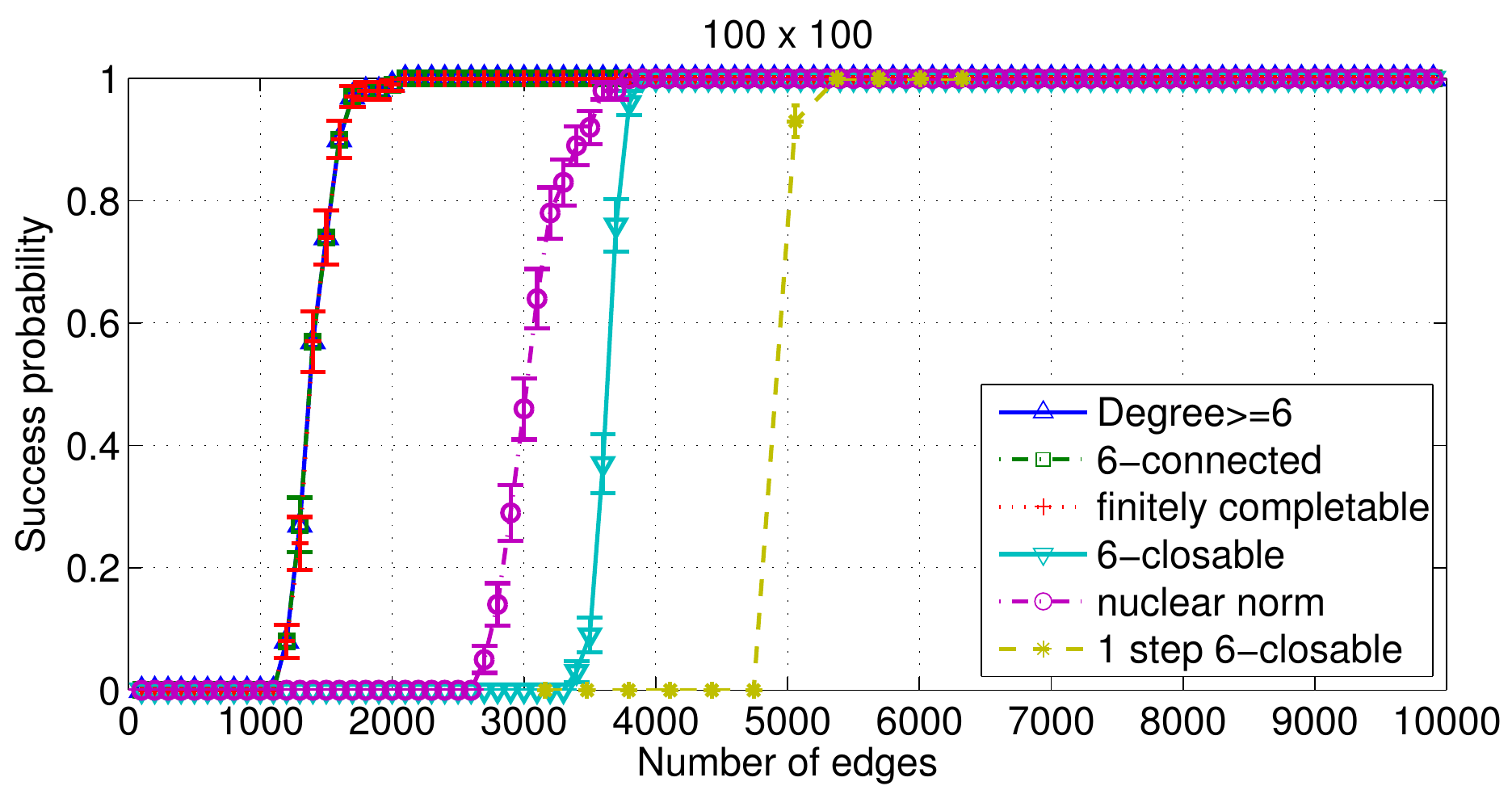}
\caption{Phase transition curves of various conditions for $100\times
100$ matrices at rank 6.}
\label{fig:conditions_100_100_6}
\end{center}
\end{figure}

Figure~\ref{fig:regular} shows the phase transition from a
non-completable mask to a completable mask for almost $2r$-regular
random masks. Here we first randomly sampled $n\times n$ $2r$-regular
masks using Steger \& Wormald algorithm~\citep{SteWor99}. Next we randomly permuted the
edges included in the mask and the edges not included in the mask
independently and concatenated them into a single list of edges. In this
way, we obtained a length $mn$ ordered list of edges that become $2r$-regular
exactly at the $2rn$th edge. For each ordered list sampled this way, we
took the first $2rn+i$ edges and checked whether the mask corresponding to
these edges was completable for $i=-15,-14,\ldots,5$.
This procedure was repeated 100 times and averaged to obtain a probability estimate.
In order to make sure that the phase transition is indeed caused by the
regularity of the mask, we conducted the same experiment with row-wise
$2r$-regular masks, i.e., each row of the mask contained exactly $2r$
entries while the number of non-zero entries varied from a column to another.

In Figure~\ref{fig:regular}, the phase transition curves for
different $n$ at rank 2 and 3 are shown. The two plots in the top part show the
results for the $2r$-regular masks, and the two plots in the bottom show the
same results for the $2r$-row-wise regular masks. For the $2r$-regular
masks, the success probability of completability sharply rises
when the number of edges exceeds $2rn-r^2$ ($i=-4$ for $r=2$ and $i=-9$
for $r=3$); the phase transition is already rather sharp for $n=10$ and
for $n\geq 20$ it becomes almost zero or one. On the other hand, the
success probabilities for the $2r$-row-wise regular masks grow rather
slowly and approach zero for large $n$. This is natural, since it is
likely for large $n$ that there is some column with non-zero entries
less than $r$, which violates the necessary conditions in Corollary~\ref{Cor:combinatorial-independence-conditions}.

\begin{figure}[tb]
\begin{center}
\includegraphics[width=.7\textwidth,clip]{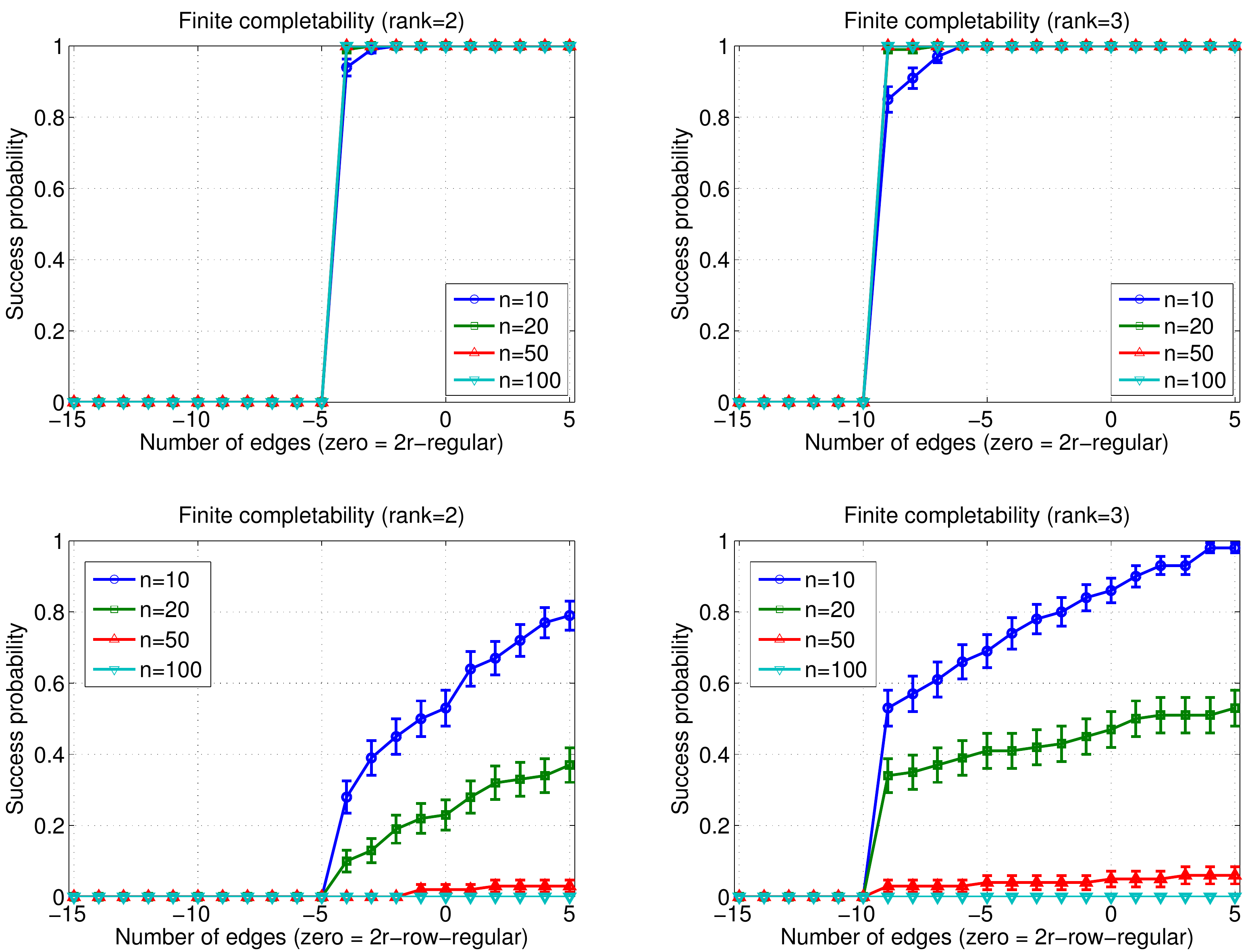}
\caption{Phase transition in an almost regular mask.}
\label{fig:regular}
\end{center}
\end{figure}

\subsection{Completability of the MovieLens data set}
\label{sec:Exps.Movielens}
This section is devoted to studying a well-known data set - the
MovieLens data published by \citet{movielensdata} - with the methods developed in this paper. We demonstrated how the algorithms given above can be used to make statements about the sets of entries which are (a) completable, (b) uniquely completable, and (c) not completable with any algorithm.

The underlying data set for the following analyses is the MovieLens 100k
data set. By convention, columns will correspond to the $1682$ movies, while the rows will correspond to the $943$ users in the data set.

For growing rank $r$, the $r$-core of the MovieLens data set was
computed by the algorithm which is standard in graph theory - by Theorem~\ref{Thm:r-r+1-core} only the missing entries in the $r$-core can be completed, and any entry not contained in the $r$-core is not completable by any algorithm. Figure~\ref{fig:movielens-rcore} shows the size (columns, rows, entries) of the $r$-core of the MovieLens data for growing $r$.

\begin{figure}[tb]
\begin{center}
\includegraphics[width=.7\textwidth,clip]{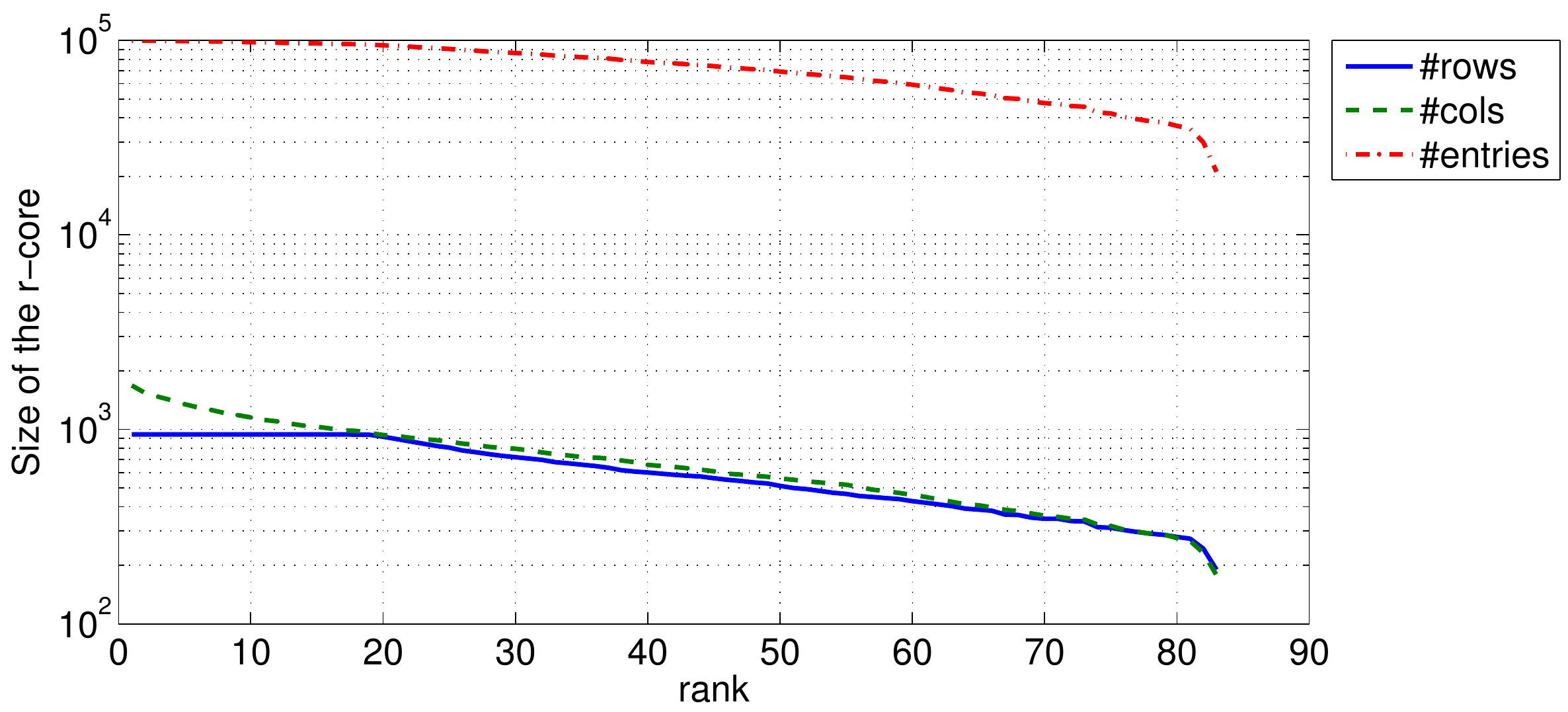}
\caption{Size of the $r$-core of the MovieLens 100k data set for
varying $r$. For each rank $r$, the figure shows the number of rows
(solid blue), the number of columns (dashed green), and the number of
entries (dash-dotted red) in the $r$-core of the mask corresponding to the observed entries of the MovieLens 100k data set. The biggest rank with non-empty $r$-core is $r=83$.}
\label{fig:movielens-rcore}
\end{center}
\end{figure}

Under rank $18$, the vast majority of the entries are in $r$-core, and so is the majority of the rows, while some columns with very few entries are removed with increasing $r$. At rank $r=18$, the number of columns in the $r$-core attains the number of rows in the $r$-core; above rank $18$, the number of rows and columns in the $r$-core diminish exponentially with the same speed. Above rank $79$, the $r$-core rapidly starts to shrink, with $r=83$ being the biggest rank with non-empty $r$-core.

For growing rank $r$, finitely completable closure $\close_r(E)$ in the
MovieLens data set were identified in the following way: First, it was
checked with Algorithm~\ref{Alg:cclosure} whether the $83$-core was
$r$-completable. If not, the completable entries in the $83$-core were
computed by an implementation of Algorithm~\ref{Alg:cclosure}. Then, the
minor closure of the completed $83$-cores was computed by Algorithm~\ref{Alg:minor}; by Theorem~\ref{Thm:r-r+1-core}, it was sufficient to check for completable entries in the $r$-core. Note that the positions of the completable entries were also computed in the process.

Figure~\ref{fig:movielens-completable} shows the number of completable entries in the MovieLens data set for growing $r$ determined in this way. %

\begin{figure}[tb]
\begin{center}
\includegraphics[width=.7\textwidth,clip]{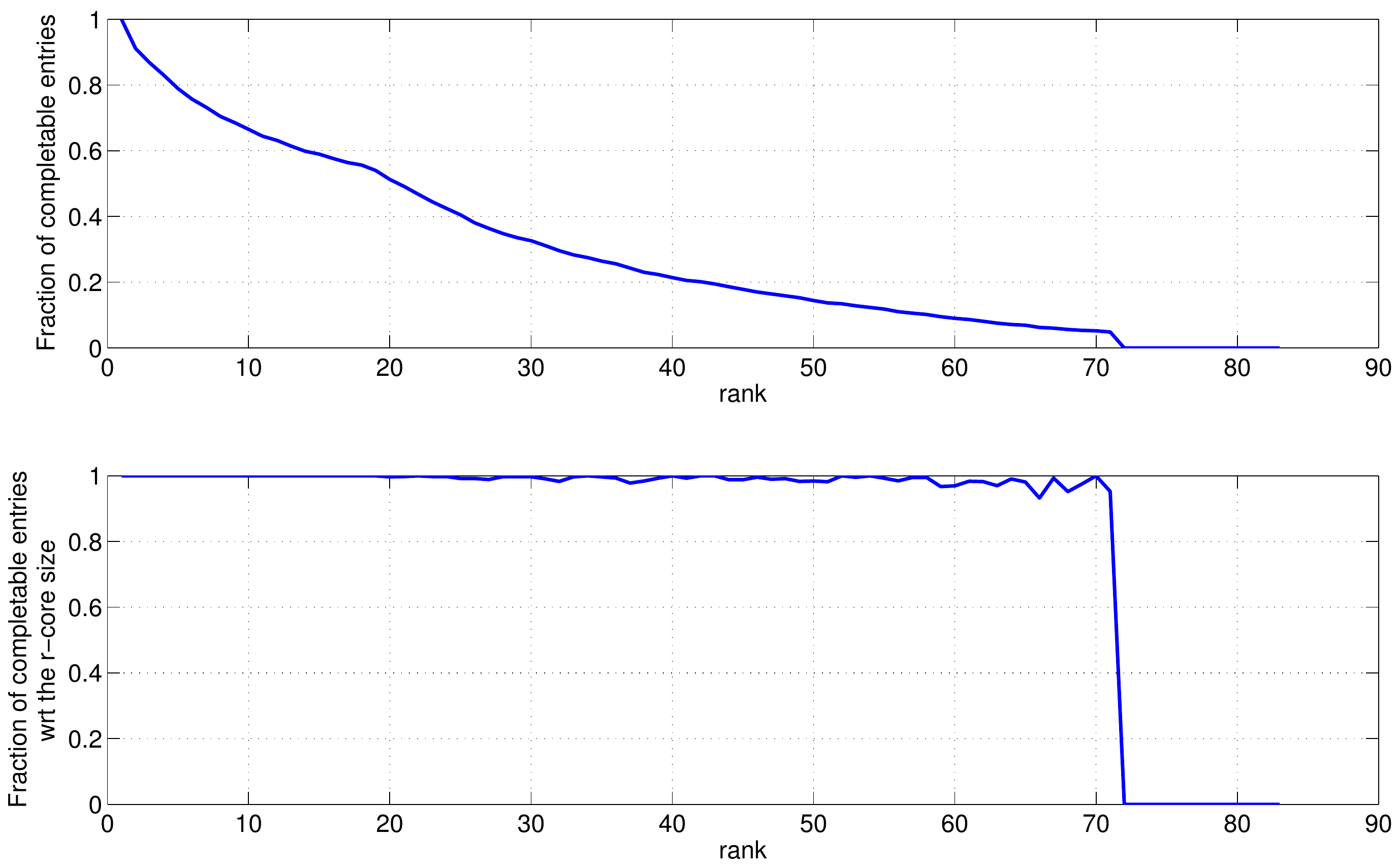}
\caption{Number of completable entries in the MovieLens 100k data set for varying $r$; observed entries are not counted as completable, only completable entries which are not observed. For each rank $r$, the upper figure shows the number of completable entries, as a fraction of all missing entries. The lower figure shows the number of completable entries, as a fraction of the missing entries in the $r$-core. For $r\ge 84$, the $r$-core is empty, thus no missing entries can be completed, see Figure~\ref{fig:movielens-rcore}.}
\label{fig:movielens-completable}
\end{center}
\end{figure}

An interesting thing to note is the inflection point at rank $r=18$. It corresponds to the phase transition in Figure~\ref{fig:movielens-rcore} where the $r$-core starts to shrink exponentially and simultaneously in rows and columns. At rank $r=72$ and above, no missing entry in the $83$-core can be completed.

\subsection{Entry-wise completion and error prediction}
In the rest of the experiments, we recapitulate some results from~\cite{KirThe13RankOneEst} on entry-wise reconstruction and error prediction for rank one matrices.

To test reconstruction, we generated $10$ random masks of size $50\times 50$ with $200$ entries sampled uniformly and a
random $(50\times 50)$ matrix of rank one.  The multiplicative noise was chosen entry-wise independent, with variance $\sigma_{i} = (i-1)/10$ for each entry. Figure \ref{fig:3algs.noisemse} compares the Mean Squared Error (MSE) for three algorithms: Nuclear Norm (using the implementation \citet{THK10}), OptSpace \citep{KMO10}, and Algorithm~\ref{Alg:loccomprkone}. It can be seen that on these masks, Algorithm~\ref{Alg:loccomprkone} is competitive with the other methods and even outperforms them for low noise.

Figure~\ref{fig:3algs.linearity} compares the error of each of the methods with the variance predicted by Algorithm~\ref{Alg:varrkone} each time the noise level changed. The figure shows that for any of the algorithms, the mean of the actual error increases with the predicted error, showing that the error estimate is useful for a-priori prediction of the actual error  - independently of the particular algorithm. Note that by construction of the data this statement holds in particular for entry-wise predictions. Furthermore, in quantitative comparison Algorithm~\ref{Alg:varrkone} also outperforms the other two in each of the bins. The qualitative reversal between the algorithms in Figures~\ref{fig:3algs.linearity}~(a) and~(b) comes from the different error measure and the conditioning on the bins.

\subsection{Universal error estimates}
For three different masks, we calculated the predicted minimum variance for each entry of the mask. The mask
sizes are all $140\times 140$. The noise was assumed to be i.i.d.~Gaussian multiplicative with $\sigma_e=1$ for each entry. Figure~\ref{fig:heatmaps} shows the predicted a-priori minimum variances for each of the masks. The structure of the mask affects the expected error. Known
entries generally have least variance, and it is less than the initial variance of $1$, which implies that
the (independent) estimates coming from other paths can be used to successfully denoise observed data.  For
unknown entries, the structure of the mask is mirrored in the pattern of the predicted errors; a diffuse mask gives a
similar error on each missing entry, while the more structured masks have structured error which is determined by
combinatorial properties of the completion graph.

\begin{figure*}[ht]
\begin{center}
\subfigure{%
\includegraphics[height=0.12\textwidth]{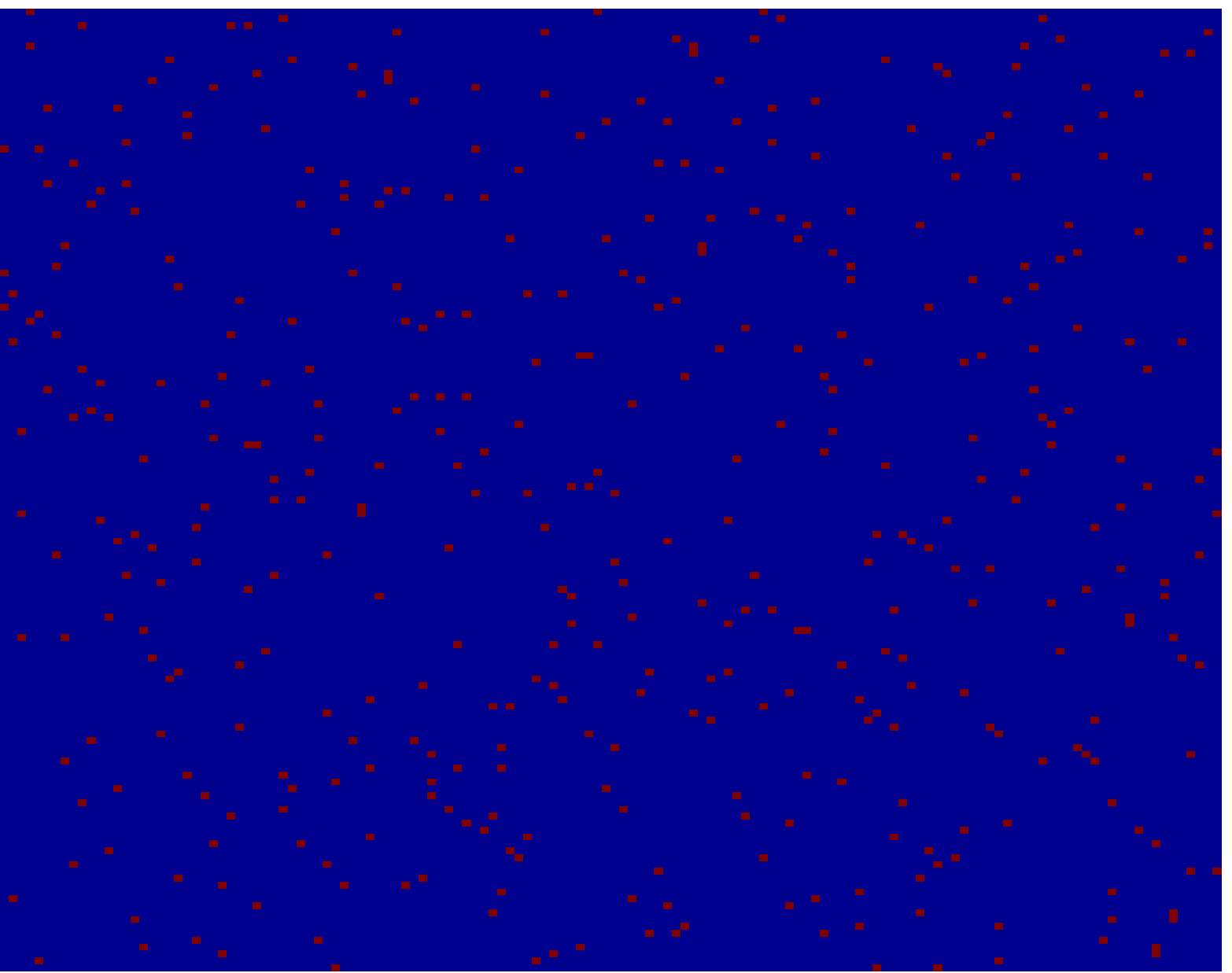}
}
\subfigure{%
\includegraphics[height=0.12\textwidth]{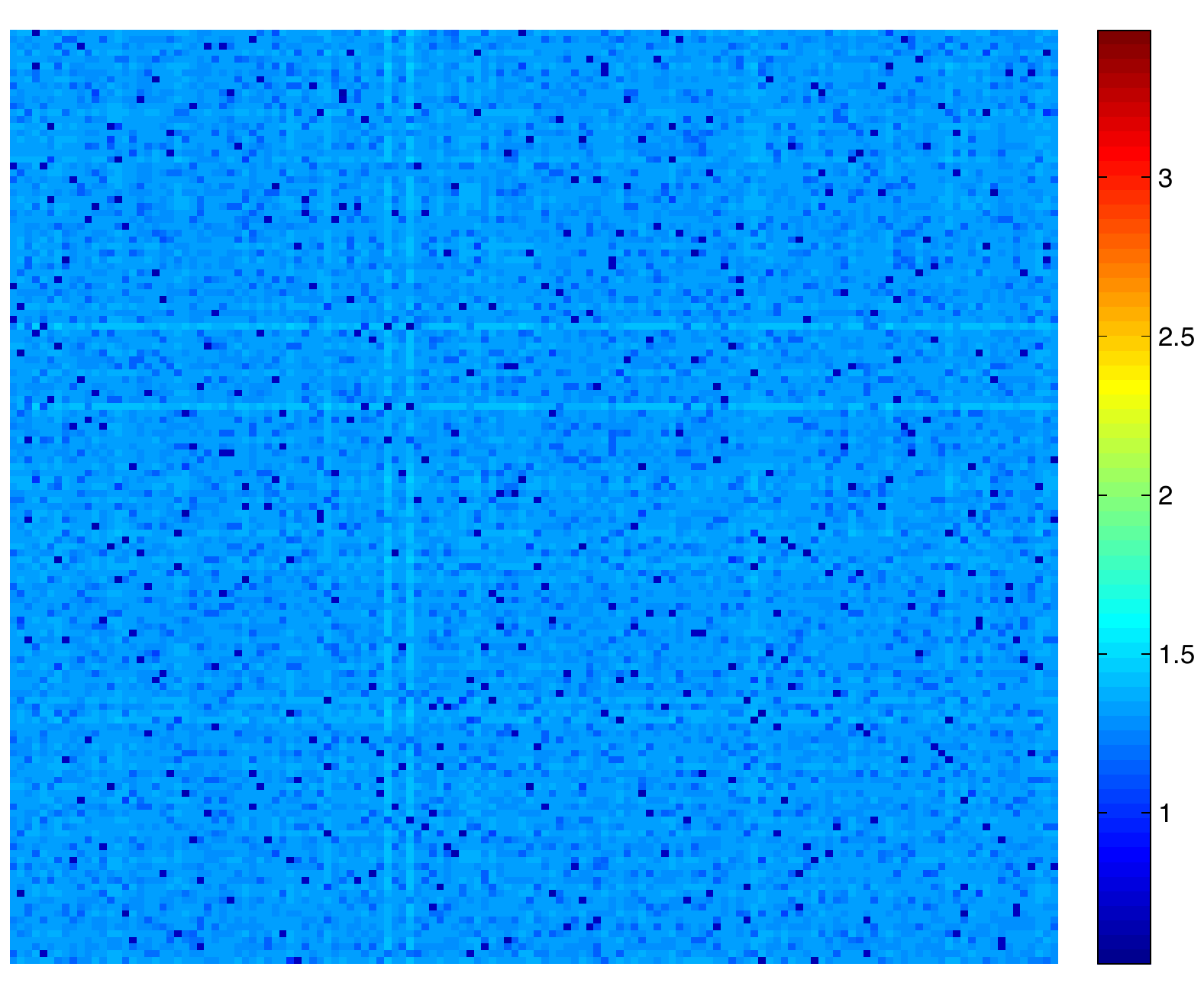}
}
\subfigure{%
\includegraphics[height=0.12\textwidth]{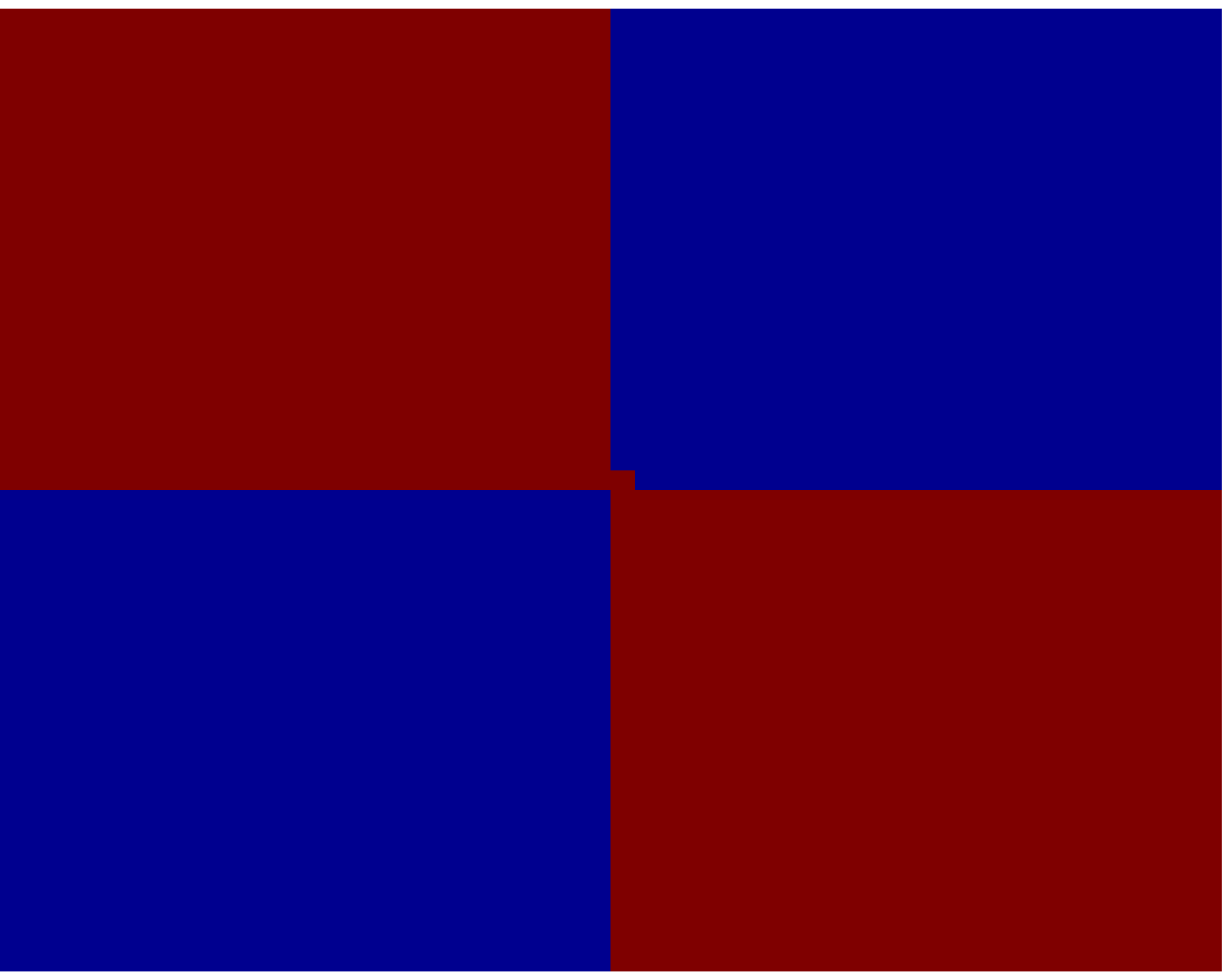}
}
\subfigure{%
\includegraphics[height=0.12\textwidth]{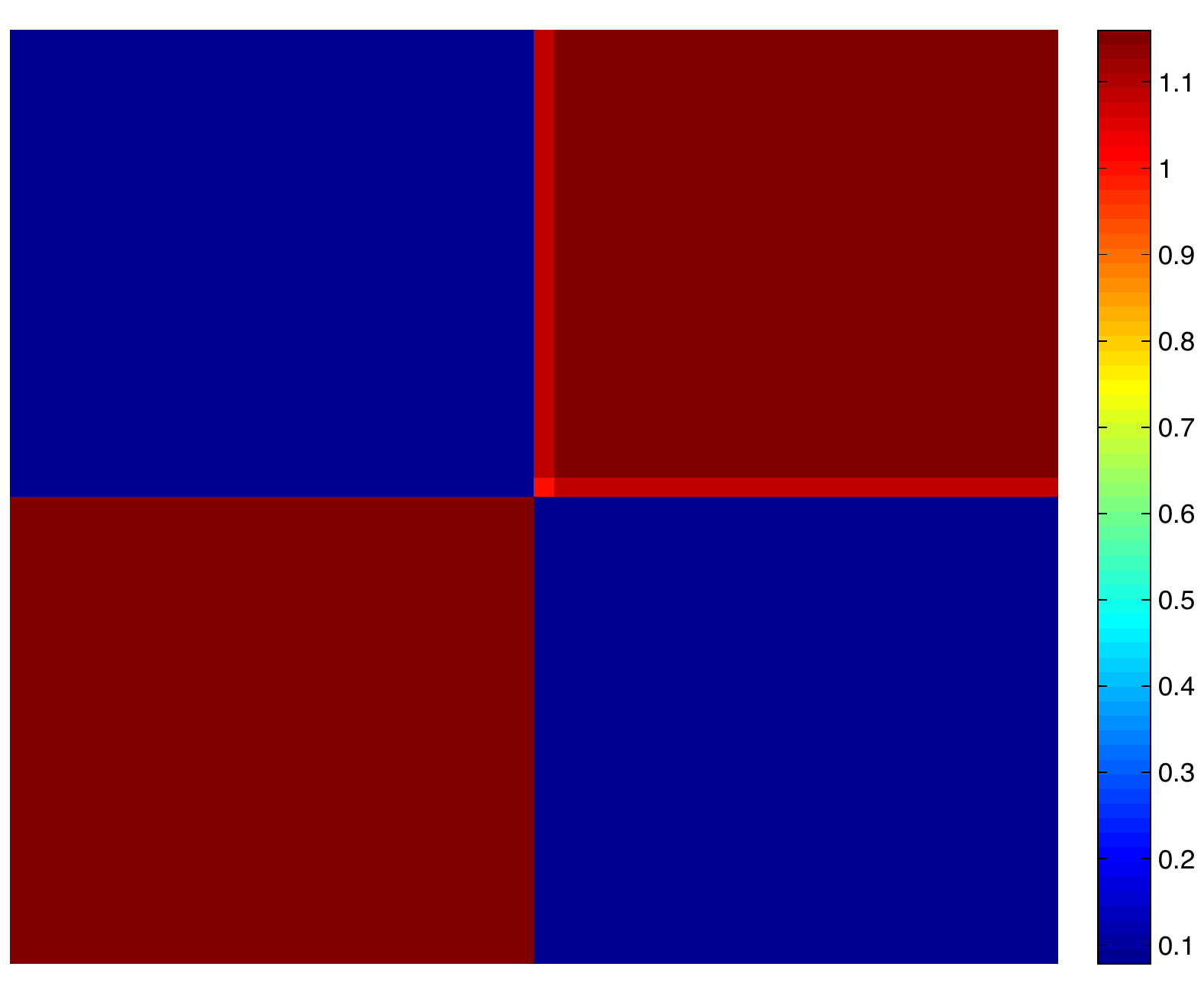}
}
\subfigure{%
\includegraphics[height=0.12\textwidth]{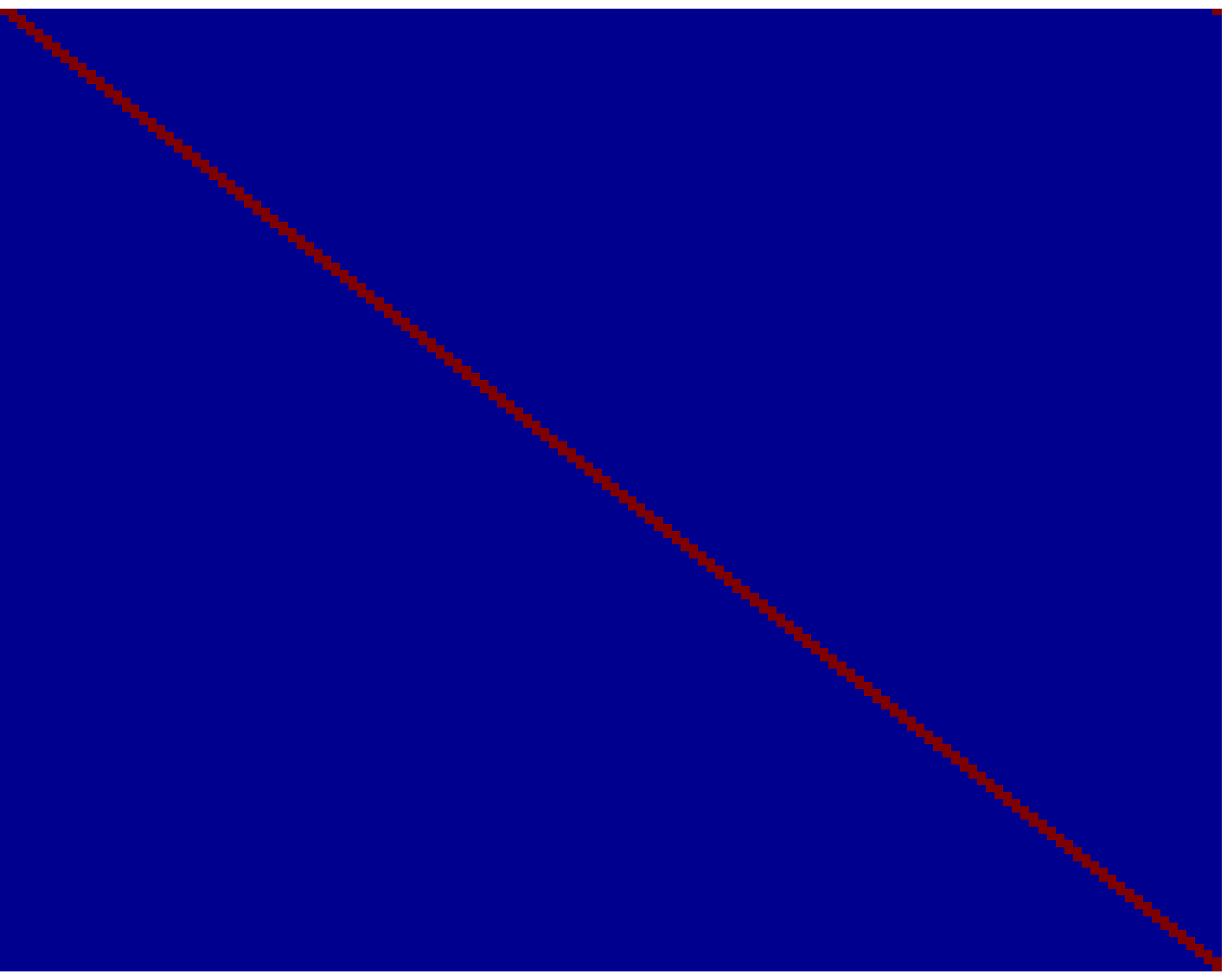}
}
\subfigure{%
\includegraphics[height=0.12\textwidth]{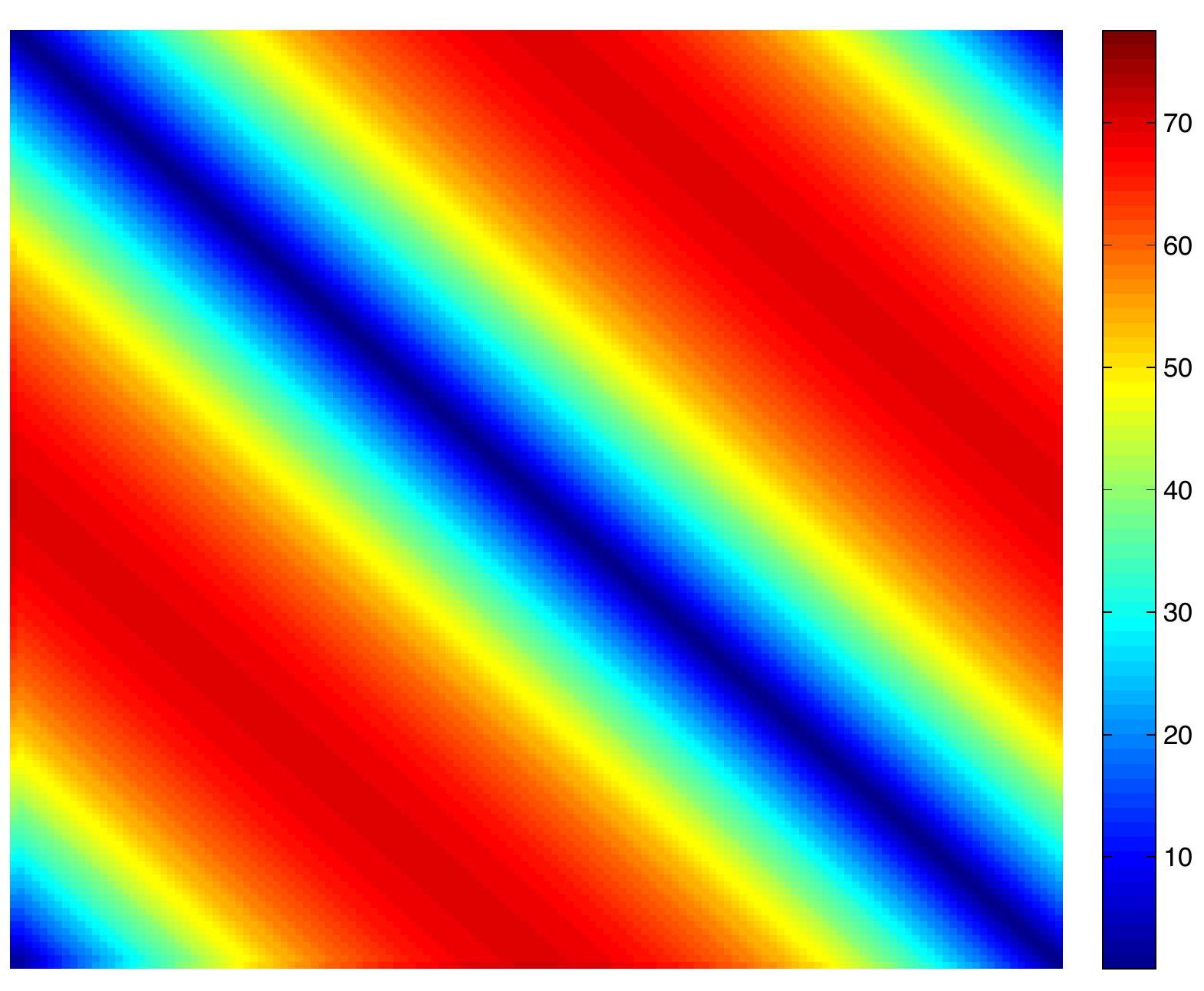}
}
\end{center}
\caption{%
\label{fig:heatmaps}
The figure shows three pairs of masks and predicted variances. A pair consists of two adjacent squares. The left half is the mask which is depicted by red/blue heatmap with red entries known and blue unknown. The right half is a multicolor heatmap with color scale, showing the predicted variance of the completion. Variances were calculated by our implementation of Algorithm~\ref{Alg:varrkone}.
}
\end{figure*}

\begin{figure*}[ht]
\begin{center}
\subfigure[mean squared errors]{%
\label{fig:3algs.noisemse}
\includegraphics[width=0.4\textwidth]{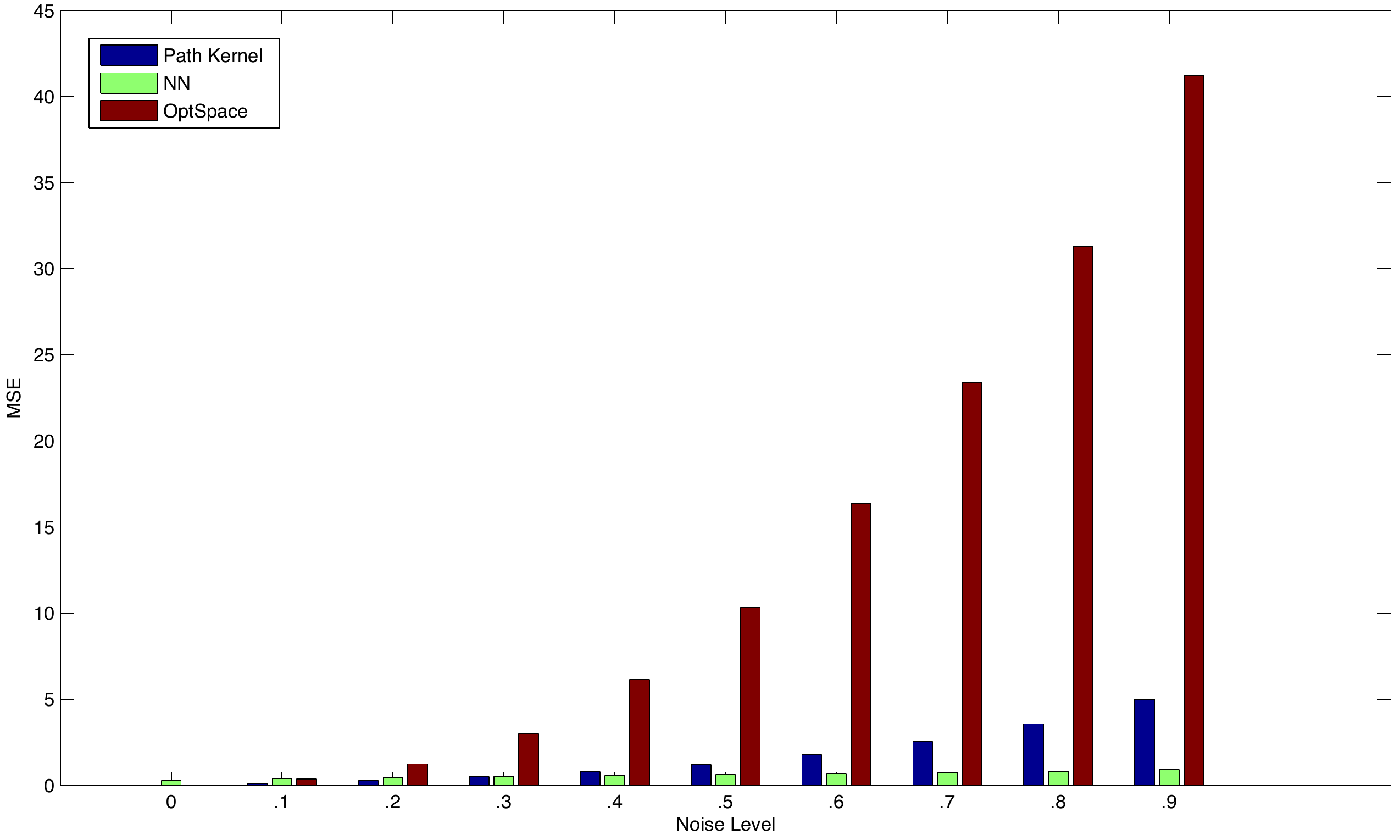}
}%
\subfigure[error vs.~predicted variance]{%
\label{fig:3algs.linearity}
\includegraphics[width=0.3\textwidth]{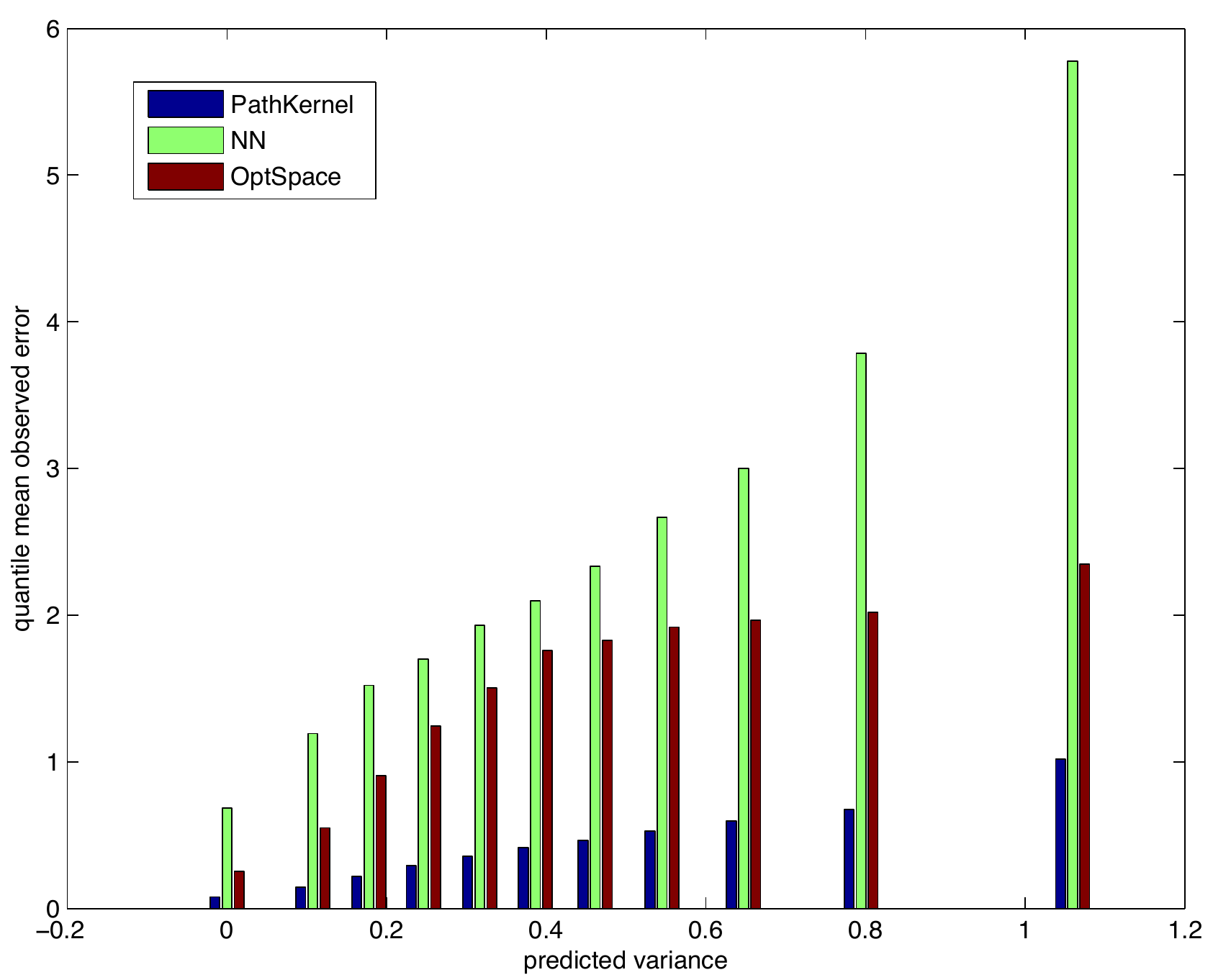}
}

\end{center}
\caption{%
For $10$ randomly chosen masks and $50\times 50$ true matrix, matrix completions were performed with Nuclear Norm (green), OptSpace (red), and Algorithm~\ref{Alg:loccomprkone} (blue) under multiplicative noise with variance increasing in increments of $0.1$. For each completed entry, minimum variances were predicted by Algorithm~\ref{Alg:varrkone}. \ref{fig:3algs.noisemse} shows the mean squared error of the three algorithms for each noise level, coded by the algorithms' respective colors. \ref{fig:3algs.linearity} shows a bin-plot of errors (y-axis) versus predicted variances (x-axis) for each of the three algorithms: for each completed entry, a pair (predicted error, true error) was calculated, predicted error being the predicted variance, and the actual prediction error being the squared logarithmic error (i.e., $\left (\log |a_{true}| - \log |a_{predicted}| \right)^2$ for an entry $a$). Then, the points were binned into $11$ bins with equal numbers of points. The figure shows the mean of the errors (second coordinate) of the value pairs with predicted variance (first coordinate) in each of the bins, the color
corresponds to the particular algorithm; each group of bars is centered on the minimum value of the associated bin.
}%
\end{figure*}

\section{Discussion and outlook}
\label{Sec:discussion}

In this paper we have demonstrated the usefulness and practicability of the algebraic combinatorial approach for matrix completion, by deriving reconstructability statments, and actual reconstruction algorithms for single missing entries. Our theory allows to treat the positions of the observations separately from the entries themselves. As a prominent model feature, we are able to separate the sampling scheme from algebraic and combinatorial conditions for reconstruction and explain existing reconstruction bounds by the combinatorial phase transition for the uniform random sampling scheme.

In our new setting, we are left with a number of major \textbf{open questions}:
\begin{itemize}
\item Characterize all \textbf{circuits and circuit polynomials} in rank 2 or higher (a more extensive account on this problem can be found in~\cite{KRT13}).
\item Give a sufficient and necessary combinatorial \textbf{criterion for unique completability}.
\item Give an \textbf{efficient algorithm certifying for unique completability} when given the positions of the observed entries (or, more generally, one which computes the number of solutions).
\item Prove the \textbf{phase transition bound for the completable core} (the phase transition bound for completability has been shown in~\cite{KirThe13Coherence}).
\item Explain the existing \textbf{guarantees for whole matrix reconstruction MSE in terms of single entry expected error}, for the various sampling models in literature. An explanation for rank one can be inferred from~\cite{KirThe13RankOneEst}.
\end{itemize}

The presented results also suggest a number of \textbf{future directions}:
\begin{itemize}
\item Problems such as \textbf{matrix completion under further constraints} such as for \textbf{symmetric matrices}, \textbf{distance matrices} or \textbf{kernel matrices}, are closely related to the ones we consider here, and can be treated by similar techniques.
(under a phase transition aspect, these models were studied in~\cite{KirThe13Coherence}; under a matroidal aspect, the theory in~\cite{KRT13} yields a starting point)
\item \textbf{Completion of tensors} is a natural generalization of matrix completion and accessible to the techniques presented here or in~\cite{KirThe13Coherence,KRT13}.
\item We have essentially shown matrix completion to be an \textbf{algebraic manifold learning problem}. This makes it accessible to the kernel/ideal learning techniques presented in~\cite{KKT14}.
\item The algebraic theory used to infer genericity and identifiability is largely independent of the matrix completion setting and can be applied to \textbf{inverse problems and compressed sensing problems that are algebraic}. In~\cite{KE14}, this was demonstrated for phase retrieval.
\end{itemize}

Sumamrizing, we argue that recognizing and exploiting algebra and combinatorics in machine learning problems is
beneficial from the practical and theoretical perspectives. When it is present, methods using underlying algebraic and combinatorial structures yield sounder statements and more practical algorithms than can be obtained when ignoring it, conversely algebra and combinatorics can profit from the various interesting structure surfacing in machine learning problems. Therefore all involved fields can only profit from a more widespread interdiscplinary collaboration with and between each other.

\section*{Acknowledgements}
We thank Andriy Bondarenko, Winfried Bruns, Eyke H\"ullermeyer, Mihyun Kang, Yusuke Kobayashi, Martin Kreuzer,
Cris Moore, Klaus-Robert M\"uller, Kazuo Murota, Kiyohito Nagano, Zvi Rosen,
Raman Sanyal, Bernd Sturmfels, Sumio Watanabe, Volkmar Welker, and G\"unter Ziegler for valuable discussions.
RT is partially supported by MEXT KAKENHI 22700138, the Global COE program ``The Research and Training Center for
New Development in Mathematics'', FK by Mathematisches Forschungsinstitut Oberwolfach (MFO), and LT
by the European Research Council under the European Union's Seventh Framework
Programme (FP7/2007-2013) / ERC grant agreement no 247029-SDModels. This research was
partially carried out at MFO, supported by FK's Oberwolfach Leibniz Fellowship.

\newpage
\bibliographystyle{plainnat}

\small
\begin{thebibliography}{44}
\providecommand{\natexlab}[1]{#1}
\providecommand{\url}[1]{\texttt{#1}}
\expandafter\ifx\csname urlstyle\endcsname\relax
\providecommand{\doi}[1]{doi: #1}\else
\providecommand{\doi}{doi: \begingroup \urlstyle{rm}\Url}\fi

\bibitem[Acar et~al.(2009)Acar, Dunlavy, and Kolda]{AcaDunKol09}
Evrim Acar, Daniel~M. Dunlavy, and Tamara~G. Kolda.
\newblock Link prediction on evolving data using matrix and tensor
factorizations.
\newblock In \emph{Data Mining Workshops, 2009. ICDMW'09. IEEE International
Conference on}, pages 262--269. IEEE, 2009.

\bibitem[Allman et~al.(2009)Allman, Matias, and Rhodes]{AMR09}
Elizabeth~S. Allman, Catherine Matias, and John~A. Rhodes.
\newblock Identifiability of parameters in latent structure models with many
observed variables.
\newblock \emph{Ann. Statist.}, 37\penalty0 (6A):\penalty0 3099--3132, 2009.

\bibitem[Argyriou et~al.(2008)Argyriou, Micchelli, Pontil, and
Ying]{ArgMicPonYin08}
Andreas Argyriou, Craig~A. Micchelli, Massimiliano Pontil, and Yi~Ying.
\newblock A spectral regularization framework for multi-task structure
learning.
\newblock In J.C. Platt, D.~Koller, Y.~Singer, and S.~Roweis, editors,
\emph{Advances in NIPS 20}, pages 25--32. MIT Press, Cambridge, MA, 2008.

\bibitem[Bamber(1985)]{BS85}
Donald Bamber.
\newblock How many parameters can a model have and still be testable?
\newblock \emph{Journal of Mathematical Psychology}, 29\penalty0 (4):\penalty0
443 -- 473, 1985.

\bibitem[Blythe et~al.(2014)Blythe, Kir\'aly, and Theran]{BlyKirThe14Running}
Duncan Blythe, Franz~J. Kir\'aly, and Louis Theran.
\newblock Algebraic combinatorial methods for low-rank matrix completion with
application to athletic performance prediction.
\newblock \emph{arXiv Preprint}, 2014.
\newblock arXiv 1406.2864.

\bibitem[Bollob{\'a}s(2001)]{BollobasBook}
B{\'e}la Bollob{\'a}s.
\newblock \emph{Random graphs}, volume~73 of \emph{Cambridge Studies in
Advanced Mathematics}.
\newblock Cambridge University Press, Cambridge, second edition, 2001.
\newblock ISBN 0-521-80920-7; 0-521-79722-5.
\newblock \doi{10.1017/CBO9780511814068}.
\newblock URL \url{http://dx.doi.org/10.1017/CBO9780511814068}.

\bibitem[Bruns and Vetter(1988)]{Bruns}
Winfried Bruns and Udo Vetter.
\newblock \emph{Determinantal Rings}.
\newblock Springer-Verlag New York, Inc., 1988.

\bibitem[Cand{\`e}s and Recht(2009)]{CanRec09}
Emmanuel~J. Cand{\`e}s and Benjamin Recht.
\newblock Exact matrix completion via convex optimization.
\newblock \emph{Found. Comput. Math.}, 9\penalty0 (6):\penalty0 717--772, 2009.
\newblock ISSN 1615-3375.

\bibitem[Cand{\`e}s and Tao(2010)]{CandesTao}
Emmanuel~J. Cand{\`e}s and Terence Tao.
\newblock The power of convex relaxation: near-optimal matrix completion.
\newblock \emph{IEEE Trans. Inform. Theory}, 56\penalty0 (5):\penalty0
2053--2080, 2010.
\newblock ISSN 0018-9448.
\newblock \doi{10.1109/TIT.2010.2044061}.
\newblock URL \url{http://dx.doi.org/10.1109/TIT.2010.2044061}.

\bibitem[Chaterjee(2012)]{C12}
Sourav Chaterjee.
\newblock Matrix estimation by {U}niversal {S}ingular {V}alue {T}hresholding.
\newblock Preprint, arXiv:1212.1247, 2012.
\newblock URL \url{http://arxiv.org/abs/1212.1247}.

\bibitem[Dress and Lov{\'a}sz(1987)]{DL87}
A.~Dress and L.~Lov{\'a}sz.
\newblock On some combinatorial properties of algebraic matroids.
\newblock \emph{Combinatorica}, 7\penalty0 (1):\penalty0 39--48, 1987.

\bibitem[Foygel and Srebro(2011)]{FoySre11}
Rina Foygel and Nathan Srebro.
\newblock Concentration-based guarantees for low-rank matrix reconstruction.
\newblock \emph{Arxiv preprint arXiv:1102.3923}, 2011.

\bibitem[Goldberg et~al.(2010)Goldberg, Zhu, Recht, Xu, and
Nowak]{GolZhuRecXuNow10}
Andrew Goldberg, Xiaojin Zhu, Benjamin Recht, Jun-Ming Xu, and Robert Nowak.
\newblock Transduction with matrix completion: Three birds with one stone.
\newblock In J.~Lafferty, C.~K.~I. Williams, J.~Shawe-Taylor, R.S. Zemel, and
A.~Culotta, editors, \emph{Advances in Neural Information Processing Systems
23}, pages 757--765. 2010.

\bibitem[Grothendieck and Dieudonn{\'e}(1965)]{EGAIV.2}
Alexander Grothendieck and Jean Dieudonn{\'e}.
\newblock {\'E}l{\'e}ments de g{\'e}om{\'e}trie alg{\'e}brique iv, deuxi{\`e}me
partie.
\newblock \emph{Publ. Math. IHES}, 24, 1965.

\bibitem[Grothendieck and Dieudonn{\'e}(1966)]{EGAIV.3}
Alexander Grothendieck and Jean Dieudonn{\'e}.
\newblock {\'E}l{\'e}ments de g{\'e}om{\'e}trie alg{\'e}brique iv,
troisi{\`e}me partie.
\newblock \emph{Publ. Math. IHES}, 28, 1966.

\bibitem[GroupLens()]{movielensdata}
GroupLens.
\newblock Movielens 100k data set.
\newblock Available online at \url{http://grouplens.org/datasets/movielens/};
as downloaded on November 27th 2012.

\bibitem[Hsu et~al.(2012)Hsu, Kakade, and Liang]{HKL12}
Daniel Hsu, Sham Kakade, and Percy Liang.
\newblock Identifiability and unmixing of latent parse trees.
\newblock In P.~Bartlett, F.C.N. Pereira, C.J.C. Burges, L.~Bottou, and K.Q.
Weinberger, editors, \emph{Advances in Neural Information Processing Systems
25}, pages 1520--1528. 2012.

\bibitem[Jackson et~al.(2007)Jackson, Servatius, and Servatius]{JSS07}
Bill Jackson, Brigitte Servatius, and Herman Servatius.
\newblock The 2-dimensional rigidity of certain families of graphs.
\newblock \emph{J. Graph Theory}, 54\penalty0 (2):\penalty0 154--166, 2007.
\newblock ISSN 0364-9024.

\bibitem[Janson and Luczak(2007)]{JL07}
Svante Janson and Malwina~J. Luczak.
\newblock A simple solution to the {$k$}-core problem.
\newblock \emph{Random Structures Algorithms}, 30\penalty0 (1-2):\penalty0
50--62, 2007.
\newblock ISSN 1042-9832.
\newblock \doi{10.1002/rsa.20147}.
\newblock URL \url{http://dx.doi.org/10.1002/rsa.20147}.

\bibitem[Kasiviswanathan et~al.(2011)Kasiviswanathan, Moore, and Theran]{KMT11}
Shiva~Prasad Kasiviswanathan, Cristopher Moore, and Louis Theran.
\newblock The rigidity transition in random graphs.
\newblock In \emph{Proceedings of the {T}wenty-{S}econd {A}nnual {ACM}-{SIAM}
{S}ymposium on {D}iscrete {A}lgorithms}, pages 1237--1252, Philadelphia, PA,
2011. SIAM.

\bibitem[Keshavan et~al.(2010{\natexlab{a}})Keshavan, Montanari, and Oh]{KMO10}
Raghunandan~H. Keshavan, Andrea Montanari, and Sewoong Oh.
\newblock Matrix completion from a few entries.
\newblock \emph{IEEE Trans. Inform. Theory}, 56\penalty0 (6):\penalty0
2980--2998, 2010{\natexlab{a}}.
\newblock ISSN 0018-9448.
\newblock \doi{10.1109/TIT.2010.2046205}.
\newblock URL \url{http://dx.doi.org/10.1109/TIT.2010.2046205}.

\bibitem[Keshavan et~al.(2010{\natexlab{b}})Keshavan, Montanari, and
Oh]{KesMonOh10}
R.H. Keshavan, A.~Montanari, and S.~Oh.
\newblock Matrix completion from a few entries.
\newblock \emph{Information Theory, IEEE Transactions on}, 56\penalty0
(6):\penalty0 2980--2998, 2010{\natexlab{b}}.

\bibitem[Kir{\'a}ly and Ehler(2014)]{KE14}
Franz~J. Kir{\'a}ly and Martin Ehler.
\newblock The algebraic approach to phase retrieval and explicit inversion at
the identifiability threshold.
\newblock \emph{arXiv e-prints}, February 2014.
\newblock arXiv:1402.4053.

\bibitem[Kir\'aly and Theran(2013{\natexlab{a}})]{KirThe13Coherence}
Franz~J. Kir\'aly and Louis Theran.
\newblock Coherence and sufficient sampling densities for reconstruction in
compressed sensing.
\newblock \emph{arXiv Preprint}, 2013{\natexlab{a}}.
\newblock arXiv 1302.2767.

\bibitem[Kir\'aly and Theran(2013{\natexlab{b}})]{KirThe13RankOneEst}
Franz~J. Kir\'aly and Louis Theran.
\newblock Obtaining error-minimizing estimates and universal entry-wise error
bounds for low-rank matrix completion.
\newblock \emph{arXiv Preprint}, 2013{\natexlab{b}}.
\newblock arXiv 1302.5337.

\bibitem[Kir\'aly et~al.(2013)Kir\'aly, Theran, Tomioka, and
UNo]{KirTheTomUno12version3}
Franz~J. Kir\'aly, Louis Theran, Ryota Tomioka, and Takeaki UNo.
\newblock The algebraic combinatorial approach for low-rank matrix completion.
\newblock \emph{arXiv Preprint}, 2013.
\newblock arXiv 1211.4116v3.

\bibitem[Kir{\'a}ly et~al.(2014)Kir{\'a}ly, Kreuzer, and Theran]{KKT14}
Franz~J. Kir{\'a}ly, Martin Kreuzer, and Louis Theran.
\newblock Dual-to-kernel learning with ideals.
\newblock \emph{arXiv e-prints}, February 2014.
\newblock arXiv:1402.0099.

\bibitem[Király et~al.(2013)Király, Rosen, and Theran]{KRT13}
Franz~J. Király, Zvi Rosen, and Louis Theran.
\newblock Algebraic matroids with graph symmetry.
\newblock \emph{arXiv Preprint}, 2013.
\newblock arXiv 1312.3777.

\bibitem[Kurdyka et~al.(2000)Kurdyka, Orro, and Simon]{KOS00}
K.~Kurdyka, P.~Orro, and S.~Simon.
\newblock Semialgebraic {S}ard theorem for generalized critical values.
\newblock \emph{J. Differential Geom.}, 56\penalty0 (1):\penalty0 67--92, 2000.

\bibitem[Meka et~al.(2009)Meka, Jain, and Dhillon]{Mek09}
Raghu Meka, Prateek Jain, and Inderjit~S. Dhillon.
\newblock Guaranteed rank minimization via singular value projection, 2009.
\newblock Eprint arXiv:0909.5457.

\bibitem[Menon and Elkan(2011)]{MenElk11}
Aditya~K. Menon and Charles Elkan.
\newblock Link prediction via matrix factorization.
\newblock \emph{Machine Learning and Knowledge Discovery in Databases}, pages
437--452, 2011.

\bibitem[Milnor(1968)]{M68}
John Milnor.
\newblock \emph{Singular points of complex hypersurfaces}.
\newblock Annals of Mathematics Studies, No. 61. Princeton University Press,
Princeton, N.J., 1968.

\bibitem[Mumford(1999)]{Mumford}
David Mumford.
\newblock \emph{The Red Book of Varieties and Schemes}.
\newblock Lecture Notes in Mathematics. Springer-Verlag Berlin Heidelberg,
1999.
\newblock ISBN 054063293X.

\bibitem[Negahban and Wainwright(2011)]{NegWai11}
Sahand Negahban and Martin~J. Wainwright.
\newblock Estimation of (near) low-rank matrices with noise and
high-dimensional scaling.
\newblock \emph{Ann. Statist.}, 39\penalty0 (2), 2011.

\bibitem[Negahban and Wainwright(2012)]{NegWai12}
Sahand Negahban and Martin~J. Wainwright.
\newblock Restricted strong convexity and weighted matrix completion: Optimal
bounds with noise.
\newblock \emph{J. Mach. Learn. Res.}, 13:\penalty0 1665--1697, 2012.

\bibitem[Oxley(2011)]{Oxley}
James Oxley.
\newblock \emph{Matroid theory}, volume~21 of \emph{Oxford Graduate Texts in
Mathematics}.
\newblock Oxford University Press, Oxford, second edition, 2011.
\newblock ISBN 978-0-19-960339-8.

\bibitem[Pittel et~al.(1996)Pittel, Spencer, and Wormald]{PSW96}
Boris Pittel, Joel Spencer, and Nicholas Wormald.
\newblock Sudden emergence of a giant {$k$}-core in a random graph.
\newblock \emph{J. Combin. Theory Ser. B}, 67\penalty0 (1):\penalty0 111--151,
1996.
\newblock ISSN 0095-8956.
\newblock \doi{10.1006/jctb.1996.0036}.
\newblock URL \url{http://dx.doi.org/10.1006/jctb.1996.0036}.

\bibitem[Salakhutdinov and Srebro(2010)]{SalSre10}
Ruslan Salakhutdinov and Nathan Srebro.
\newblock Collaborative filtering in a non-uniform world: Learning with the
weighted trace norm.
\newblock In J.~Lafferty, C.~K.~I. Williams, J.~Shawe-Taylor, R.S. Zemel, and
A.~Culotta, editors, \emph{Advances in Neural Information Processing Systems
23}, pages 2056--2064. 2010.

\bibitem[Schwartz(1980)]{S80}
J.~T. Schwartz.
\newblock Fast probabilistic algorithms for verification of polynomial
identities.
\newblock \emph{J. ACM}, 27\penalty0 (4), October 1980.

\bibitem[Singer and Cucuringu(2010)]{Sin10}
Amit Singer and Mihai Cucuringu.
\newblock Uniqueness of low-rank matrix completion by rigidity theory.
\newblock \emph{Siam J Matrix Anal Appl}, 31\penalty0 (4):\penalty0 1621--1641,
2010.

\bibitem[Srebro and Shraibman(2005)]{SreShr05}
Nathan Srebro and Adi Shraibman.
\newblock Rank, trace-norm and max-norm.
\newblock In \emph{Proc. of the 18th Annual Conference on Learning Theory
(COLT)}, pages 545--560. Springer, 2005.

\bibitem[{Srebro} et~al.(2005){Srebro}, {Rennie}, and {Jaakkola}]{SreRenJaa05}
Nathan {Srebro}, Jason D.~M. {Rennie}, and Tommi~S. {Jaakkola}.
\newblock Maximum-margin matrix factorization.
\newblock In Lawrence~K. Saul, Yair Weiss, and {L\'{e}on} Bottou, editors,
\emph{Advances in Neural Information Processing Systems 17}, pages
1329--1336. MIT Press, Cambridge, MA, 2005.

\bibitem[Steger and Wormald(1999)]{SteWor99}
Angelika Steger and Nicholas Wormald.
\newblock Generating random regular graphs quickly.
\newblock \emph{Combinatorics Probability and Computing}, 8\penalty0
(4):\penalty0 377--396, 1999.

\bibitem[Tomioka et~al.(2010)Tomioka, Hayashi, and Kashima]{THK10}
Ryota Tomioka, Kohei Hayashi, and Hisashi Kashima.
\newblock On the extension of trace norm to tensors.
\newblock In \emph{NIPS Workshop on Tensors, Kernels, and Machine Learning},
2010.

\end{thebibliography}

\newpage

\appendix

\section{Algebraic Geometry Fundamentals}\label{Sec:app-algebra}

\subsection{Algebraic Genericity}
We will briefly review the concept of genericity for our purposes. Intuitively, algebraic genericity describes that some statements holds for almost all objects, with the exceptions having an algebraic structure. The following results will be stated for algebraic varieties over the real or complex numbers, that is, over the field $\K$, where $\K=\R$ or $\K=\C$.

\begin{Def}\label{Def:genconds}
Let $\calY\subseteq \K^n$ be an algebraic variety. Let $P$ be some property of points $y\in\calY$. Write $P(\calY)=\{y\in\calY\;:\;y\;\mbox{has property}\;P\}$, and $\neg P(\calY) = \calY\setminus P(\calY)$.
\begin{description}
\item[(i)] We call $P$ an \emph{open condition} if $P(\calY)$ is a Zariski open subset of $\calY$.
\item[(ii)] We call $P$ a \emph{Zariski-generic condition} if there is an open dense subset $U\subseteq \calY$ such that $U\subseteq P(\calY)$.
\item[(iii)] We call $P$ a \emph{Hausdorff-generic condition} if $\neg P(\calY)$ is a $\calY$-Hausdorff zero set.
\end{description}
\end{Def}

The different types of conditions above can be put in relation to each other:
\begin{Prop}\label{Prop:genconds}
Keep the notation of Definition~\ref{Def:genconds}.
\begin{description}
\item[(i)] If $P$ is a Zariski-generic condition, then $P$ is a Hausdorff-generic condition as well.
\item[(ii)] Assume $\calY$ is irreducible, and $P(\calY)$ is non-empty. If $P$ is an open condition, then it is a Zariski-generic condition.
\item[(iii)] Assume $\calY$ is irreducible, and $P(\calY)$ is constructible in the Zariski topology, i.e., can be written as finite union and intersection of open and closed sets. If $P$ is a Hausdorff-generic condition, then it is a Zariski-generic condition as well.
\end{description}
\end{Prop}
\begin{proof}
(i) follows from the fact that Zariski closed sets of smaller Krull dimension are Hausdorff zero sets.\\
(ii) follows from the fact that non-empty Zariski open sets are dense in an irreducible algebraic set.\\
(iii) as $P(\calY)$ is Zariski-constructible, it will have positive Hausdorff measure if and only if it contains a non-empty (relatively Zariski) open set. The statement then follows from (ii).
\end{proof}

Furthermore, Hausdorff-genericity is essentially states that the condition holds, universally with probability one:
\begin{Prop}
Keep the notation of Definition~\ref{Def:genconds}. The following are equivalent:
\begin{description}
\item[(i)] $P$ is a Hausdorff-generic condition.
\item[(ii)] For all Hausdorff-continuous random variables $X$ taking values in $\calY$, the statement $P(X)$ holds with probability one.
\end{description}
\end{Prop}
\begin{proof}
(i)$\Leftrightarrow$ (ii) follows from taking Radon-Nikodym derivatives. %
\end{proof}

All relevant properties and conditions which are referenced from the main corpus describe (a) irreducible varieties - in this case, the determinantal variety, and (b) are Zariski-constructible. Therefore, by Proposition~\ref{Prop:genconds}, all three definitions agree for the purpose of this paper. The terminology used in the paper can be given as follows in the above definitions:

\begin{Def}
Let $\calY\subseteq \K^n$ be an algebraic variety. Let $P$ be some property of points $y\in\calY$. We say ``a generic $y\in\calY$ has property $P$'' if $P$ is a Hausdorff-generic condition for points in $\calY$.
\end{Def}

\subsection{Open Conditions and Generic Properties of Morphisms}\label{app:algebraic-geometry}
In this section, we will summarize some algebraic geometry results used in the main corpus. The following results will always be stated for algebraic varieties over $\C$.

\begin{Prop}\label{Prop:irreducible}
Let $f:\calX\rightarrow \calY$ be a morphism of algebraic varieties (over any field). Then, if $\calX$ is irreducible, so is $f(\calX)$. In particular, if $f$ is surjective, and $\calX$ is irreducible, then $\calY$ also is.
\end{Prop}
\begin{proof}
This is classical and follows directly from the fact that morphisms of algebraic varieties are continuous in the Zariski topology.
\end{proof}

\begin{Thm}
Let $f:\calX\rightarrow \calY$ be a morphism of algebraic varieties. The function
$$\calY\rightarrow\N,\quad y\mapsto \dim f^{-1}(y)$$
is upper semicontinuous in the Zariski topology.
\end{Thm}
\begin{proof}
This follows from~\cite[Th{\'e}or{\`e}me~13.1.3]{EGAIV.3}.
\end{proof}

\begin{Prop}\label{Prop:flatgen}
Let $f:\calX\rightarrow \calY$ be a morphism of algebraic varieties, with $\calY$ be irreducible. Then, there is an open dense subset $V\subseteq \calY$ such that $f: U\rightarrow V$, where $U=f^{-1}(V)$, is a flat morphism.
\end{Prop}
\begin{proof}
This follows from~\cite[Th{\'e}or{\`e}me~6.9.1]{EGAIV.2}.
\end{proof}

\begin{Thm}\label{Thm:openconds}
Let $f:\calX\rightarrow \calY$ be a morphism of algebraic varieties. Let $d,\nu\in\N$. Then, the following are open conditions for $y\in \calY$:
\begin{description}
\item[(i)] $\dim f^{-1}(y)\le d$.
\item[(ii)] $f$ is unramified over $y$.
\item[(iii)] $f$ is unramified over $y$, and the number of irreducible components of $f^{-1}(y)$ equals $\nu$.
\end{description}
In particular, if $f$ is surjective, then the following is an open property as well:
\begin{description}
\item[(iv)] $f$ is unramified over $y$, and $\card {f^{-1}(y)}=\nu$, for some $\nu\in\N$.
\end{description}
\end{Thm}
\begin{proof}
(i) follows from~\cite[Corollaire~6.1.2]{EGAIV.2}.\\
(ii) follows from~\cite[Th{\'e}or{\`e}me~12.2.4(v)]{EGAIV.3}.\\
(iii) follows from~\cite[Th{\'e}or{\`e}me~12.2.4(vi)]{EGAIV.3}.\\
(iv) follows from (i), applied in the case $\dim f^{-1}(y)\le 0$ which is equivalent to $\dim f^{-1}(y) = 0$ due to surjectivity of $f$, and (iii).
\end{proof}

\begin{Cor}\label{Cor:genericprops}
Let $f:\calX\rightarrow \calY$ be a generically unramified and surjective morphism of algebraic varieties, with $\calY$ be irreducible. Then, there are unique $d,\nu\in\N$ such that the following sets are Zariski closed, proper subsets of $\calY$ (and therefore Hausdorff zero sets):
\begin{description}
\item[(i)] $\{y\;:\;\dim f^{-1}(y)\neq d\}$
\item[(ii)] $\{y\;:\;f\;\mbox{is ramified at}\;y\}$
\item[(iii)] $\{y\;:\;f\;\mbox{is ramified at}\;y\}\cup\{y\;:\;\card {f^{-1}(y)}\neq \nu\}$
\end{description}
\end{Cor}
\begin{proof}
This is implied by Theorem~\ref{Thm:openconds}~(i), (ii) and (iii), using that a non-zero open subset of the irreducible variety $\calY$ must be open dense, therefore its complement in $\calY$ a closed and a proper subset of $\calY$.
\end{proof}

\begin{Prop}\label{Prop:opencert}
Let $f:\calX\rightarrow\calY$ be a morphism of algebraic varieties, with $\calY$ irreducible. Then, the following are equivalent:
\begin{description}
\item[(i)] $f$ is unramified over $y$ and $\card{f^{-1}(y)} = \nu$.
\item[(ii)] There is a Borel open neighborhood $U\subseteq \calY$ of $y\in U$, such that $f$ is unramified over $U$ and $\card{f^{-1}(z)} = \nu$ for all $z\in U$.
\item[(iii)] There is a Zariski open neighborhood $U\subseteq \calY$ of $y\in U$, dense in $\calY$, such that $f$ is unramified over $U$ and $\card{f^{-1}(z)} = \nu$ for all $z\in U$.
\end{description}
\end{Prop}
\begin{proof}
The equivalence is implied by Corollary~\ref{Cor:genericprops} and the fact that $\calY$ is irreducible. Note that either condition implies that $f$ is generically unramified due to Theorem~\ref{Thm:openconds}~(ii) and irreducibility of $\calY$.
\end{proof}

\subsection{Real versus Complex Genericity}

We derive some elementary results how generic properties over the complex and real numbers relate. While some could be taken for known results, they appear not to be folklore - except maybe Lemma~\ref{Lem:cmprl}. In any case, they seem not to be written up properly in literature known to the authors. A first version of the statements below has also appeared as part of~\cite{KE14}.

\begin{Def}
Let $\calX\subseteq \C^n$ be a variety. We define the \emph{real part} of $\calX$ to be $\calX_\R:=\calX\cap \R^n$.
\end{Def}

\begin{Lem}\label{Lem:cmprl}
Let $\calX\subseteq \C^n$ be a variety. Then, $\dim \calX_\R\le \dim \calX$, where $\dim \calX_\R$ denotes the Krull dimension of $\calX_\R$, regarded as a (real) subvariety of $\R^n$, and $\dim \calX$ the Krull dimension of $\calX$, regarded as subvariety of $\C^n$.
\end{Lem}
\begin{proof}
Let $k=n - \dim\calX$. By~\cite[section~1.1]{Mumford}, $\calX$ is contained in some complete intersection variety $\calX'=\Van (f_1,\dots, f_k)$. That is $(f_1,\dots, f_k)$ is a complete intersection, with $f_i\in\C[X_1,\dots, X_n]$ and $\dim\calX'=\dim\calX$, such that $f_i$ is a non-zero divisor modulo $f_1,\dots, f_{i-1}$. Define $g_i:=f_i\cdot f_i^*$, one checks that $g_i\in\R[X_1,\dots, X_n]$, and define $\calY:=\Van (g_1,\dots, g_k)$ and $\calY_\R:=\calY\cap \R^n$. The fact that $f_i$ is a non-zero divisor modulo $f_1,\dots, f_{i-1}$ implies that $g_i$ is a non-zero divisor modulo $g_1,\dots, g_{i-1}$; since $g_i\cdot h \cong 0$ modulo $g_1,\dots, g_{i-1}$ implies $f_i\cdot (h\cdot f_i^*) \cong 0$ modulo $f_1,\dots, f_{i-1}$. Therefore, $\dim\calY_\R \le \dim\calX$; by construction, $\calX'\subseteq \calY$, and $\calX\subseteq \calX'$, therefore $\calX_\R\subseteq \calY_\R$, and thus $\dim \calX_\R\le \dim \calY_\R$. Combining it with the above inequality yields the claim.
\end{proof}

\begin{Def}
Let $\calX\subseteq \C^n$ be a variety. If $\dim \calX = \dim \calX_\R$, we call $\calX$ \emph{observable over the reals}. If $\calX$ equals the (complex) Zariski-closure of $\calX_\R$, we call $\calX$ \emph{defined over the reals}.
\end{Def}

\begin{Prop}\label{Prop:realdefobs}
Let $\calX\subseteq \C^n$ be a variety.
\begin{description}
\item[(i)] If $\calX$ is defined over the reals, then $\calX$ is also observable over the reals.
\item[(ii)] The converse of (i) is false.
\item[(iii)] If $\calX$ irreducible and observable over the reals, then $\calX$ is defined over the reals.
\end{description}
\end{Prop}
\begin{proof}
(i) Let $k=n - \dim\calX_\R$. By~\cite[section~1.1]{Mumford}, $\calX_\R$ is contained in some complete intersection variety $\calX'=\Van (f_1,\dots, f_k)$, with $f_i\in\R[X_1,\dots, X_n]$ a complete intersection. By an argument, analogous to the proof of Lemma~\ref{Lem:cmprl}, one sees that the $f_i$ are a complete intersection in $\C [X_1,\dots, X_n]$ as well. Since the Zariski-closure of $\calX_\R$ and $\calX$ are equal, it holds that $f_i\in \Id (\calX)$. Therefore, $\calX\subseteq \Van (f_1,\dots, f_k)$, which imples $\dim\calX\le n-k$, and by definition of $k$, as well $\dim\calX\le\dim \calX_\R$. With Lemma~\ref{Lem:cmprl}, we obtain $\dim\calX_\R = \dim\calX$, which was the statement to prove.\\
(ii) It suffices to give a counterexample: $\calX = \{1,i\}\subseteq \C$. Alternatively (in a context where $\varnothing$ is not a variety) $\calX = \{(1,x)\;:\; x\in \C\}\cup \{(i,x)\;:\; x\in \C\}\subseteq \C^2$.\\
(iii) By definition of dimension, Zariski-closure preserves dimension. Therefore, the closure $\overline{\calX_\R}$ is a sub-variety of $\calX$, with $\dim \overline{\calX_\R} = \dim\calX$. Since $\calX$ is irreducible, equality $\overline{\calX_\R} = \calX$ must hold.
\end{proof}

\begin{Thm}\label{Thm:genreal}
Let $\calX\subseteq \C^n$ be an irreducible variety which is observable over the reals, let $\calX_\R$ be its real part. Let $P$ be an algebraic property. Assume that a generic $x\in\calX$ is $P$. Then, a generic $x\in \calX_\R$ has property $P$ as well.
\end{Thm}
\begin{proof}
Since $P$ is an algebraic property, the $P$ points of $\calX$ are contained in a proper sub-variety $\calZ\subseteq \calX$, with $\dim\calZ\lneq \dim\calX$. Since $\calX$ is observable over the reals, it holds $\dim\calX=\dim\calX_\R$. By Lemma~\ref{Lem:cmprl}, $\dim\calZ_\R\le \dim\calZ$. Putting all (in-)equalities together, one obtains $\dim\calZ_\R\lneq \dim\calX_\R$. Therefore, the $\calZ_\R$ is a proper sub-variety of $\calX_\R$; and the $P$ points of $\calX_\R$ are contained in it - this proves the statement.
\end{proof}

\subsection{Algebraic Properties of the Masking}
We conclude with checking the conditions previously discussed in the specific case of the masking:

\begin{Prop}\label{Prop:masking}
For $E\subseteq \calE$, consider the determinantal variety $\calM(m\times n,r)$ (over $\C$), and the masking
$$\Omega: \calM(m\times n,r)\rightarrow \C^{\card{E}},\quad \mA \mapsto \{A_e,e\in E\}.$$
\begin{description}
\item[(i)] The determinantal variety $\calM(m\times n,r)$ is irreducible.
\item[(ii)] The determinantal variety $\calM(m\times n,r)$ is observable over the reals.
\item[(iii)] The determinantal variety $\calM(m\times n,r)$ is defined over the reals.
\item[(iv)] The variety $\Omega(\calM(m\times n,r))$ is irreducible.
\item[(v)] The map $\Omega$ is generically unramified.
\end{description}
\end{Prop}
\begin{proof}
(i) follows from Proposition~\ref{Prop:irreducible}, applied to the surjective map
$$\Upsilon:\C^{m\times r}\times \C^{n\times r}\rightarrow \calM(m\times n,r),\; (U,V)\mapsto UV^\top,$$
and irreducibility of affine space $\C^{m\times r}\times \C^{n\times r}$.\\
(ii) follows from considering the map $\Upsilon$ over the reals, observing that the rank its Jacobian is not affected by this.\\
(iii) follows from (i), (ii) and Proposition~\ref{Prop:realdefobs}~(iii).\\
(iv) follows from (i) and Proposition~\ref{Prop:irreducible}, applied to $\Omega$.\\
(v) follows from the fact that $\Omega$ is a coordinate projection, therefore linear.
\end{proof}

\section{Advanced Algorithm for Minor Closure}
\label{Sec:takeakialgo}
\begin{center}
{\bfseries Takeaki Uno}\\
\texttt{uno@nii.jp} \\
National Institute of Informatics\\
Tokyo 101-8430, Japan

\end{center}
\subsection{Closability and the $r$-Closure}
\label{sec:Algos.closability.graph}

A crucial step in minor closure algorithm~\ref{Alg:minor} is to find a
$r\times r$ biclique in a (sub)graph$(V,W,E)$.

Algorithm~\ref{alg:find_a_clique} can find an $d_1\times d_2$
biclique efficientl based on two ideas: (i) iterate over row vertices and
(ii) work only on the $(d_2,d_1)$-core; here $(d_2,d_1)$-core is the maximal
subgraph of $(V,W,E)$ that the degrees of the row and column vertices
are at least $d_2$ and $d_1$, respectively.

A naive approach for finding an
$r\times r$ biclique might be to iterate over edges in $E$ and for
each edge $(v,w)\in E$ whose nodes have at least $r-1$ neighbors, check whether the subgraph induced by $(N(w),N(v))$
contains an $r-1\times r-1$ biclique. Instead our algorithm
iterates over nodes in $V$ and for each node $v\in V$, check whether the
subgraph induced by $(V',N(v))$
contains an $r-1\times r$ biclique, where $V'$ is defined by removing all
previously attempted nodes and the current $v$ from $V$. Iterating over
row vertices results in smaller number of iterations because $|V|\leq
|E|$, and allows us to avoid double checking, because previously
attempted nodes can be removed from $V'$. Concentrating on the
$(d_2,d_1)$-core is natural, because no $d_1\times d_2$ biclique
contains row or column vertex with degree less than $d_2$ or $d_1$,
respectively. We present the pruning step for finding the
$(d_2,d_1)$-core in Algorithm~\ref{alg:find_core}.

\begin{algorithm}
\caption{\texttt{FindAClique}$((V,W,E),d_1,d_2)$}
\label{alg:find_a_clique}
\begin{algorithmic}[1]
\State Inputs:  bipartite graph $(V,W,E)$, size of the bipartite
clique to be found $d_1\times d_2$.
\State Output: vertex sets of a clique $(I,J)$.
\State $(V,W,E)\leftarrow \texttt{FindCore}((V,W,E),d_2,d_1)$.
\If{$|V|<d_1$ or $|W|<d_2$}
\State Return $(\emptyset,\emptyset)$.
\EndIf
\State $V'\leftarrow V$.
\For{each $v\in V$}
\If{$d_1=1$ and $|N(v)|\geq d_2$}
\State Return $(\{v\},N(v))$.
\EndIf
\State $V'\leftarrow V'\backslash \{v\}$, $W'\leftarrow
N(v)$, $E'\leftarrow (V'\times W')\cap E$.
\State $(I',J')\leftarrow\texttt{FindAClique}((V',W',E'),d_1-1,d_2)$.
\If{$|I'|>0$ and $|J'|>0$}
\State Return $(I'\cup\{v\},J')$.
\EndIf
\EndFor
\State Return $(\emptyset,\emptyset)$.
\end{algorithmic}
\end{algorithm}

\begin{algorithm}
\caption{\texttt{FindCore}$((V,W,E),d_1,d_2)$}
\label{alg:find_core}
\begin{algorithmic}[1]
\State Inputs: bipartite graph $(V,W,E)$, minimum degree $d_1$ for
the row vertices and $d_2$ for the column vertices.
\State Output: pruned bipartite graph $(V,W,E)$.
\While{true}
\State $V'\leftarrow\{v\in V: |N(v)|<d_1\}$.
\State $W'\leftarrow\{w\in W: |N(w)|<d_2\}$.
\If{$V'=\emptyset$ and $W'=\emptyset$}
\State Return $(V,W, (V\times W)\cap E)$.
\EndIf
\State $V\leftarrow V\backslash V'$.
\State $W\leftarrow W\backslash W'$.
\EndWhile
\end{algorithmic}
\end{algorithm}

\end{document}